%% file: main.tex
\documentclass[12pt,a4paper]{article}
\input{preamble}
\let\epsilon\varepsilon
\let\tilde\widetilde

\begin{document}
\maketitle  


\begin{abstract}

This paper establishes the (nearly) optimal approximation error  characterization of deep rectified linear unit (ReLU) networks for smooth functions in terms of both width and depth simultaneously. 
To that end, we first prove that multivariate polynomials can be approximated by deep ReLU networks of width $\mathcal{O}(N)$ and depth $\mathcal{O}(L)$ with an approximation error  $\mathcal{O}(N^{-L})$.
Through local Taylor expansions and their deep ReLU network approximations, we show that deep ReLU networks of width $\mathcal{O}(N\ln N)$ and depth $\mathcal{O}(L\ln L)$ can approximate $f\in C^s([0,1]^d)$ with a nearly optimal approximation error   $\mathcal{O}(\|f\|_{C^s([0,1]^d)}N^{-2s/d}L^{-2s/d})$. 
Our estimate is non-asymptotic in the sense that it is valid for arbitrary width and depth specified by $N\in\mathbb{N}^+$ and $L\in\mathbb{N}^+$, respectively. 

\end{abstract}


\section{Introduction}
 Deep neural networks  have made significant impacts in many 
fields of computer science and engineering, especially for large-scale and high-dimensional learning problems. Well-designed neural network architectures, efficient training algorithms, and high-performance computing technologies have made neural-network-based methods very successful in real applications. Especially in supervised learning; e.g., image classification and objective detection, the great advantages of neural-network-based methods over traditional learning methods have been demonstrated. Understanding the approximation capacity of deep neural networks has become a key question for revealing the power of deep learning. A large number of experiments in real applications have shown the large capacity of deep network approximation from many empirical points of view, motivating much effort in establishing the theoretical foundation of deep network approximation. One of the fundamental problems is the characterization of the optimal approximation error  of deep neural networks of arbitrary depth and width. 

\subsection{Main result}
Previously, the quantitative characterization of the approximation power of deep feed-forward neural networks (FNNs) with rectified linear unit (ReLU) activation functions was provided in \cite{2019arXiv190605497S}. For ReLU FNNs with width $\calO(N)$ and depth $\calO(L)$, the deep network approximation of $f\in C([0,1]^d)$ admits an approximation error   $\calO\big(\omega_f( N^{-2/d}L^{-2/d})\big)$ in the $L^p$-norm for any $p\in[1,\infty]$, where $\omega_f(\cdot)$ is the modulus of continuity of $f$. In particular, for the class of H\"older continuous functions, the approximation error  is nearly optimal.\footnote{``nearly optimal'' up to a logarithmic  factor.} 
The next question is whether the smoothness of functions can improve the approximation error. In this paper,  we investigate the deep network approximation of smaller function space, such as the smooth function space $C^s([0,1]^d)$.

In Theorem~\ref{thm:Main} below, we prove by construction that ReLU FNNs with width $\calO(N\ln N)$ and depth $\calO(L \ln L)$ can approximate $f\in C^s([0,1]^d)$ with a nearly optimal approximation error  $\calO(\|f\|_{C^s([0, 1]^d)}N^{-2s/d}L^{-2s/d})$, where the norm $\|\cdot\|_{C^s([0, 1]^d)}$ is defined as  
\[
\|f\|_{C^s([0, 1]^d)}\coloneqq\max \big\{\|\partial^\bmalpha f\|_{L^\infty([0,1]^d)}: \|\bmalpha\|_1\le s,\,\bmalpha\in \N^d\big\}\quad \tn{for any $f\in C^s([0,1]^d)$.}
\]  

\begin{theorem}
    \label{thm:Main}
	Given a smooth function $f\in  C^s([0,1]^d)$ with $s\in \N^+$, for any $N,L\in \N^+$, there exists a function $\phi$ implemented by a ReLU FNN with width $C_1(N+2)\log_2(8N)$ and depth $C_2(L+2)\log_2(4L)+2d$ such that 
	\begin{equation*}
		\|\phi-f\|_{L^\infty([0,1]^d)}\le C_3 \|f\|_{C^s([0, 1]^d)} N^{-2s/d}L^{-2s/d},
	\end{equation*}
	where $C_1=17s^{d+1}3^dd$, $C_2=18s^2$, and $C_3=85(s+1)^d8^s$. 
\end{theorem}

As we can see from Theorem~\ref{thm:Main}, the smoothness improves the approximation error  in $N$ and $L$; e.g., $s\ge d$ implies $N^{-2s/d}L^{-2s/d} \le N^{-2}L^{-2}$. However, we would like to remark that the improved approximation error  is at the price of a prefactor much larger than $d^d$ if $s\geq d$. The proof of Theorem~\ref{thm:Main} will be presented in Section~\ref{sec:proofOfMainThm} and its tightness will be discussed in Section~\ref{sec:optimalityOfMainThm}.  In fact, the logarithmic terms in width and depth in Theorem~\ref{thm:Main} can be further reduced if the approximation error  is weakened. 
Given any $\tildeN,\tildeL\in\N^+$ with
\[
\tildeN \ge C_1(1+2)\log_2(8)=17s^{d+1}3^{d+2}d 
\quad \tn{and}\quad  
\tildeL \ge C_2(1+2)\log_2(4)+2d=108s^2+2d,
\]
there exist $N,L\in \N^+$ such that
\begin{equation*}
    C_1(N+2)\log_2(8N)\le \tildeN < C_1\big((N+1)+2\big)\log_2\big(8(N+1)\big)
\end{equation*}
and
\begin{equation*}
    C_2(L+2)\log_2(4L)+2d\le \tildeL< C_2\big((L+1)+2\big)\log_2\big(4(L+1)\big)+2d.
\end{equation*}
It follows that
\begin{equation*}
    N\ge \frac{N+3}{4}
    > \frac{\tildeN}{4C_1\log_2(8N+8)}
    \ge \frac{\tildeN}{4C_1\log_2(8\tildeN+8)}
    =\frac{\tildeN}{68s^{d+1}3^d d\log_2(8\tildeN+8)}
\end{equation*}
and
\begin{equation*}
    L\ge \frac{L+3}{4}
    > \frac{\tildeL-2d}{4C_2\log_2(4L+4)}
    \ge \frac{\tildeL-2d}{4C_2\log_2(4\tildeL+4)}
    = \frac{\tildeL-2d}{72s^2\log_2(4\tildeL+4)}.
\end{equation*}
Thus, we have an immediate corollary.
\begin{corollary}
    \label{cor:approxSmoothFun}
    Given a function $f\in  C^s([0,1]^d)$ with $s\in \N^+$, for any $\tildeN,\tildeL\in \N^+$,  there exists a function $\phi$ implemented by a ReLU FNN with width $ \tildeN$ and depth $\tildeL$ such that 
    \begin{equation*}
    \|\phi-f\|_{L^\infty([0,1]^d)}\le \tildeC_1 \|f\|_{C^s([0, 1]^d)}\Big(\tfrac{\tildeN}{\tildeC_2\log_2(8\tildeN+8)}\Big)^{-2s/d}\Big(\tfrac{\tildeL-2d}{\tildeC_3\log_2(4\tildeL+4)}\Big)^{-2s/d}
    \end{equation*}
    for any $\tildeN\ge 17s^{d+1}3^{d+2} d$ and $\tildeL\ge 108s^2+2d$, where $\tildeC_1=85(s+1)^d8^s$, $\tildeC_2=68s^{d+1}3^d d$, and $\tildeC_3=72s^2$.
\end{corollary}

Theorem~\ref{thm:Main} and Corollary~\ref{cor:approxSmoothFun} characterize the approximation  error in terms of  total number of neurons (with an arbitrary distribution in width and depth) and the smoothness of the target function to be approximated. 
The only result in this direction we are aware of in the literature is Theorem~$4.1$ of \cite{yarotsky2019}.
It shows that ReLU FNNs with width $2d+10$ and depth $L$ achieve a nearly optimal error $\calO((\tfrac{L}{\ln L})^{-2s/d})$ for sufficiently large $L$ when approximating functions in the unit ball of $C^s([0,1]^d)$. This result is essentially a special case of Corollary~\ref{cor:approxSmoothFun} by setting $\tildeN=\calO(1)$ and $\tildeL$ sufficiently large.

\subsection{Contributions and related work}

Our key contributions can be summarized as follows.

\begin{enumerate}[(i)]

    \item \textbf{Upper bound}: 
    We provide a \textbf{quantitative} and \textbf{non-asymptotic} approximation error  $\calO(\|f\|_{C^s([0,1]^d)} N^{-2s/d}L^{-2s/d})$ when the ReLU FNN has width $\calO(N\ln N)$ and depth $\calO(L\ln L)$ for functions in $C^s([0,1]^d)$ in Theorem~\ref{thm:Main}. In real applications, the first question is to decide the network width and depth since they are two required hyper-parameters. The approximation error  as a function of width and depth in this paper can directly answer this question, while the approximation results in terms of the total number of parameters in the literature cannot, because there are many architectures sharing the same number of parameters. Actually, an immediate corollary of our theorem as we shall discuss can also describe our theory in terms of the total number of parameters. 
    Furthermore, our results contain approximation error  estimates for both wide networks with fixed finite depth and deep networks with fixed finite width.
    
    \item \textbf{Lower bound}: 
    Through the Vapnik-Chervonenkis (VC) dimension upper bound of ReLU FNNs in \cite{pmlr-v65-harvey17a}, 
    we prove a lower bound 
    \[C\big(N^2L^2 (\ln N)^3 (\ln L)^3 \big)^{-s/d}\quad \tn{for some positive constant $C$}\]
    for the approximation error  of the functions in the unit ball of $C^s([0,1]^d)$ approximated by ReLU FNNs with width $\calO(N\ln N)$ and depth $\calO(L\ln L)$
    in Section~\ref{sec:optimalityOfMainThm}. Thus, the approximation error  $\calO(N^{-2s/d}L^{-2s/d})$ in Theorem~\ref{thm:Main} is nearly optimal for the unit ball of $C^s([0,1]^d)$.
    
    \item \textbf{Approximation of polynomials}: 
    It is proved by construction in Proposition~\ref{prop:approxPolynomial} that ReLU FNNs with width $\calO(N)$ and depth $\calO(L)$ can approximate polynomials on $[0,1]^d$ with an approximation error  $\calO(N^{-L})$.  This is a non-trivial extension of the result $\calO(2^{-L})$ for polynomial approximation by fixed-width ReLU FNNs with depth $L$ in \cite{yarotsky2017}.

    \item \textbf{Uniform approximation}: 
    The approximation error  in this paper is measured in the $L^\infty([0,1]^d)$-norm as a result of Theorem~\ref{thm:Gap}. To achieve this, given a ReLU FNN  approximating the target function $f$ uniformly well on $[0,1]^d$ except for a small region,  we develop a technique to construct a new ReLU FNN with a similar size to approximate $f$ \textbf{uniformly} well  on $[0,1]^d$ in Theorem~\ref{thm:Gap}. This technique can be applied to improve approximation errors from the $L^p$-norm to the $L^\infty$-norm for other function spaces in general, e.g., the continuous function space in \cite{2019arXiv190605497S}, which is of independent interest. 
\end{enumerate}

In particular, if we denote the best approximation error  of functions in $C^s_u([0,1]^d)$ approximated by ReLU FNNs with width $\tildeN$ and depth $\tildeL$ as 
\begin{equation*}
    \varepsilon_{s,d}(\tildeN,\tildeL) \coloneqq \sup_{f\in C^s_u([0,1]^d)} \Big(\inf_{\phi\in \NNF(\NNwidth\le \tildeN;\,\NNdepth \le \tildeL)} \|\phi-f\|_{L^\infty([0,1]^d)}\Big)\quad \tn{for any $\tildeN,\tildeL\in\N^+$,}
\end{equation*}
where $C^s_u([0,1]^d)$ denotes the unit ball of $C^s([0,1]^d)$ defined by
\begin{equation*}
\Csub\coloneqq \big\{f\in C^s([0,1]^d):\|\partial^\bmalpha f\|_{L^\infty([0,1]^d)}\le 1,\ \tn{for all }\bmalpha\in \N^d \tn{ with } \|\bmalpha\|_1\le s\big\}.
\end{equation*}

By combining the upper and lower bounds stated above, we have 
\begin{equation*}
\underbrace{C_1(s,d)\cdot \Big(\tildeN^2\tildeL^2 {\color{black}\ln (\tildeN\tildeL)}\Big)^{-s/d}\ \le}_{\tn{proved in Section~\ref{sec:optimalityOfMainThm}}}
\   \varepsilon_{s,d}(\tildeN,\tildeL) \ 
 \underbrace{ \le\  C_2(s,d)\cdot\Big(\tfrac{\tildeN^2\tildeL^2}{{\color{black}(\ln \tildeN\ln \tildeL)^2} }\Big)^{-s/d} }_{\tn{shown in Corollary~\ref{cor:approxSmoothFun}}}, 
\end{equation*}
where $C_1(s,d)$ and $C_2(s,d)$ are two positive constants in $s$ and $d$, and $C_2(s,d)$ can be \textbf{explicitly} represented by $s$ and $d$.

The expressiveness of deep neural networks has been studied extensively from many perspectives, e.g., in terms of combinatorics \cite{NIPS2014_5422}, topology \cite{ 6697897}, VC-dimension \cite{Bartlett98almostlinear,Sakurai,pmlr-v65-harvey17a}, fat-shattering dimension \cite{Kearns,Anthony:2009}, information theory \cite{PETERSEN2018296}, and classical approximation theory \cite{Cybenko1989ApproximationBS,HORNIK1989359,barron1993,yarotsky18a,yarotsky2017,doi:10.1137/18M118709X,ZHOU2019,10.3389/fams.2018.00014,2019arXiv190501208G,2019arXiv190207896G,suzuki2018adaptivity,Ryumei,Wenjing,Bao2019ApproximationAO,2019arXiv191210382L,MO,shijun6}. In the early works of approximation theory for neural networks, the universal approximation theorem  \cite{Cybenko1989ApproximationBS,HORNIK1991251,HORNIK1989359} without approximation errors showed that, given any $\varepsilon>0$, there exists a sufficiently large neural network approximating a target function in a certain function space within an error $\varepsilon$. For one-hidden-layer neural networks and functions with integral representations, Barron \cite{barron1993,barron2018approximation} showed an asymptotic approximation error  $\calO(\frac{1}{\sqrt{N}})$ in the $L^2$-norm, leveraging an idea that is similar to Monte Carlo sampling for high-dimensional integrals. For very deep ReLU neural networks with width fixed as $\calO(d)$ and depth $\calO(L)$, Yarotsky \cite{yarotsky18a,yarotsky2019} showed that the nearly optimal approximation errors for Lipschitz continuous functions and functions in the unit ball of $C^s([0,1]^d)$  are $\calO(L^{-2/d})$ and $\calO((L/\ln L)^{-2s/d})$, respectively. Note that the  results are asymptotic in the sense that $L$ is required to be sufficiently large and the prefactors of these rates are unknown. To obtain a generic result that characterizes the approximation error  for arbitrary width and depth with known prefactors to guide applications, the authors of \cite{2019arXiv190605497S} demonstrated  that the nearly optimal approximation error  for ReLU FNNs with width $\calO(N)$ and depth $\calO(L)$ to approximate Lipschitz continuous functions on $[0,1]^d$ is $\calO(N^{-2/d}L^{-2/d})$.
Such a nearly optimal error is further improved to an optimal one, $\calO\big((N^2L^2\ln N)^{-1/d}\big)$, in a more recent paper \cite{shijun6}.
In this paper, we extend this generic framework to $C^s([0,1]^d)$ with a nearly optimal approximation error  $\calO(\|f\|_{C^s([0,1]^d)} N^{-2s/d}L^{-2s/d})$. 

Most related works are summarized in Table~\ref{tab:relatedWorks} for the comparison of our contributions in this paper and the results in the literature.

\begin{table}[!ht]    
    \caption{A summary of existing approximation errors of ReLU FNNs for $\tn{Lip}([0,1]^d)$ (the Lipschitz continuous function space) and $C^s_u([0,1]^d)$ (the unit ball of $C^s([0,1]^d)$).} 
    \label{tab:relatedWorks}
    \centering  
    \resizebox{0.985\textwidth}{!}{ 
        \begin{tabular}{ccccccccc} 
            \toprule
            paper    &      function class &  width & depth &  approximation error & $L^p([0,1]^d)$-norm & tightness & valid for \\
            \midrule
            \cite{yarotsky2017}& polynomial &$\calO(1)$  & $\calO(L)$ &  $\calO(2^{-L})$ & $p=\infty$ & &  any $L\in \N^+$\\
            
            this paper & polynomial & $\calO(N)$ & $\calO(L)$   & $\calO(N^{-L})$ & $p=\infty$ & & any $N,L\in\N^+$\\
            
            \midrule
            \cite{SHEN201974} & $\tn{Lip}([0,1]^d)$ & $\calO(N)$ & $3$  & $\calO(N^{-2/d})$ & $p\in [1,\infty)$ & nearly tight in $N$ & any $N\in \N^+$\\
            
            \cite{yarotsky18a} & $\tn{Lip}([0,1]^d)$ & $2d+10$ & $\calO(L)$  & $\calO(L^{-2/d})$ &  $p=\infty$&  nearly tight in $L$ & large $L\in \N^+$\\
            
            \cite{2019arXiv190605497S} & $\tn{Lip}([0,1]^d)$ & $\calO(N)$ & $\calO(L)$  & $\calO(N^{-2/d}L^{-2/d})$ &  $p\in[1,\infty]$&nearly tight in $N$ and $L$ & any $N,L\in \N^+$\\

            \cite{shijun6} & $\tn{Lip}([0,1]^d)$ & $\calO(N)$ & $\calO(L)$  &$\calO\big((N^2L^2\ln N)^{-1/d}\big)$
            &  $p\in[1,\infty]$&  tight in $N$ and $L$ & any $N,L\in \N^+$
            \\
            
            \midrule

            \cite{yarotsky2019} & $C^s_u([0,1]^d)$ & $2d+10$& $\calO(L)$  & $\calO\big((L/\ln L)^{-2s/d}\big)$ & $p=\infty$& neatly tight in $L$ & large $L\in \N^+$\\
            
            this paper & $C^s_u([0,1]^d)$ & $\calO(N\ln N)$ & $\calO(L\ln L)$  & $\calO(N^{-2s/d}L^{-2s/d})$ & $p=\infty$&nearly tight in $N$ and $L$ & any $N,L\in \N^+$\\    
            
            this paper & $C^s_u([0,1]^d)$ & $\calO(N)$ & $\calO(L)$  & $\calO\big((N/\ln N)^{-2s/d}(L/\ln L)^{-2s/d}\big)$ & $p=\infty$&nearly tight in $N$ and $L$ & any $N,L\in \N^+$\\   
            \bottomrule
        \end{tabular} 
    }
\end{table}

\subsection{Discussion}

We will discuss the comparison of our theory with existing works and the application scope  in machine learning.

\subsubsection*{Approximation errors in $\calO(N)$ and $\calO(L)$ versus $\calO(W)$}

It is fundamental and indispensable to characterize deep network approximation in terms of width $\calO(N)$\footnote{For simplicity, we omit $\calO(\cdot)$ in the following discussion.} and depth $\calO(L)$ simultaneously in realistic applications, while the approximation in terms of the number of nonzero parameters $W$ is probably only of interest in theory. First, networks used in practice are specified via width and depth and, therefore, Theorem~\ref{thm:Main} can provide an error  bound for such networks. However, existing results in $W$ cannot serve this purpose because they may be only valid for networks with other widths and depths. Theories in terms of $W$ essentially have a single variable to control the network size in three types of structures: 1) a fixed width $N$ and a varying depth $L$; 2) a fixed depth $L$ and a varying width $N$; 3) both the width and depth are controlled by the target error $\epsilon$ (e.g., $N$ is a polynomial of $\frac{1}{\epsilon^d}$ and $L$ is a polynomial of $\ln(\frac{1}{\epsilon})$). Therefore, given a network with arbitrary width $N$ and depth $L$, there might not be a known theory in terms of $W$ to quantify the performance of this structure. Second, the error  characterization in terms of $N$ and $L$ is more useful than that in terms of $W$, because most existing optimization and generalization analyses are based on $N$ and $L$ \cite{Arthur18,Yuan1,Chen1,Arora2019FineGrainedAO,AllenZhu2019LearningAG,Weinan2019,Weinan2019APE,Ji2020PolylogarithmicWS}, to the best of our knowledge. Approximation results in terms of $N$ and $L$ are more consistent with optimization and generalization analysis tools to obtain a full error  analysis.

Most existing approximation theories for deep neural networks so far focus on the approximation error  in the number of parameters $W$ \cite{Cybenko1989ApproximationBS,HORNIK1989359,barron1993,DBLP:journals/corr/LiangS16,yarotsky2017,poggio2017,DBLP:journals/corr/abs-1807-00297,PETERSEN2018296,10.3389/fams.2018.00014,yarotsky18a,Ryumei,2019arXiv190501208G,2019arXiv190207896G,Wenjing,2019arXiv191210382L,suzuki2018adaptivity,Bao2019ApproximationAO,Opschoor2019,yarotsky2019,doi:10.1137/18M118709X,Hadrien,doi:10.1002/mma.5575,ZHOU2019,MO,bandlimit}. Controlling two variables $N$ and $L$ in our theory is more challenging than controlling one variable $W$ in the literature. The characterization of deep network approximation in terms of $N$ and $L$ can imply an approximation error in terms of $W$, while this may not be true the other way around, e.g., our theorems cannot be derived from results in \cite{yarotsky2019}. Let us discuss the first type of structure mentioned in the previous paragraph, which includes the best-known result for a nearly optimal approximation error, $\calO((W/\ln W)^{-2s/d})$, for functions in the unit ball of $C^s([0,1]^d)$ using ReLU FNNs with $W$ parameters \cite{yarotsky2019}. As an example to show how Theorem~\ref{thm:Main} in terms of $N$ and $L$ can be applied to show a similar result in terms of $W$. The main idea is to specify the value of $N$ and $L$ in Theorem~\ref{thm:Main} to show the desired corollary. For example, if we let 
$N=\calO(1)$ in Theorem~\ref{thm:Main}, then we have the following corollary, which is essentially equivalent to Theorem~$4.1$ of  \cite{yarotsky2019}.

\begin{corollary}
    \label{coro:parameters}
    Given any function $f$ in the unit ball of $C^s([0,1]^d)$ with $s\in \N^+$,  there exists a function $\phi$ implemented by a ReLU FNN with $W$ parameters such that
    \begin{equation*}
    \|\phi-f\|_{L^\infty([0,1]^d)}\le \calO\big((\tfrac{W}{\ln W})^{-2s/d}\big)\quad \tn{for large $W\in\N^+$}.
    \end{equation*}
\end{corollary}
As we can see in this example,  it is simple to derive Corollary~\ref{coro:parameters} above and Theorem~$4.1$ of \cite{yarotsky2019} using Theorem~\ref{thm:Main} in this paper. However, Theorem~\ref{thm:Main} cannot be derived from any existing result that characterizes approximation errors in terms of the number of parameters. 
Therefore, Theorem~\ref{thm:Main} goes beyond existing results on the approximation of deep neural networks.

Note that the logarithmic term in the approximation error   is not significant  in the case of $s>1$ since it can be cancelled out in the sense that  $\big(\tfrac{W}{\ln W}\big)^{-2s/d}\lesssim W^{-2\tilde{s}/d}$ for any $\tilde{s}\in (1,s)$. We remark that
Theorem~$3.3$ of \cite{yarotsky2019} provides a better approximation error by a logarithmic term: ReLU FNNs with $W$ nonzero parameters can approximate a function $f$ in the unit ball of $C^s([0,1]^d)$ within an error $\calO(W^{-2s/d})$. However, the network architecture therein is relatively complex and $s$-dependent as stated by the authors of \cite{yarotsky2019}. In fact, it contains many $s$-dependent blocks (sub-networks), making it difficult to implement if $s$ is not known in applications.
In contrast, our network architecture in Corollary~\ref{cor:approxSmoothFun} is simple and can be pre-specified once the width $\tildeN$ and depth $\tildeL$ therein are given.

\subsubsection*{Continuity of the weight selection}

We would like to discuss the continuity of the weight selection as a map $\Sigma:F_{s,d}\rightarrow \mathbb{R}^W$, where $F_{s,d}$ denotes the unit ball of the $d$-dimensional Sobolev space with smoothness $s$. For a fixed network architecture with a fixed number of parameters $W$, let $g:\mathbb{R}^W\rightarrow C([0,1]^d)$ be the map of realizing a ReLU FNN from a given set of parameters in $\mathbb{R}^W$ to a function in $C([0,1]^d)$. Suppose that the map $\Sigma$ is continuous such that $\|f-g(\Sigma(f))\|_{L^\infty([0,1]^d)}\leq \epsilon$ for all $f\in F_{s,d}$. Then $W\geq c \epsilon^{-d/s}$ with some constant $c$ depending only on $s$. This conclusion is given in Theorem 3 of \cite{yarotsky2017}, which is a corollary of Theorem~4.2 of \cite{Devore89optimalnonlinear} in a more general form.  These theorems mean that the weight selection map $\Sigma$ corresponding to our constructive proof in Theorem~\ref{thm:Main} in this paper is not continuous, since our error is better than $\calO(W^{-s/d})$. Theorem~4.2 of \cite{Devore89optimalnonlinear} is essentially a min-max criterion to evaluate weight selection maps maintaining continuity: the approximation error  obtained by minimizing over all continuous selections $\Sigma$ and network realizations $g$ and maximizing over all target functions is bounded below by $\calO(W^{-s/d})$. In the worst case, a continuous weight selection cannot enjoy an approximation error  beating $\calO(W^{-s/d})$. However, Theorem~4.2 of \cite{Devore89optimalnonlinear} does not exclude the possibility that most functions of interest in practice may still enjoy a continuous weight selection with the approximation error  in Theorem~\ref{thm:Main}. 
It would be interesting in future work to investigate whether continuous weight selection is possible for many functions commonly encountered in real applications.

\subsubsection*{Application scope of our theory in machine learning}

In deep learning, given a target function $f$, the final goal is to train a function $\phi(\bm{x};\bm{\theta})$ approximating $f$ well, where
 $\phi(\bm{x};\bm{\theta})$ is a function in $\bmx\in \calX$ realized by a network architecture parameterized with $\bmtheta\in\R^W$. To get the best solution, one needs to identify the expected risk minimizer
 \begin{equation*}
	\bm{\theta}_{\mathcal{D}}\coloneqq \argmin_{\bm{\theta}\in\R^W}R_{\mathcal{D}}(\bm{\theta}),\quad \tn{where}\ R_{\mathcal{D}}(\bm{\theta})=
	\mathbb{E}_{\bm{x}\sim U(\calX)} \left[\ell\big( \phi(\bm{x};\bm{\theta}),f(\bm{x})\big)\right]
\end{equation*}
with a loss function usually taken as $\ell(y,y')=\frac{1}{2}|y-y'|^2$ and an unknown data distribution $U(\calX)$.

In practice, only data samples $\{( \bm{x}_i,f(\bm{x}_i){ )}\}_{i=1}^n$ instead of $f$ and $U(\calX)$ are available. Thus,  
the empirical risk minimizer $\bm{\theta}_{\mathcal{S}}$ is used to model/approximate the expected risk minimizer $\bm{\theta}_{\mathcal{D}}$, where
\begin{equation}\label{eqn:emloss}
	\bm{\theta}_{\mathcal{S}}\coloneqq\argmin_{\bm{\theta}\in\R^W}R_{\mathcal{S}}(\bm{\theta}),\quad \tn{where}\ R_{\mathcal{S}}(\bm{\theta}):=
	\frac{1}{n}\sum_{i=1}^n \ell\big( \Phi(\bm{x}_i,\bm{\theta}),f(\bm{x}_i)\big).
\end{equation}

 In real applications, only a numerical solution (denoted as $\bm{\theta}_{\mathcal{N}}$) is achieved when a numerical optimization method is applied to solve \eqref{eqn:emloss}.
 Hence,   the actually learned function 
 generated by the network  is  $\phi(\bm{x};\bm{\theta}_{\mathcal{N}})$.
Since $R_{\mathcal{D}}(\bm{\theta}_{\mathcal{N}})$ is the expected inference error  over all possible data samples, it can quantify how good  $\phi(\bm{x};\bm{\theta}_{\mathcal{N}})$ is. Note that
\begin{align}\label{eqn:gen}
		\myMathResize[0.8]{
		R_{\mathcal{D}}(\bm{\theta}_{\mathcal{N}})
		}
		&\myMathResize[0.8]{
		=\underbrace{[R_{\mathcal{D}}(\bm{\theta}_{\mathcal{N}})-R_{\mathcal{S}}(\bm{\theta}_{\mathcal{N}})]}_{\tn{GE}}
		+\underbrace{[R_{\mathcal{S}}(\bm{\theta}_{\mathcal{N}})-R_{\mathcal{S}}(\bm{\theta}_{\mathcal{S}})]}_{\tn{OE}}   +\underbrace{[R_{\mathcal{S}}(\bm{\theta}_{\mathcal{S}})-R_{\mathcal{S}}(\bm{\theta}_{\mathcal{D}})]}_{\tn{$\le 0$ by \eqref{eqn:emloss}}} +\underbrace{[R_{\mathcal{S}}(\bm{\theta}_{\mathcal{D}})-R_{\mathcal{D}}(\bm{\theta}_{\mathcal{D}})]}_{\tn{GE}}
		+\underbrace{R_{\mathcal{D}}(\bm{\theta}_{\mathcal{D}})}_{\tn{AE}}
		}
		\nonumber \\
			&\myMathResize[0.8]{
			\le \underbrace{R_{\mathcal{D}}(\bm{\theta}_{\mathcal{D}})}_{\tn{\color{blue}Approximation error  (AE)}} 
			\ +\  
			\underbrace{[R_{\mathcal{S}}(\bm{\theta}_{\mathcal{N}})-R_{\mathcal{S}}(\bm{\theta}_{\mathcal{S}})]}_{\tn{\color{blue}Optimization error  (OE)}}
			\  + \  
			\underbrace{[R_{\mathcal{D}}(\bm{\theta}_{\mathcal{N}})-R_{\mathcal{S}}(\bm{\theta}_{\mathcal{N}})]
			+[R_{\mathcal{S}}(\bm{\theta}_{\mathcal{D}})-R_{\mathcal{D}}(\bm{\theta}_{\mathcal{D}})]}_{\tn{\color{blue}Generalization error  (GE)}}
			}.
\end{align}

Constructive approximation provides an upper bound of $R_{\mathcal{D}}(\bm{\theta}_{\mathcal{D}})$ in terms of the network size. For example, Theorem~\ref{thm:Main} and its corollaries provide an upper bound $\calO(\|f\|_{C^s([0,1]^d)}N^{-2s/d}L^{-2s/d})$ of $R_{\mathcal{D}}(\bm{\theta}_{\mathcal{D}})$ for $C^s([0,1]^d)$. The second term of \eqref{eqn:gen} is bounded by the optimization error  of the numerical algorithm applied to solve the empirical loss minimization problem in \eqref{eqn:emloss}. The study of the bounds for the third and fourth terms is referred to as the generalization error  analysis of neural networks. 
	
One of  the key targets in  the area of deep learning is to develop algorithms to reduce  $R_\calD{(\bmtheta_\calN)}$.
Our analysis here provides an upper bound of the approximation  error  $R_{\mathcal{D}}(\bm{\theta}_{\mathcal{D}})$ for smooth functions, which is crucial to control $R_\calD{(\bmtheta_\calN)}$.   Instead of deriving an approximator to attain  the  error  bound,  deep learning algorithms aim to identify a solution $\phi(\bm{x};\bm{\theta}_{\mathcal{N}})$ reducing the generalization and optimization errors in \eqref{eqn:gen}.  Solutions minimizing both generalization and optimization errors will lead to a good solution only if we also have a good upper bound estimate of $R_{\mathcal{D}}(\bm{\theta}_{\mathcal{D}})$ as shown in \eqref{eqn:gen}.   Independent of whether our analysis here  leads to a good approximator, which is an interesting topic to pursue,  the  theory here does provide a key ingredient in the error  analysis of deep learning algorithms.

We would like to emphasize that the introduction of the ReLU activation function to image classification is one of the key techniques that boost the performance of deep learning \cite{NIPS2012_c399862d} with surprising generalization, which is the main reason that we focus on ReLU FNNs in this paper.

\vspace{8pt}
\textbf{Organization}: The rest of the present paper is organized as follows.
In Section~\ref{sec:approxSmoothFun}, we prove Theorem~\ref{thm:Main} by combining two theorems (Theorems~\ref{thm:Gap} and \ref{thm:MainGap}) that will be proved later. We will also discuss the optimality of Theorem~\ref{thm:Main} in Section~\ref{sec:approxSmoothFun}. Next, Theorem~\ref{thm:Gap} will be proved in Section~\ref{sec:badRegionRemoval} while Theorem~\ref{thm:MainGap} will be shown in Section~\ref{sec:4}. Several propositions supporting Theorem~\ref{thm:MainGap} will be presented in Section~\ref{sec:proofOfProposition}. Finally, Section~\ref{sec:conclusion} concludes this paper with a short discussion.

\section{Approximation of smooth functions}
\label{sec:approxSmoothFun}

In this section, we will prove the quantitative approximation error  in Theorem~\ref{thm:Main} by construction and discuss its tightness. 
Notation throughout the proof will be
summarized in Section~\ref{sec:notation}. The proof of Theorem~\ref{thm:Main} is mainly based on Theorems~\ref{thm:Gap} and \ref{thm:MainGap}, which will be proved in Sections~\ref{sec:badRegionRemoval} and \ref{sec:4}, respectively. To show the tightness of Theorem~\ref{thm:Main}, we will introduce the VC-dimension in Section~\ref{sec:optimalityOfMainThm}.

\subsection{Notation}
\label{sec:notation}

Now let us summarize the main notation of this paper as follows.
\begin{itemize}
\item Let $\R$, $\Q$, and $\Z$ denote the set of real numbers, rational numbers, and integers, respectively.
    
    \item Let $\N$ and $\N^+$ denote the set of natural numbers and positive natural numbers, respectively.  That is,
    $\N^+=\{1,2,3,\cdots\}$ and $\N=\N^+\bigcup\{0\}$.

	\item 
	Vectors and matrices are denoted in a bold font.  Standard vectorization is adopted in matrix and vector computation. For example, a scalar plus a vector means adding the scalar to each entry of the vector. Additionally, ``['' and ``]''  are used to  partition matrices (vectors) into blocks, e.g., $\bmA=\left[\begin{smallmatrix}\bmA_{11}&\bmA_{12}\\ \bmA_{21}&\bmA_{22}\end{smallmatrix}\right]$ and $\bm{v}=\left[\def\arraystretch{0.748}\begin{array}{c}
		v_1  \\
		\vdots \\
		v_d
	\end{array}\right]=[v_1,\cdots,v_d]^T\in \R^d$.
	
    \item Let $\one_{S}$ be the characteristic (indicator) function on a set $S$; i.e., $\one_{S}$ is equal to $1$ on $S$ and $0$ outside $S$.
    \item Let $\ball(\bmx,r)\subseteq \R^d$ be the closed ball with a center $\bmx\subseteq \R^d$ and a radius $r\ge 0$. 
    
    \item  Similar to ``$\min$'' and ``$\max$'', let $\middleValue(x_1,x_2,x_3)$ be the middle value of three inputs $x_1$, $x_2$, and $x_3$\footnote{``$\middleValue$'' can be defined via $\middleValue(x_1,x_2,x_3)=x_1+x_2+x_3-\max(x_1,x_2,x_3)-\min(x_1,x_2,x_3)$, which can be implemented by a ReLU FNN.}. For example, $\middleValue(2,1,3)=2$ and $\middleValue(3,2,3)=3$.
    
    \item The set difference of two sets $A$ and $B$ is denoted by $A\backslash B:=\{x:x\in A,\ x\notin B\}$. 
    
    	\item	For a real number $p\in[1,\infty)$, the $p$-norm of $\bmx=[x_1,x_2,\cdots,x_d]^T\in \R^d$ is defined by
    \begin{equation*}
    	\|\bmx\|_p\coloneqq \big(|x_1|^p+|x_2|^p+\cdots+|x_d|^p\big)^{1/p}.
    \end{equation*}

    \item For any $x\in \R$, let $\lfloor x\rfloor:=\max \{n: n\le x,\ n\in \Z\}$ and $\lceil x\rceil:=\min \{n: n\ge x,\ n\in \Z\}$.
    
    \item Assume $\bm{n}\in \N^d$; then $f(\bm{n})=\mathcal{O}(g(\bm{n}))$ means that there exists positive $C$ independent of $\bm{n}$, $f$, and $g$ such that $ f(\bm{n})\le C g(\bm{n})$ when all entries of $\bm{n}$ go to $+\infty$.
    
    \item 
    The modulus of continuity of a continuous function $f\in C([0,1]^d)$ is defined as 
    \begin{equation*}
    \omega_f(r)\coloneqq \sup\big\{|f(\bmx)-f(\bmy)|: \|\bmx-\bmy\|_2\le r,\ \bmx,\bmy\in [0,1]^d\big\}\quad\tn{for any $r\ge0$.}
    \end{equation*}

    \item 
    A $d$-dimensional multi-index is a $d$-tuple
    $\bmalpha=[\alpha_1,\alpha_2,\cdots,\alpha_d]^T\in \N^d.$
    Several related notation are listed below.
    \begin{itemize}
        \item  $\|\bmalpha\|_1=|\alpha_1|+|\alpha_2|+\cdots+|\alpha_d|$;
        \item $\bmx^\bmalpha=x_1^{\alpha_1}  x_2^{\alpha_2} \cdots x_d^{\alpha_d}$, where $\bmx=[x_1,x_2,\cdots,x_d]^T$;
        \item $\bmalpha!=\alpha_1!\alpha_2!\cdots \alpha_d!$;
        \item  $
        \partial^\bmalpha=\tfrac{\partial^{\alpha_1}}{\partial x_1^{\alpha_1}}\tfrac{\partial^{\alpha_2}}{\partial x_2^{\alpha_2}}\cdots \tfrac{\partial^{\alpha_d}}{\partial x_d^{\alpha_d}}
        $.
    \end{itemize}

 	\item For any closed cube $Q\subseteq \R^d$ and a real number $r>0$, let $rQ$ denote the closed cube which shares the same center of $Q$ and whose sidelength is the product of $r$ and the sidelength of $Q$.

    \item 
    Given any $K\in N^+$ and $\delta\in (0,\tfrac{1}{K})$, define a trifling region  $\Omega([0,1]^d,K,\delta)$ of $[0,1]^d$ as 
    \begin{equation}
    \label{eq:badRegionDef}
    \Omega([0,1]^d,K,\delta)\coloneqq\bigcup_{i=1}^{d} \Big\{\bmx=[x_1,x_2,\cdots,x_d]^T\in[0,1]^d: x_i\in \cup_{k=1}^{K-1}(\tfrac{k}{K}-\delta,\tfrac{k}{K})\Big\}.
    \end{equation}
    In particular, $\Omega([0,1]^d,K,\delta)=\emptyset$ if $K=1$. See Figure~\ref{fig:region} for two examples of the trifling region.
    \begin{figure}[ht!]        
        \centering
        \begin{minipage}{0.95\textwidth}
		\centering
		\begin{subfigure}[b]{0.4\textwidth}
			\centering
			\includegraphics[width=0.9\textwidth]{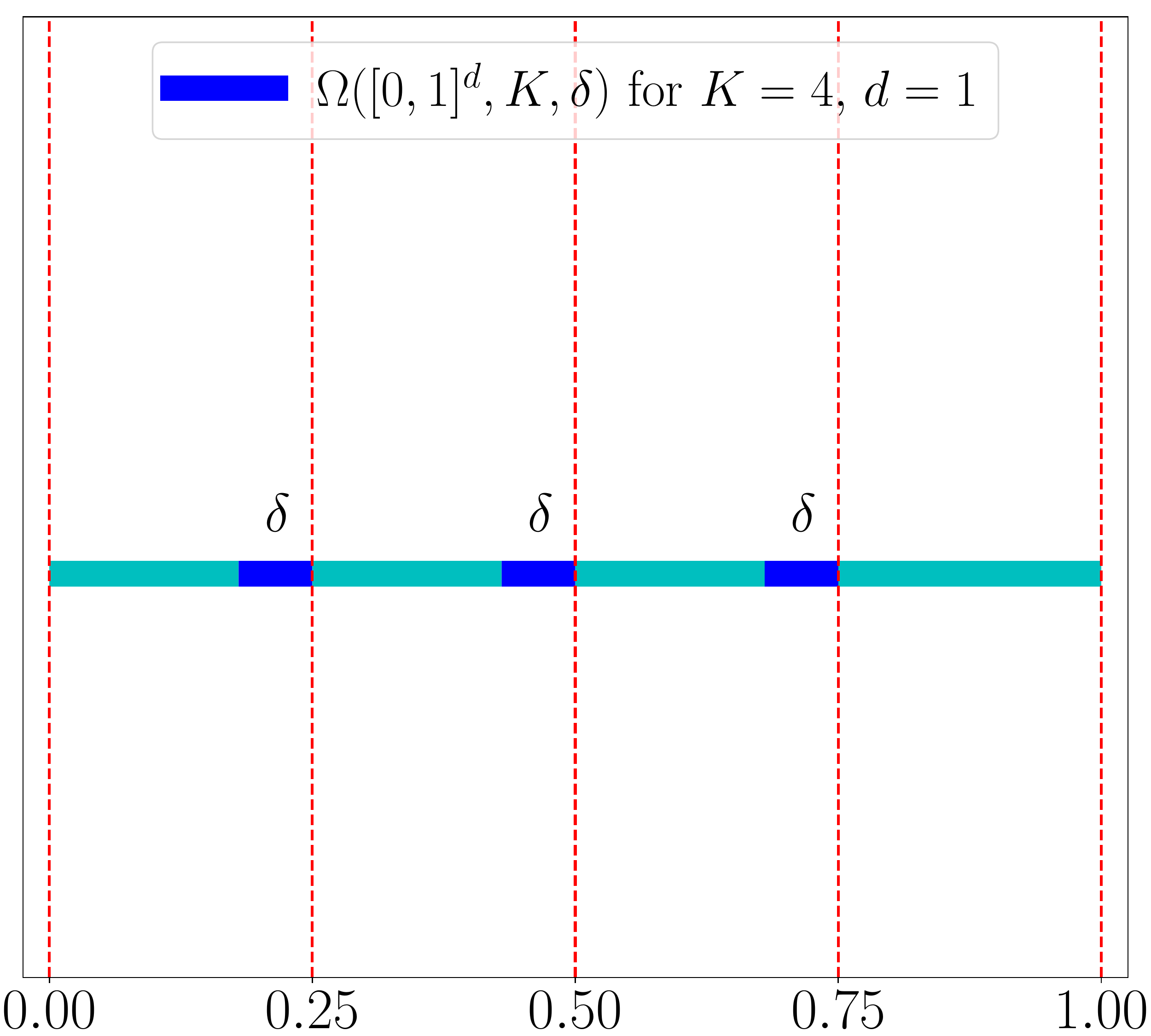}
			\subcaption{}
		\end{subfigure}
			\begin{minipage}{0.01\textwidth}
				\
			\end{minipage}
		\begin{subfigure}[b]{0.4\textwidth}
			\centering
			\includegraphics[width=0.9\textwidth]{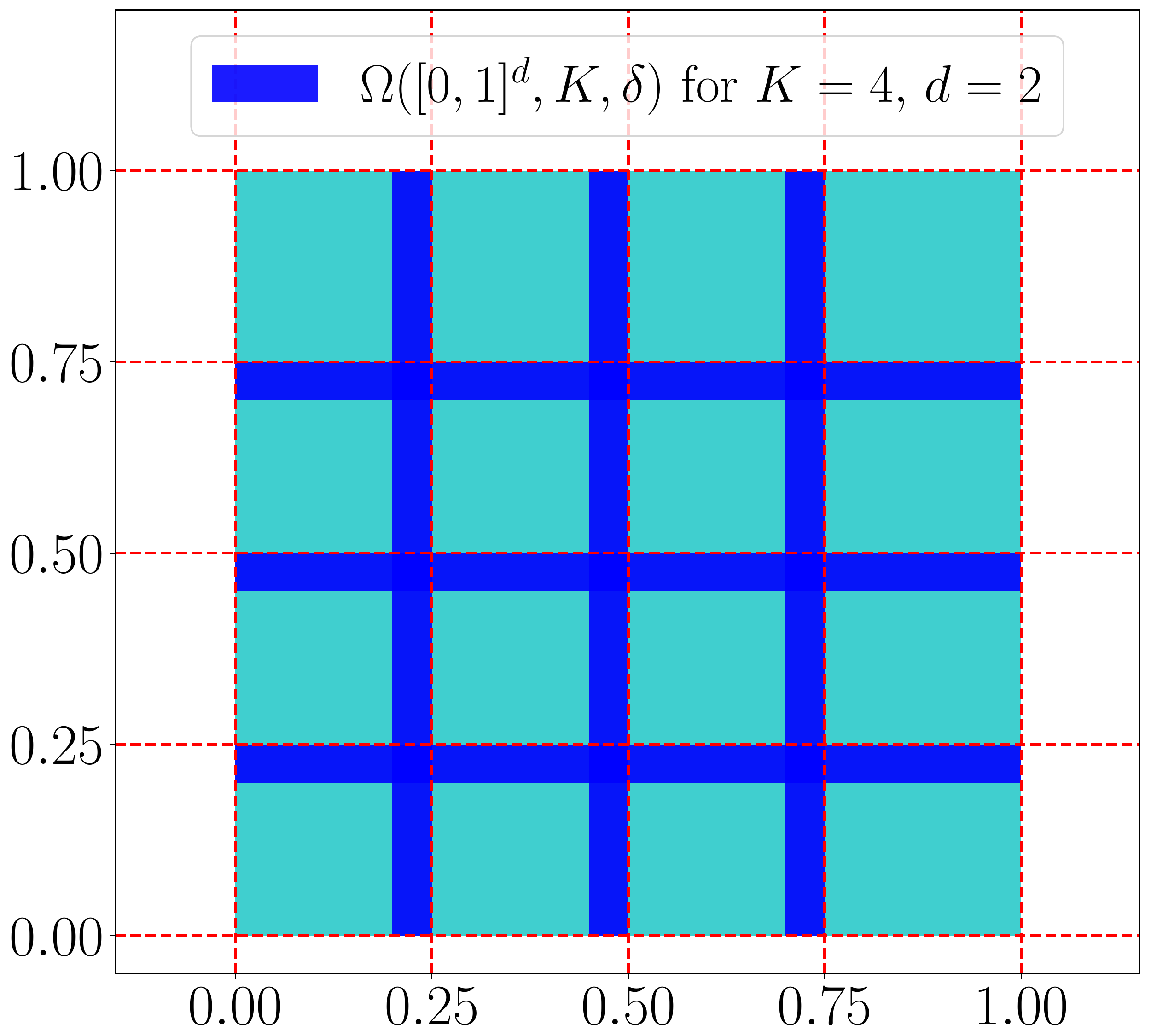}
			\subcaption{}
		\end{subfigure}
	\end{minipage}
        \caption{Two examples of the trifling region. (a)  $K=4,d=1$. (b) $K=4,d=2$.}
        \label{fig:region}
    \end{figure}

    \item
    Given $E\subseteq \R^d$, let $C^s(E)$ denote the set containing all  functions, all $k$-th order partial derivatives of which exist and are continuous on $E$ for any $k\in \N$ with $0\le k\le s$. In particular,   $C^0(E)$, also denoted by $C(E)$, is the set of continuous functions on $E$. For the case $s=\infty$, $C^\infty(E)=\bigcap_{s=0}^\infty C^s(E)$.
    The $C^s$-norm is defined by
    \begin{equation*}
    	\|f\|_{C^s(E)}\coloneqq \max \big\{\|\partial^\bmalpha f\|_{L^\infty(E)}: \bmalpha\in \N^d \tn{ with } \|\bmalpha\|_1\le s\big\}.
    \end{equation*} 
    Generally, $E$ is assigned as $[0,1]^d$ in this paper. In particular, the closed unit ball of $C^s([0,1]^d)$ is denoted by
    \[\Csub\coloneqq \big\{f\in C^s([0,1]^d):\|f\|_{C^s([0,1]^d)}\le 1\big\}.\]
    
    \item We use ``$\NNF$" to mean ``functions implemented by ReLU FNNs'' for short and use Python-type notation to specify a class of functions implemented by ReLU FNNs with several conditions. To be precise, we use $\NNF(\tn{c}_1\NNspace \tn{c}_2\NNspace \cdots\NNspace \tn{c}_m)$ to denote the function set containing all functions implemented by ReLU FNN architectures satisfying $m$ conditions given by $\{\tn{c}_i\}_{1\leq i\leq m}$,
    each of which may specify the number of inputs ($\NNinput$), the number of outputs ($\NNoutput$), the total number of nodes in all hidden layers ($\NNnode$), the number of hidden layers ($\NNdepth$), the number of total parameters ($\NNparameter$), and the width in each hidden layer ($\NNwidthvec$), the maximum width of all hidden layers ($\NNwidth$), etc. For example, if $\phi\in \NNF(\NNinput=2 \NNspace \NNwidthvec=[100,100]\NNspace\NNoutput=1)$, then $\phi$ is a function  satisfying the following conditions.
    \begin{itemize}
        \item $\phi$ maps from $\R^2$ to $\R$.
        \item $\phi$ is implemented by a ReLU FNN with two hidden layers and the number of nodes in each hidden layer being $100$.
    \end{itemize}

    \item Let $\sigma:\R\to \R$ denote the rectified linear unit (ReLU), i.e. $\sigma(x)=\max\{0,x\}$. With the abuse of notation, we define $\sigma:\R^d\to \R^d$ as $\sigma(\bmx)=\left[\begin{array}{c}
              \max\{0,x_1\}  \\
              \vdots \\
              \max\{0,x_d\}
     \end{array}\right]$ for any $\bmx=[x_1,\cdots,x_d]^T\in \R^d$.

    \item For a function $\phi\in \NNF(\NNinput=d\NNspace\NNwidthvec=[N_1,N_2,\cdots,N_L]\NNspace\NNoutput=1)$, if we set $N_0=d$ and $N_{L+1}=1$, then the architecture of the network implementing $\phi$ can be briefly described as follows:
    \begin{equation*}
    	\begin{aligned}
    		\bm{x}=\widetilde{\bm{h}}_0 \myto{2.2}^{\bm{W}_0,\ \bm{b}_0} \bm{h}_1\myto{1.12}^{\sigma} \tilde{\bm{h}}_1 \quad \cdots\quad
    		\myto{2.7}^{\bm{W}_{L-1},\ \bm{b}_{L-1}} \bm{h}_L\myto{1.12}^{\sigma} \tilde{\bm{h}}_L \myto{2.2}^{\bm{W}_{L},\ \bm{b}_{L}} \bm{h}_{L+1}=\phi(\bm{x}),
    	\end{aligned}
    \end{equation*}
    where $\bm{W}_i\in \R^{N_{i+1}\times N_{i}}$ and $\bm{b}_i\in \R^{N_{i+1}}$ are the weight matrix and the bias vector in the $i$-th affine linear transform $\calL_i$ in $\phi$, respectively, i.e., 
    \[\bm{h}_{i+1} =\bm{W}_i\cdot \tilde{\bm{h}}_{i} + \bm{b}_i\eqqcolon \calL_i(\tilde{\bm{h}}_{i})\quad \tn{for $i=0,1,\cdots,L$}\]  
    and
    \[
    \tilde{\bm{h}}_i=\sigma(\bm{h}_i)\quad \tn{for $i=1,2,\cdots,L$.}
    \]
    In particular, $\phi$ can be represented in a form of function compositions as follows
    \begin{equation*}
    	\phi =\calL_L\circ\sigma\circ\calL_{L-1}\circ \sigma\circ \ \cdots \  \circ \sigma\circ\calL_1\circ\sigma\circ\calL_0,
    \end{equation*}
    which has been illustrated in Figure~\ref{fig:ReLUeg}.
    \begin{figure}[!htp]        
    	\centering
    	\includegraphics[width=0.7\textwidth]{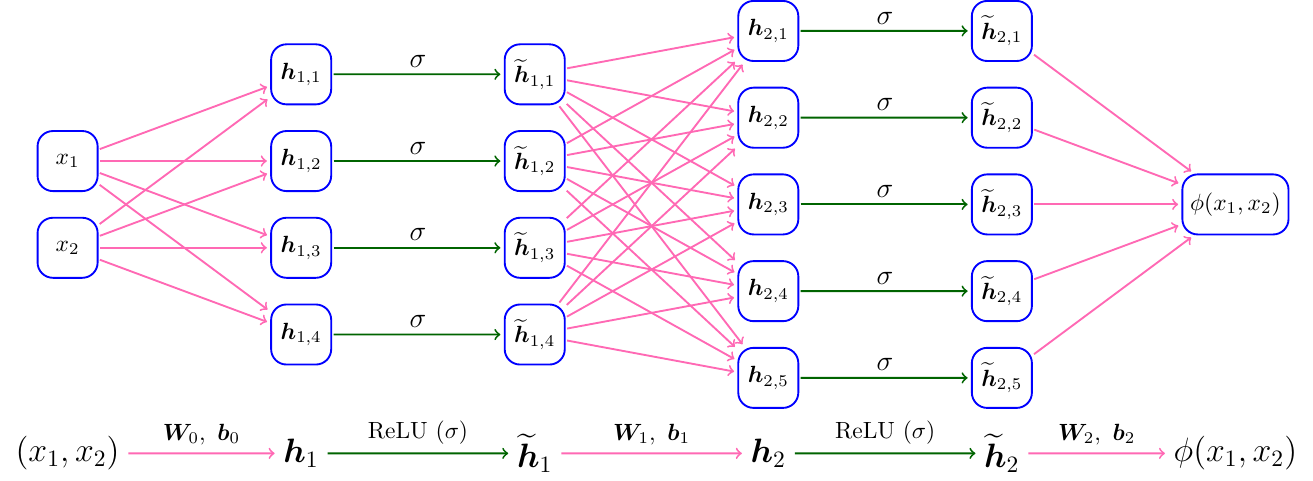}
    	\caption{An example of a ReLU FNN with width $5$ and depth $2$. }
    	\label{fig:ReLUeg}
    \end{figure}

    \item The expression ``a network (architecture) with (of) width $N$ and depth $L$'' means
    \begin{itemize}
    	\item The maximum width of this network (architecture) for all \textbf{hidden} layers  is no more than $N$.
    	\item The number of \textbf{hidden} layers of this network (architecture) is  no more than $L$.
    \end{itemize}

\item For any $\theta\in[0,1)$, suppose its binary representation is $\theta=\sum_{\ell=1}^{\infty}\theta_\ell2^{-\ell}$ with $\theta_\ell\in \{0,1\}$. We introduce a special notation $\bin   0.\theta_1\theta_2\cdots \theta_L$ to denote the $L$-term binary representation of $\theta$, i.e., $\bin 0.\theta_1\theta_2\cdots \theta_L\coloneqq \sum_{\ell=1}^{L}\theta_\ell2^{-\ell}\approx \theta$. 
\end{itemize}

\subsection{Proof of Theorem~\ref{thm:Main}}
\label{sec:proofOfMainThm}
The introduction of the trifling region $\Omega([0,1]^d,K,\delta)$ is due to the fact that ReLU FNNs cannot approximate a step function uniformly well (as the ReLU activation function is continuous), which is also the reason for the main difficulty in obtaining approximation errors in the $L^\infty([0,1]^d)$-norm in our previous papers \cite{SHEN201974,2019arXiv190605497S}. The trifling region is a key technique to simplify the proofs of theories in \cite{SHEN201974,2019arXiv190605497S} as well as the proof of Theorem~\ref{thm:Main}. 

First, we present Theorem~\ref{thm:Gap} to show that, as long as good uniform approximation by a ReLU FNN can be obtained outside the trifling region, the uniform approximation error  can also be well controlled inside the trifling region when the network size is slightly increased. Second, as a simplified version of Theorem~\ref{thm:Main} ignoring the approximation error  in the trifling region $\Omega([0,1]^d,K,\delta)$, Theorem~\ref{thm:MainGap} shows the existence of a ReLU FNN approximating a target smooth function uniformly well outside the trifling region. Finally, Theorems~\ref{thm:Gap} and \ref{thm:MainGap} immediately lead to Theorem~\ref{thm:Main}. Theorem~\ref{thm:Gap} can be applied to improve the theories in \cite{SHEN201974,2019arXiv190605497S} to obtain approximation errors in the $L^\infty([0,1]^d)$-norm.

\begin{theorem}
    \label{thm:Gap}
    Given any $\varepsilon>0$, $N,L,K\in \N^+$, and $\delta\in(0, \tfrac{1}{3K}]$,
    assume $f\in C([0,1]^d)$ and $\tildephi$ is a function implemented by a ReLU FNN with width $N$ and depth $L$. If 
    \begin{equation*}
    |\tildephi(\bmx)-f(\bmx)|\le \varepsilon\quad \tn{for any $\bmx\in [0,1]^d\backslash \Omega([0,1]^d,K,\delta)$,}
    \end{equation*}
    then there exists a new function  $\phi$ implemented by a ReLU FNN with width $3^d(N+4)$ and depth $L+2d$ such that 
    \begin{equation*}
    |\phi(\bmx)-f(\bmx)|\le \varepsilon+d\cdot\omega_f(\delta)\quad \tn{for any $\bmx\in [0,1]^d$.}
    \end{equation*}
\end{theorem}

\begin{theorem}
	\label{thm:MainGap}
	Assume that $f\in C^s([0,1]^d)$ satisfies $\|\partial^\bmalpha f\|_{L^\infty([0,1]^d)}\le 1$ for any $\bmalpha\in\N^d$ with $\|\bmalpha\|_1\le s$. For any $N,L\in \N^+$, there exists a function $\phi$ implemented by a ReLU FNN with width $16s^{d+1}d(N+2)\log_2(8N)$ and depth $18s^2(L+2)\log_2(4L)$ such that
	\begin{equation*}
		|\phi(\bmx)-f(\bmx)|\le 84(s+1)^d8^sN^{-2s/d}L^{-2s/d}\quad \tn{for any $\bmx\in [0,1]^d\backslash \Omega([0,1]^d,K,\delta)$,}
	\end{equation*}
	where $K=\lfloor N^{1/d}\rfloor^2\lfloor L^{2/d}\rfloor$ and $\delta$ is an arbitrary number in  $(0,\tfrac{1}{3K}]$.
\end{theorem}

We first prove Theorem~\ref{thm:Main} by assuming Theorems~\ref{thm:Gap} and \ref{thm:MainGap} are true. The proofs of Theorems~\ref{thm:Gap} and \ref{thm:MainGap} can be found in Sections~\ref{sec:badRegionRemoval} and \ref{sec:4}, respectively. 
\begin{proof}[Proof of Theorem~\ref{thm:Main}]
	 We may assume $\|f\|_{C^s([0,1]^d)}>0$ since $\|f\|_{C^s([0,1]^d)}=0$ is a trivial case. Define $\tildef\coloneqq\tfrac{f}{\|f\|_{C^s([0,1]^d)}}\in \Csub$.
Set $K= \lfloor N^{1/d}\rfloor ^2 \lfloor L^{2/d}\rfloor$ and choose a small $\delta\in (0,\tfrac{1}{3K}]$ such that
\[d\cdot\omega_\tildef(\delta)\le N^{-2s/d}L^{-2s/d}.\]

Clearly, $\|\partial^\bmalpha \tildef\|_{L^\infty([0,1]^d)}\le 1$ for any $\bmalpha\in\N^d$ with $\|\bmalpha\|_1\le  s$.
By Theorem~\ref{thm:MainGap}, there exists a function $\widehat{\phi}$ implemented by a ReLU FNN with width $16s^{d+1}d(N+2)\log_2(8N)$ and depth $18s^2(L+2)\log_2(4L)$ such that
\begin{equation*}
	|\widehat{\phi}(\bmx)-\tildef(\bmx)|\le 84(s+1)^d8^sN^{-2s/d}L^{-2s/d}\eqqcolon \varepsilon\quad \tn{for any $\bmx\in [0,1]^d\backslash \Omega([0,1]^d,K,\delta)$.}
\end{equation*}
By Theorem~\ref{thm:Gap}, there exists a new function $\tildephi$ implemented by a ReLU FNN with width \[3^d\big(16s^{d+1}d(N+2)\log_2(8N)+4\big)\le 17s^{d+1}3^d d(N+2)\log_2(8N)\] and depth $18s^2(L+2)\log_2(4L)+2d$ such that 
\begin{equation*}
	\begin{split}
		\|\tildephi-\tildef\|_{L^\infty([0,1]^d)}
		\le \varepsilon+ d\cdot\omega_\tildef(\delta)
		&= 84(s+1)^d8^sN^{-2s/d}L^{-2s/d}+d\cdot\omega_\tildef(\delta)\\
		&\le 85(s+1)^d8^sN^{-2s/d}L^{-2s/d}.
	\end{split}
\end{equation*}
Finally, set $\phi=\|f\|_{C^s([0,1]^d)}\cdot \tildephi$; then
\begin{equation*}
	\begin{split}
		\|{\phi}-{f}\|_{L^\infty([0,1]^d)}
		&=\|f\|_{C^s([0,1]^d)}\cdot\|\tildephi-\tildef\|_{L^\infty([0,1]^d)}\\
		&\le  85(s+1)^d8^s\|f\|_{C^s([0,1]^d)}N^{-2s/d}L^{-2s/d},
	\end{split}
\end{equation*}
and $\phi$ can also be implemented by a ReLU FNN with width $17s^{d+1}3^d d(N+2)\log_2(8N)$ and depth $18s^2(L+2)\log_2(4L)+2d$.
So we finish the proof.
\end{proof}

\subsection{Optimality of Theorem~\ref{thm:Main}}
\label{sec:optimalityOfMainThm}
In this section, we will show that the approximation error  in Theorem~\ref{thm:Main} is nearly tight in terms of VC-dimension. The key is the VC-dimension upper bound of ReLU FNNs  in \cite{pmlr-v65-harvey17a} will lead to a contradiction if our approximation is not optimal. This idea was used in \cite{yarotsky2017} to prove its tightness for ReLU FNNs of width $\calO(d)$ and depth sufficiently large to approximate smooth functions. 

Let us first present the definitions of VC-dimension and related concepts. Let  $H$ be a class of functions mapping from a general domain $\mathcal{X}$ to $\{0,1\}$. 
We say $H$ shatters the set $\{\bmx_1,\bmx_2,\cdots,\bmx_m\}\subseteq \mathcal{X}$ if
\begin{equation*}
\Big| \Big\{\big[h(\bmx_1),h(\bmx_2),\cdots,h(\bmx_m)\big]^T\in \{0,1\}^m: h\in H\Big\}\Big|=2^m,
\end{equation*}
where $|\cdot|$ means the size of a set. This equation means, given any $\theta_i\in \{0,1\}$ for $i=1,2,\cdots,m$, there exists $h\in H$ such that
$h(\bmx_i)=\theta_i$ for all $i$. For a general function set $\scrF$ mapping from $\mathcal{X}$ to $\R$, we say $\scrF$ shatters $\{\bmx_1,\bmx_2,\cdots,\bmx_m\}\subseteq \mathcal{X}$ if $\calT\circ \scrF$ does, where \begin{equation*}
		\calT(t)\coloneqq \genfrac{\{}{.}{0pt}{0}{1,\ t\ge 0,}{0, \ t< 0\phantom{,}} \quad \tn{and}\quad \calT\circ \scrF\coloneqq \{\calT\circ f: f\in \scrF\}.
\end{equation*}

For any $m\in \N^+$, we define the growth function of $H$ as 
\begin{equation*}
\Pi_H(m)\coloneqq \max_{\bmx_1,\bmx_2,\cdots,\bmx_m\in \mathcal{X}} \Big| \Big\{\big[h(\bmx_1),h(\bmx_2),\cdots,h(\bmx_m)\big]^T\in \{0,1\}^m: h\in H\Big\}\Big|.
\end{equation*}
	
\begin{definition}[VC-dimension]
     Let $H$ be a class of functions from $\mathcal{X}$ to $\{0,1\}$.
    The VC-dimension of $H$, denoted by  $\vcd(H)$, is the size of the largest shattered set, namely, 
    \begin{equation*}
        \vcd(H)\coloneqq \sup \Big(\{0\}\bigcup\big\{ m\in\N^+ : \Pi_H(m)=2^m\big\}\Big).
    \end{equation*}
		Let $\scrF$ be a class of functions from $\mathcal{X}$ to $\R$. The VC-dimension of $\scrF$, denoted by $\vcd(\scrF)$, is defined by $\vcd(\scrF)\coloneqq\vcd(\calT\circ\scrF)$, where
	\begin{equation*}
	    \calT(t)\coloneqq \genfrac{\{}{.}{0pt}{0}{1,\ t\ge 0,}{0, \ t< 0\phantom{,}} \quad \tn{and}\quad \calT\circ \scrF\coloneqq \{\calT\circ f: f\in \scrF\}.
	\end{equation*}
	In particular, the expression ``VC-dimension of a network (architecture)'' means the VC-dimension of the function set that consists of all functions implemented by this network (architecture).
\end{definition}

Recall that  $\Csub$ denotes the unit ball of $C^s([0,1]^d)$.
Theorem~\ref{thm:linkVcdRate2} below shows that the best possible approximation error  of functions in $\Csub$ approximated by functions in $\scrF$ is bounded by a formula characterized by $\vcd(\scrF)$.

\begin{theorem}
	\label{thm:linkVcdRate2}
	Given any $s,d\in \N^+$, there exists a (small) positive constant $C_{s,d}$ determined by $s$ and $d$ such that: For any $\varepsilon>0 $ and a function set $\scrF$ with all elements defined on $[0,1]^d$, if $\vcd(\scrF)\ge 1$ and
	\begin{equation}
	\label{eq:distLipScrF2}
	\inf_{\phi\in \scrF}\|\phi-f\|_{L^\infty([0,1]^d)}\le \varepsilon\quad \tn{for any $f\in \Csub$,}
	\end{equation}
	then $\vcd(\scrF)\ge C_{s,d}\, \varepsilon^{-d/s}$.\,\footnote{In fact, $C_{s,d}$ can be expressed by $s$ and $d$ with a \textbf{explicitly} formula as we remark in the proof of this theorem. However, the formula may be very complicated.}
\end{theorem}

This theorem demonstrates the connection between the VC-dimension of $\scrF$ and the approximation error  using elements of $\scrF$ to approximate functions in $\Csub$.  
To be precise, the best possible approximation error  is controlled by $\vcd(\scrF)^{-s/d}$ up to a constant.
It is shown in \cite{pmlr-v65-harvey17a} that the VC-dimension of ReLU FNNs with a fixed architecture with $W$ parameters and $L$ layers has an upper bound $\calO(WL\ln W)$. It follows that the VC-dimension of ReLU FNNs with width $N$ and depth $L$ is bounded by $\calO(N^2L\cdot L \cdot\ln(N^2L))\le\calO(N^2L^2\ln(NL))$. 
That is, $\vcd(\scrF)\le  \calO(N^2L^2\ln(NL)) $, where 
\[\scrF=\NNF(\NNinput=d\NNspace\NNwidth \le N\NNspace\NNdepth\le L \NNspace\NNoutput=1).\]
Hence, the approximation error  of functions in $\Csub$, approximated by ReLU FNNs with width $N$ and depth $L$, has a lower bound
\[C(s,d)\cdot\big(N^2L^2 \ln (NL) \big)^{-s/d}\] 
for some positive constant $C(s,d)$ determined by $s$ and $d$.
When the width and depth become $\calO(N\ln N)$ and $\calO(L\ln L)$, respectively, the lower bound of the approximation error  becomes 
\[C(s,d)\cdot\big(N^2L^2 (\ln N)^3 (\ln L)^3 \big)^{-s/d} \]
for some positive constant $C(s,d)$ determined by $s$ and $d$.
These two lower bounds mean that our approximation errors in Theorem~\ref{thm:Main} and Corollary~\ref{cor:approxSmoothFun} are nearly optimal. 

Now let us present the detailed proof of  Theorem~\ref{thm:linkVcdRate2}.
\begin{proof}[Proof of Theorem~\ref{thm:linkVcdRate2}]
	To find a subset of $\scrF$ shattering $\calO(\varepsilon^{-d/s})$ points in $[0,1]^d$, we divided the proof into two steps.
	\begin{itemize}
		\item Construct $\{f_\chi:\chi\in \scrX\}\subseteq\Csub$ that scatters $\calO(\varepsilon^{-d/s})$ points, where $\scrX$ is a function set defined later.
		\item Design $\phi_\chi\in \scrF$, for each $\chi\in \scrX$, based on $f_\chi$ and Equation~\eqref{eq:distLipScrF2} such that	$\{\phi_\chi:\chi\in\scrX\}\subseteq\scrF$ also shatters $\calO(\varepsilon^{-d/s})$ points.
	\end{itemize}  
	The details of these two steps can be found below.
	
	\mystep{1}{Construct $\{f_\chi:\chi\in \scrX\}\subseteq\Csub$ that scatters $\calO(\varepsilon^{-d/s})$ points.}
	
	Let $K=\calO(\varepsilon^{-1/s})$ be an integer determined later and divide $[0,1]^d$ into $K^d$ non-overlapping sub-cubes $\{Q_{\bmbeta}\}_{\bmbeta}$ as follows: 
	\[Q_{\bmbeta}\coloneqq \big\{{\bm{x}}=[x_1,x_2,\cdots,x_d]^T\in[0,1]^d: x_i\in [\tfrac{\beta_i}{K},\tfrac{\beta_i+1}{K}]\tn{ for } i=1,2,\cdots,d\big\}\]
	for any index vector  ${\bmbeta}= [\beta_1,\beta_2,\cdots,\beta_d]^T\in \{0,1,\cdots,K-1\}^d$.

	There exists $\tildeg\in C^\infty(\R^d)$ such that $\tildeg(\bmzero)=1$ and $\tildeg(\bmx)=0$ for $\|\bmx\|_2\ge 1/3$.\footnote{In fact, such a function $\tildeg$ is called ``bump function''. An example can be attained by setting $\tildeg(\bmx)=C\exp(\tfrac{1}{\|3\bmx\|_2^2-1})$ if $\|\bmx\|_2< 1/3$ and $ \tildeg(\bmx)=0$ if $\|\bmx\|_2\ge 1/3$, where $C$ is a proper constant such that $\tildeg(\bmzero)=1$.} Then, $g\coloneqq \tildeg/\tildeC_{s,d}\in \Csub$ by setting $\tildeC_{s,d}\coloneqq \|\tildeg\|_{C^s([0,1]^d)}>0$.
	
	Define
	\begin{equation*}
		\scrX\coloneqq \big\{\chi: \chi \tn{   is  a map from   } \{0,1,\cdots,K-1\}^d \tn{   to   } \{-1,1\}\big\}
	\end{equation*}
	and \[g_\bmbeta\coloneqq K^{-s}g\big(K (\bmx-\bmx_{Q_\bmbeta})\big)\quad \tn{ for each $\bmbeta\in \{0,1,\cdots,K-1\}^d$,}\]
	where $\bmx_{Q_\beta}$ is the center of $Q_\bmbeta$.
	
	Next, for each $\chi\in \scrX$, we can define $f_\chi$ via
	\begin{equation*}
		f_\chi (\bm{x})\coloneqq \sum_{{\bm{\beta}}\in \{0,1,\cdots,K-1\}^d} \chi (\bmbeta)g_\bmbeta(\bmx).
	\end{equation*}
	Then $f_\chi\in \Csub$ for each $\chi\in \scrX$, since it satisfies the following two conditions. 
	\begin{itemize}
		\item By the definition of $g_\bmbeta$ and $\chi$, we have
		\begin{equation*}
			\big\{\bmx: \chi(\bmbeta)g_\bmbeta(\bmx)\neq 0\big\}\subseteq \ball(\bmx_{Q_\bmbeta},\tfrac{1}{3K}) \subseteq\tfrac23 Q_\bmbeta\quad \tn{for each $\bmbeta\in \{0,1,\cdots,K-1\}^d$,} 
		\end{equation*}
		 which implies that $f_\chi\in C^\infty([0,1]^d)$.
		
		\item For any $\bmx\in Q_\bmbeta$,  $\bmbeta\in \{0,1,\cdots,K-1\}^d$, and $\bmalpha\in \N^d$ with $\|\bmalpha\|_1\le s$, 
		\begin{equation*}
			\partial ^\bmalpha f_\chi(\bmx)=\chi(\bmbeta)\partial^\bmalpha g_\bmbeta(\bmx)=K^{-s}\chi(\bmbeta)K^{\|\alpha\|_1}\partial^\bmalpha g\big(K(\bmx-\bmx_\bmbeta)\big),
		\end{equation*} 
		from which we
		 deduce $|\partial^\bmalpha f_\chi(\bmx)|=\big|K^{-(s-\|\alpha\|_1)} \partial^\bmalpha g\big(K(\bmx-\bmx_\bmbeta)\big)\big|\le 1$.
	\end{itemize}
	It is easy to check that  $ \{f_\chi:\chi\in \scrX\}\subseteq \Csub$ 
	can shatter
	$K^d=\calO(\varepsilon^{-d/s})$
	points in $[0,1]^d$.

	\mystep{2}{Construct $\{\phi_\chi:\chi\in\scrX\}$ that also scatters $\calO(\varepsilon^{-d/s})$ points.}

	By Equation~\eqref{eq:distLipScrF2}, for each $\chi\in\scrX$, there exists $\phi_\chi \in \scrF$ such that 
	\begin{equation*}
		\|\phi_\chi-f_\chi\|_{L^\infty([0,1]^d)}\le \varepsilon+\varepsilon/2.
	\end{equation*}
	Let $\mu(\cdot)$ denote the Lebesgue measure of a set. Then, for each $\chi\in \scrX$, there exists $\calH_\chi\subseteq [0,1]^d$ with $\mu(\calH_\chi)=0$ such that
	\begin{equation*}
		|\phi_\chi(\bmx)-f_\chi(\bmx)|\le \tfrac{3}{2}\varepsilon\quad \tn{for any $\bmx\in [0,1]^d\backslash \calH_\chi$}.
	\end{equation*}
	Set $\calH=\bigcup_{\chi\in \scrX} \calH_\chi$; then we have $\mu(\calH)=0$ and 
	\begin{equation}
		\label{eq:f-Phi}
		|\phi_\chi(\bmx)-f_\chi(\bmx)|\le \tfrac{3}{2}\varepsilon\quad \tn{for any $\chi\in \scrX$ and $\bmx\in [0,1]^d\backslash \calH$}.
	\end{equation}

	Clearly, there exists $r\in (0,1)$ such that
	\begin{equation*}
		g_\bmbeta(\bmx)\ge \tfrac12 g_\bmbeta(\bmx_{Q_\bmbeta})>0\quad \tn{for any $\bmx\in rQ_\bmbeta$,}
	\end{equation*}
	where $\bmx_{Q_\bmbeta}$ is the center of $Q_\bmbeta$.
	
	Note that $(rQ_\bmbeta)\backslash\calH$ is not empty, since $\mu\big((rQ_\bmbeta)\backslash\calH\big)>0$ for each $\bmbeta$. Then, for any $\chi\in \scrX$ and $\bmbeta\in\{0,1,\cdots,K-1\}^d$, there exists $\bmx_\bmbeta\in (rQ_\bmbeta)\backslash\calH$ such that
	\begin{equation}
		\label{eq:fLB2}
		|f_\chi (\bm{x}_\bmbeta)|=|g_{\bmbeta}(\bm{x}_\bmbeta)|\ge \tfrac{1}{2}|g_{\bmbeta} (\bm{x}_{Q_{\bmbeta}})|=\tfrac{1}{2}K^{-s}g(\bmzero)=\tfrac{1}{2}K^{-s}/\tildeC_{s,d}\ge 2\varepsilon,
	\end{equation}
	where the last inequality is attained by setting $K=\lfloor (4 \varepsilon \tildeC_{s,d})^{-1/s}\rfloor$. Note that it is necessary to verify $K\neq 0$; we do this later in the proof.
	
	By Equations~\eqref{eq:f-Phi} and \eqref{eq:fLB2},  we have, for each $\bmbeta\in \{0,1,\cdots,K-1\}^d$ and each $\chi\in \scrX$,	
	\begin{equation*}
		|f_\chi (\bmx_{\bmbeta})|\ge 2 \varepsilon> \tfrac{3}{2}\varepsilon\ge  |f_\chi (\bmx_{\bmbeta})-\phi_\chi (\bmx_{\bmbeta})|,
	\end{equation*}		
	implying  $f_\chi (\bmx_{\bmbeta})$ and $\phi_\chi (\bmx_{\bmbeta})$ have the same sign. Then  $ \{\phi_\chi:\chi\in\scrX\}$ shatters $\big\{\bmx_\bmbeta:\bmbeta\in \{0,1,\cdots,K-1\}^d\big\}$ since $\{f_\chi:\chi\in\scrX\}
	$ shatters $\big\{\bmx_\bmbeta:\bmbeta\in \{0,1,\cdots,K-1\}^d\big\}$. Hence, 
	\begin{equation*}
		\tn{VCDim}(\scrF)\ge \tn{VCDim}\big( \{\phi_\chi:\chi\in\scrX\} \big)\geq K^d=\lfloor (4\varepsilon \tildeC_{s,d})^{-1/s}\rfloor^d\ge 2^{-d}(4\varepsilon \tildeC_{s,d})^{-d/s},
	\end{equation*}
	where the last inequality comes from the fact that $\lfloor x\rfloor \ge x/2$ for any $x\in [1,\infty)$. 
	
	Finally, by setting 
	\[C_{s,d}= 2^{-d}(4\tildeC_{s,d})^{-d/s}=2^{-d}\big(4\|\tildeg\|_{C^s([0,1]^d)}\big)^{-d/s},\] 
	we have \[\tn{VCDim}(\scrF)
	\ge 2^{-d}(4\varepsilon \tildeC_{s,d})^{-d/s}
	=2^{-d}(4 \tildeC_{s,d})^{-d/s}\varepsilon^{-d/s}
	= C_{s,d} \varepsilon^{-d/s}\]
	and 
	\[K=\lfloor (4 \varepsilon \tildeC_{s,d})^{-1/s}\rfloor=
	\lfloor \varepsilon^{-1/s} (4  \tildeC_{s,d})^{-1/s}\rfloor
	=\lfloor  \varepsilon ^{-1/s} (2^d C_{s,d})^{1/d}\rfloor\ge 1,\] where the last inequality comes from the assumption $\varepsilon\le (2^dC_{s,d})^{s/d}$. Such an assumption is reasonable since 
	$\varepsilon> (2^dC_{s,d})^{s/d}$ is a trivial case, which implies
	\begin{equation*}
	    \tn{VCDim}(\scrF)\ge 1\ge 2^{-d}=  C_{s,d} \Big((2^dC_{s,d})^{s/d}\Big)^{-d/s}> C_{s,d} \varepsilon^{-d/s}. 
	\end{equation*}
	So we finish the proof.  
\end{proof}

\section{Proof of Theorem~\ref{thm:Gap}}
\label{sec:badRegionRemoval}
Intuitively speaking, Theorem~\ref{thm:Gap} shows that if a ReLU FNN can implement a function $g$ approximating the target function $f$ well except for the trifling region, then  we  can design a new ReLU network with a similar size to approximate $f$ well on the whole domain. For example, if $g$ approximates a one-dimensional continuous function $f$ well except for a region in $\R$ with a sufficiently small measure $\delta$, then $\middleValue\big(g(x+\delta),g(x),g(x-\delta)\big)$ can approximate $f$ well on the whole domain, where $\middleValue(\cdot,\cdot,\cdot)$ is a function returning the middle value of three inputs and can be implemented via a ReLU FNN as shown in Lemma~\ref{lem:approxMid}. This key idea is called the horizontal shift (translation) of $g$ in this paper.

\begin{lemma}
	\label{lem:approxMid}
	The middle value function $\middleValue(x_1,x_2,x_3)$ can be implemented by  a ReLU FNN  with width $14$ and depth $2$.
\end{lemma}
\begin{proof}
	Recall the fact that
	\begin{equation}\label{eq:maxeq1}
		x=\sigma(x)-\sigma(-x)\quad \tn{and}\quad |x|=\sigma(x)+\sigma(-x)\quad \tn{for any $x\in\R$.}
	\end{equation}
	Therefore,
	\begin{equation}\label{eq:maxeq2}
		\begin{split}
			\max(x,y)&=\frac{x+y+|x-y|}{2}\\
			&=\tfrac12\sigma(x+y)-\tfrac12\sigma(-x-y)+\tfrac12\sigma(x-y)+\tfrac12\sigma(-x+y),
		\end{split}
	\end{equation}
	for any $x,y\in \R$.
	Thus, $\max(x_1,x_2,x_3)$ can be implemented by the network shown in Figure~\ref{fig:maxx1x2x3}.
	\begin{figure}[!htp]
		\centering
		\includegraphics[width=0.87\textwidth]{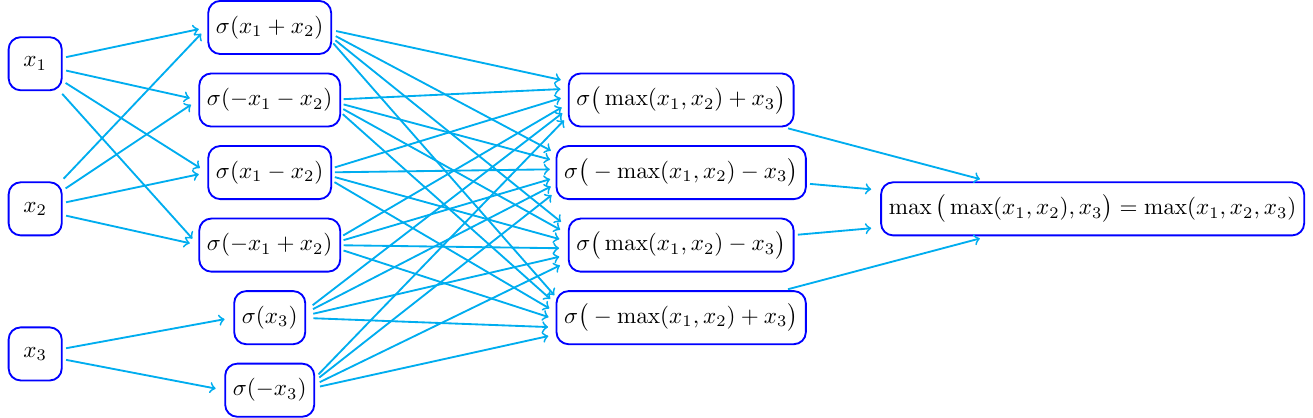}	
		\caption{An illustration of the network architecture implementing $\max(x_1,x_2,x_3)$ based on Equations~\eqref{eq:maxeq1} and \eqref{eq:maxeq2}. 
		}
		\label{fig:maxx1x2x3}
	\end{figure}

	Clearly, 
	\begin{equation*}
		\max(x_1,x_2,x_3)\in \NNF(\NNinput=3\NNspace\NNwidthvec=[6,4]).
	\end{equation*} 
	Similarly, we have
	\begin{equation*}
		\min(x_1,x_2,x_3)\in \NNF(\NNinput=3\NNspace\NNwidthvec=[6,4]).
	\end{equation*}
	
	It is easy to check that
	\begin{equation*}
		\begin{split}
			\middleValue(x_1,x_2,x_3)
			&=x_1+x_2+x_3-\max(x_1,x_2,x_3)-\min(x_1,x_2,x_3)\\
			&=\sigma(x_1+x_2+x_3)-\sigma(-x_1-x_2-x_3)-\max(x_1,x_2,x_3)-\min(x_1,x_2,x_3).
		\end{split}     
	\end{equation*}
	Hence, 
	\begin{equation*}
		\middleValue(x_1,x_2,x_3)\in \NNF(\NNinput=3\NNspace\NNwidthvec=[14,10]).
	\end{equation*}
	That is,
	$\middleValue(x_1,x_2,x_3)$ can be implemented by a ReLU FNN with width $14$ and depth $2$. So we finish the proof.
\end{proof}

The next lemma shows a simple but useful property of the $\middleValue(x_1,x_2,x_3)$ function that helps to exclude poor approximation in the trifling region.

\begin{lemma}
    \label{lem:middleValue}
    For any $\varepsilon>0$, if at least two elements of $\{x_1,x_2,x_3\}$ are in $\ball(y,\varepsilon)$, then 
    $\middleValue(x_1,x_2,x_3)\in \ball(y,\varepsilon)$.
\end{lemma}
\begin{proof}
    Without loss of generality, we may assume $x_1,x_2\in \ball(y,\varepsilon)$ and $x_1\le x_2$.
    Then the proof can be divided into three cases.
    \begin{enumerate}
        \item If $x_3< x_1$, then $x_3<x_1\le x_2$, implying $\middleValue(x_1,x_2,x_3)=x_1\in \ball(y,\varepsilon)$.
        \item If $x_1\le x_3\le x_2$, then $\middleValue(x_1,x_2,x_3)=x_3\in \ball(y,\varepsilon)$ since $y-\varepsilon\le x_1\le x_3\le x_2\le y+\varepsilon$.
        \item If $x_2<x_3$, then $x_1\le x_2<x_3$, implying $\middleValue(x_1,x_2,x_3)=x_2\in \ball(y,\varepsilon)$.
    \end{enumerate}
    So we finish the proof.
\end{proof}

\vspace{10pt}

Next, given a function $g$ approximating $f$ well on $[0,1]$ except for the trifling region, Lemma~\ref{lem:oneDimExtension} below shows how to use the $\middleValue(x_1,x_2,x_3)$ function to construct a new function $\phi$ uniformly approximating $f$ well on $[0,1]$, leveraging the useful property of $\middleValue(x_1,x_2,x_3)$ in Lemma~\ref{lem:middleValue}.

\begin{lemma}
	\label{lem:oneDimExtension}
	Given any $\varepsilon>0$, $K\in \N^+$, and $\delta\in(0,\tfrac{1}{3K}]$,
	assume $f\in C([0,1])$ and $g:\R\to \R$ is a general function with
	\begin{equation}\label{eq:g:in:f:ball} 
		|g(x)-f(x)|\le \varepsilon, \  \tn{i.e.,}\ g(x)\in \ball\big(f(x),\varepsilon\big)   \quad \tn{for any $x\in [0,1]\backslash \Omega([0,1],K,\delta)$}.
	\end{equation}
	Then 
	\begin{equation*}
		|\phi(x)-f(x)|\le \varepsilon+\omega_f(\delta)\quad \tn{for any $x\in [0,1]$},
	\end{equation*}
	where 
	\begin{equation*}
		\phi(x)\coloneqq  \middleValue\big(g(x-\delta),g(x),g(x+\delta)\big)\quad \tn{for any $x\in \R$.}
	\end{equation*}
\end{lemma}
\begin{proof}
	Divide $[0,1]$ into $K$ small intervals denoted by $Q_k=[\tfrac{k}{K},\tfrac{k+1}{K}]$ for $k=0,1,\cdots,K-1$. For each $k\in \{0,1,\cdots,K-1\}$, we further divide $Q_k$ into four small closed intervals as shown in Figure~\ref{fig:Qk}, i.e.,
	\begin{equation*}
		Q_k=Q_{k,1}\bigcup Q_{k,2}\bigcup Q_{k,3}\bigcup Q_{k,4},
	\end{equation*} 
	where $Q_{k,1}=[\tfrac{k}{K},\tfrac{k}{K}+\delta],\ Q_{k,2}=[\tfrac{k}{K}+\delta,\tfrac{k+1}{K}-2\delta],\ Q_{k,3}=[\tfrac{k+1}{K}-2\delta,\tfrac{k+1}{K}-\delta],$
	and $Q_{k,4}=[\tfrac{k+1}{K}-\delta,\tfrac{k+1}{K}]$.	
	\begin{figure}[H]
		\centering
		\includegraphics[width=0.805\textwidth]{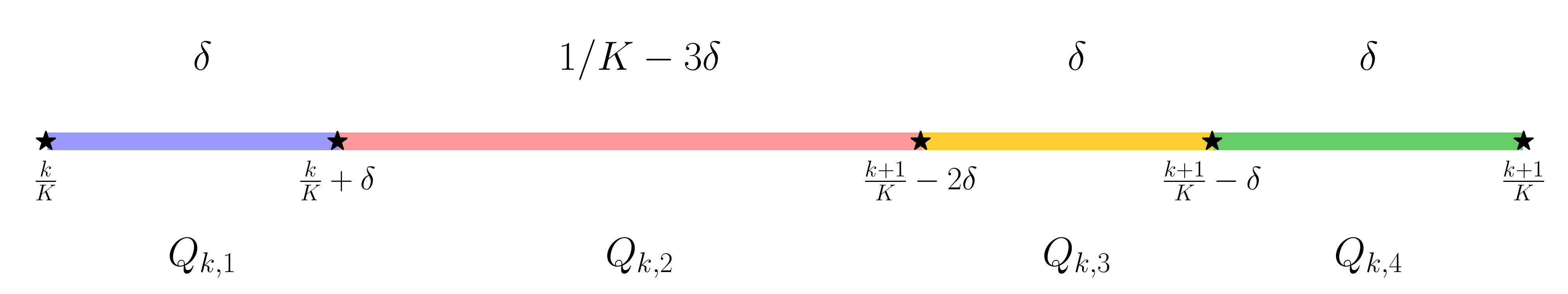}
		\caption{An illustration of $Q_{k,i}$ for $i=1,2,3,4$.}
		\label{fig:Qk}
	\end{figure}
	
	It is easy to verify that
	\begin{itemize}
	    \item $Q_{k,i}\subseteq [0,1]\backslash \Omega([0,1],K,\delta)$ for $k=0,1,\cdots,K-1$ and $i=1,2,3$;
	    \item $Q_{K-1,4}\subseteq [0,1]\backslash \Omega([0,1],K,\delta)$.
	\end{itemize}

	To estimate the difference between $\phi(x)$ and $f(x)$,  we consider the following four cases of $x$ in $[0,1]$ for each $k\in \{0,1,\cdots,K-1\}$.
	
	\mycase{1}{$x\in Q_{k,1}$.} 
	
	If $x\in Q_{k,1}$, then $x\in[0,1]\backslash \Omega([0,1],K,\delta)$ and 
	\[x+\delta\in Q_{k,2}\bigcup Q_{k,3}\subseteq [0,1]\backslash \Omega([0,1],K,\delta).\]
	It follows from Equation~\eqref{eq:g:in:f:ball}  that 
	\begin{equation*}
		g(x)\in \ball\big(f(x),\varepsilon\big)\subseteq \ball\big(f(x),\varepsilon+\omega_f(\delta)\big)
	\end{equation*} and 
	\begin{equation*}
		g(x+\delta)\in \ball\big(f(x+\delta),\varepsilon\big)\subseteq \ball\big(f(x),\varepsilon+\omega_f(\delta)\big).
	\end{equation*}
	By Lemma~\ref{lem:middleValue}, we get 
	\begin{equation*}
		\middleValue\big(g(x-\delta),g(x),g(x+\delta)\big)\in \ball\big(f(x),\varepsilon+\omega_f(\delta)\big).
	\end{equation*}
	
	\mycase{2}{$x\in Q_{k,2}$.} If $x\in Q_{k,2}$, then \begin{equation*}
	    x-\delta,x,x+\delta
	    \in Q_{k,1}\bigcup  Q_{k,2} \bigcup Q_{k,3}
	    \subseteq [0,1]\backslash \Omega([0,1],K,\delta). 
	\end{equation*}
	It follows  from Equation~\eqref{eq:g:in:f:ball}  that
	\begin{equation*}
		g(x-\delta)\in \ball \big(f(x-\delta),\varepsilon \big)\subseteq \ball\big(f(x),\varepsilon+\omega_f(\delta)\big),
	\end{equation*}
	\begin{equation*}
		g(x)\in \ball \big(f(x),\varepsilon \big)\subseteq \ball\big(f(x),\varepsilon+\omega_f(\delta)\big),
	\end{equation*}
	and 
	\begin{equation*}
		g(x+\delta)\in \ball \big(f(x+\delta),\varepsilon \big)\subseteq \ball\big(f(x),\varepsilon+\omega_f(\delta)\big).
	\end{equation*}
	Then,  by Lemma~\ref{lem:middleValue}, we have
	\begin{equation*}
		\middleValue\big(g(x-\delta),g(x),g(x+\delta)\big)\in \ball\big(f(x),\varepsilon+\omega_f(\delta)\big).
	\end{equation*}
	
	\mycase{3}{$x\in Q_{k,3}$.}  If $x\in Q_{k,3}$, then $x\in [0,1]\backslash \Omega([0,1],K,\delta)$ and \[x-\delta\in Q_{k,1}\bigcup Q_{k,2}\subseteq [0,1]\backslash \Omega([0,1],K,\delta).\]
	It follows  from Equation~\eqref{eq:g:in:f:ball}  that
	\begin{equation*}
		g(x)\in \ball\big(f(x),\varepsilon\big)\subseteq \ball\big(f(x),\varepsilon+\omega_f(\delta)\big)
	\end{equation*} and 
	\begin{equation*}
		g(x-\delta)\in \ball\big(f(x-\delta),\varepsilon\big)\subseteq \ball\big(f(x),\varepsilon+\omega_f(\delta)\big).
	\end{equation*}
	By Lemma~\ref{lem:middleValue}, we get 
	\begin{equation*}
		\middleValue\big(g(x-\delta),g(x),g(x+\delta)\big)\in \ball\big(f(x),\varepsilon+\omega_f(\delta)\big).
	\end{equation*}
	
	\mycase{4}{$x\in Q_{k,4}$.}  If $x\in Q_{k,4}$, we can divide this case into two sub-cases.
	\begin{itemize}
		\item   If $k\in \{0,1,\cdots,K-2\}$, then $x-\delta\in Q_{k,3}\in [0,1]\backslash \Omega([0,1],K,\delta)$ and $x+\delta\in Q_{k+1,1}\subseteq [0,1]\backslash \Omega([0,1],K,\delta)$. It follows  from Equation~\eqref{eq:g:in:f:ball}  that
		\begin{equation*}
			g(x-\delta)\in \ball\big(f(x-\delta),\varepsilon\big)\subseteq \ball\big(f(x),\varepsilon+\omega_f(\delta)\big)
		\end{equation*} and 
		\begin{equation*}
			g(x+\delta)\in \ball\big(f(x+\delta),\varepsilon\big)\subseteq \ball\big(f(x),\varepsilon+\omega_f(\delta)\big).
		\end{equation*}
		By Lemma~\ref{lem:middleValue}, we get 
		\begin{equation*}
			\middleValue\big(g(x-\delta),g(x),g(x+\delta)\big)\in \ball\big(f(x),\varepsilon+\omega_f(\delta)\big).
		\end{equation*}
		
		\item If $k=K-1$, then $x\in Q_{k,4}=Q_{K-1,4}\subseteq [0,1]\backslash \Omega([0,1],K,\delta)$ and $x-\delta\in Q_{k,3}\subseteq [0,1]\backslash \Omega([0,1],K,\delta)$. It follows  from Equation~\eqref{eq:g:in:f:ball}  that
		\begin{equation*}
			g(x)\in \ball\big(f(x),\varepsilon\big)\subseteq \ball\big(f(x),\varepsilon+\omega_f(\delta)\big)
		\end{equation*} and 
		\begin{equation*}
			g(x-\delta)\in \ball\big(f(x-\delta),\varepsilon\big)\subseteq \ball\big(f(x),\varepsilon+\omega_f(\delta)\big).
		\end{equation*}
		By Lemma~\ref{lem:middleValue}, we get 
		\begin{equation*}
			\middleValue\big(g(x-\delta),g(x),g(x+\delta)\big)\in \ball\big(f(x),\varepsilon+\omega_f(\delta)\big).
		\end{equation*}
	\end{itemize} 
	Since $[0,1]=\bigcup_{k=0}^{K-1}\Big(\bigcup_{i=1}^{4} Q_{k,i}\Big)$, we have 
	\begin{equation*}
		\middleValue\big(g(x-\delta),g(x),g(x+\delta)\big)\in \ball\big(f(x),\varepsilon+\omega_f(\delta)\big)\quad \tn{for any $x\in [0,1]$.}
	\end{equation*}
	Recall that $\phi(x)=\middleValue\big(g(x-\delta),g(x),g(x+\delta)\big)$. Then we have
	\begin{equation*}
		|\phi(x)-f(x)|\le \varepsilon+\omega_f(\delta)\quad \tn{for any $x\in [0,1]$.}
	\end{equation*}
	So we finish the proof.
\end{proof}

The next lemma below extend Lemma~\ref{lem:oneDimExtension} to the multidimensional case.
\begin{lemma}
	\label{lem:dDimExtension}
	Given any $\varepsilon>0$, $K\in \N^+$, and $\delta\in (0,\tfrac{1}{3K}]$,
	assume $f\in C([0,1]^d)$ and $g:\R^d\to\R$ is a general function with 
	\begin{equation*}
		|g(\bmx)-f(\bmx)|\le \varepsilon,\ \tn{i.e.,}\ g(\bmx)\in\ball\big(f(\bmx),\varepsilon\big)\quad \tn{for any $\bmx\in [0,1]^d\backslash \Omega([0,1]^d,K,\delta)$.}
	\end{equation*} 
	Then 
	\begin{equation*}
		|\phi(\bmx)-f(\bmx)|\le \varepsilon+d\cdot\omega_f(\delta)\quad \tn{for any $\bmx\in [0,1]^d$,}
	\end{equation*}
	where $\phi\coloneqq  \phi_d$ is defined by induction through
	\begin{equation}
		\label{eq:phiInduction}
		\phi_{i+1}(\bmx)\coloneqq  \middleValue\big(\phi_{i}(\bmx-\delta\bme_{i+1}),\phi_{i}(\bmx),\phi_{i}(\bmx+\delta\bme_{i+1})\big)\quad \tn{for $i=0,1,\cdots,d-1$,}
	\end{equation}
	where $\phi_0=g$ and $\{\bme_i\}_{i=1}^d$ is the standard basis in $\mathbb{R}^d$. 	 
\end{lemma}
\begin{proof}
	For $\ell=0,1,\cdots,d$, we define 
	\begin{equation*}
		E_\ell\coloneqq  \bigg\{\bmx=[x_1,x_2,\cdots,x_d]^T: x_i\in \left\{\begin{smallmatrix*}[l]
			[0,1],&\tn{if } i\le \ell,\\
			[0,1]\backslash \Omega([0,1],K,\delta),&\tn{if } i>\ell
		\end{smallmatrix*}\right. \bigg\}.
	\end{equation*}
	Clearly, $E_0=[0,1]^d\backslash \Omega([0,1]^d,K,\delta)$ and $E_d=[0,1]^d$. See Figure~\ref{fig:Eell} for the illustrations of $E_\ell$ for $\ell=0,1,\cdots,d$ when $K=4$ and $d=2$.
	\begin{figure}[!htp]
		\centering
		\begin{subfigure}{0.32\textwidth}
			\centering
			\includegraphics[width=0.875\textwidth]{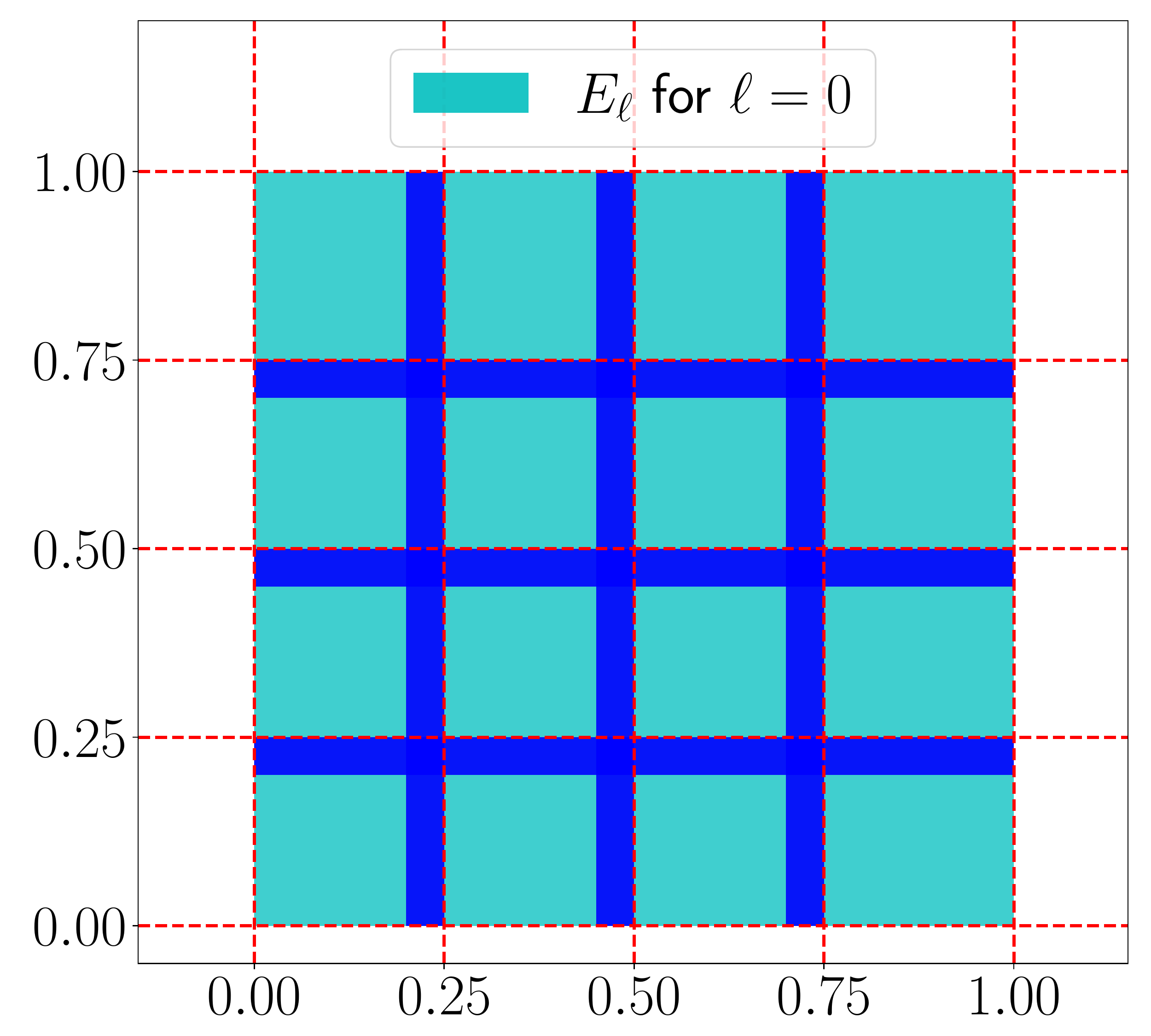}
		\end{subfigure}
		\begin{subfigure}{0.32\textwidth}
			\centering
			\includegraphics[width=0.875\textwidth]{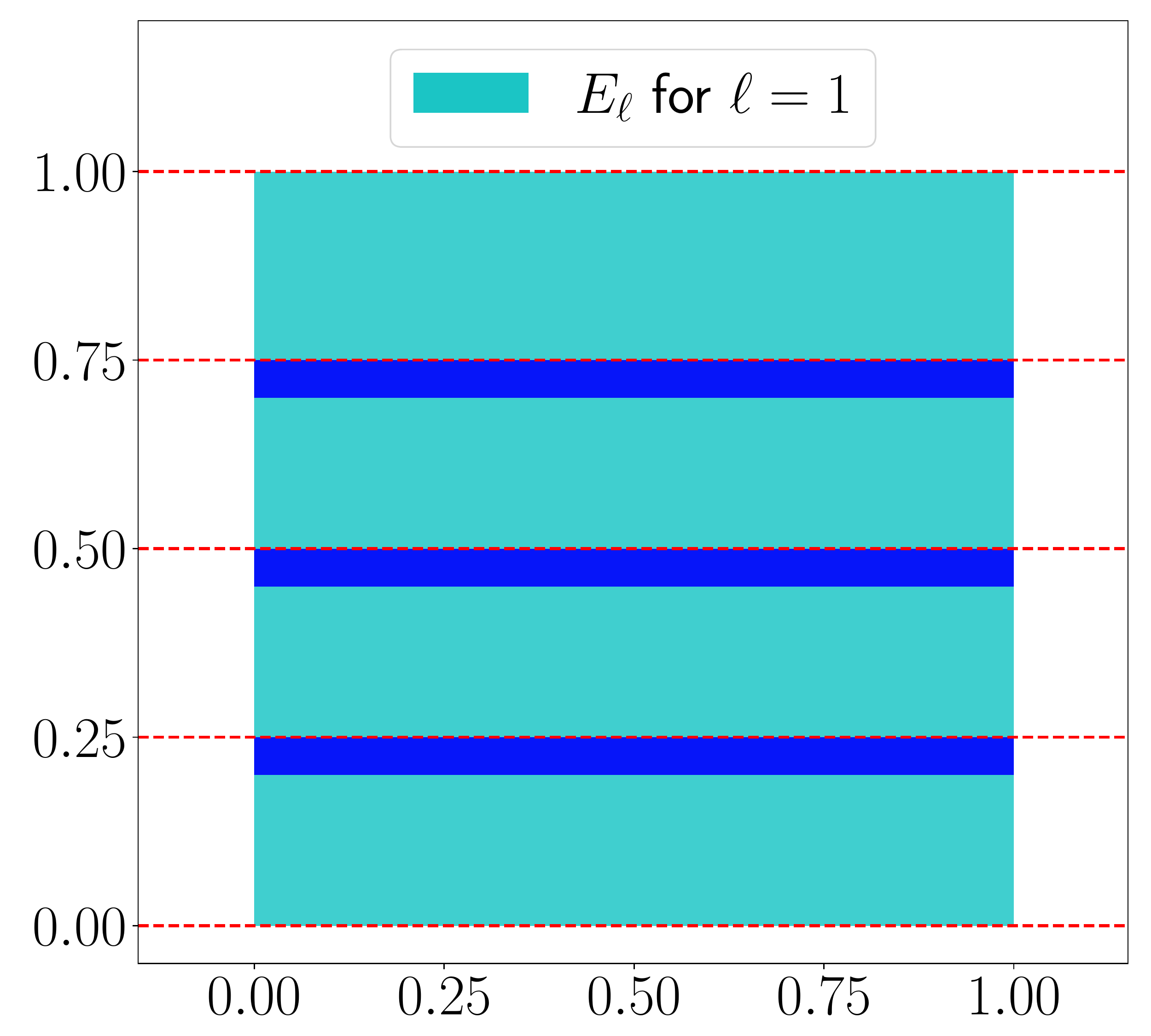}
		\end{subfigure}
		\begin{subfigure}{0.32\textwidth}
			\centering
			\includegraphics[width=0.875\textwidth]{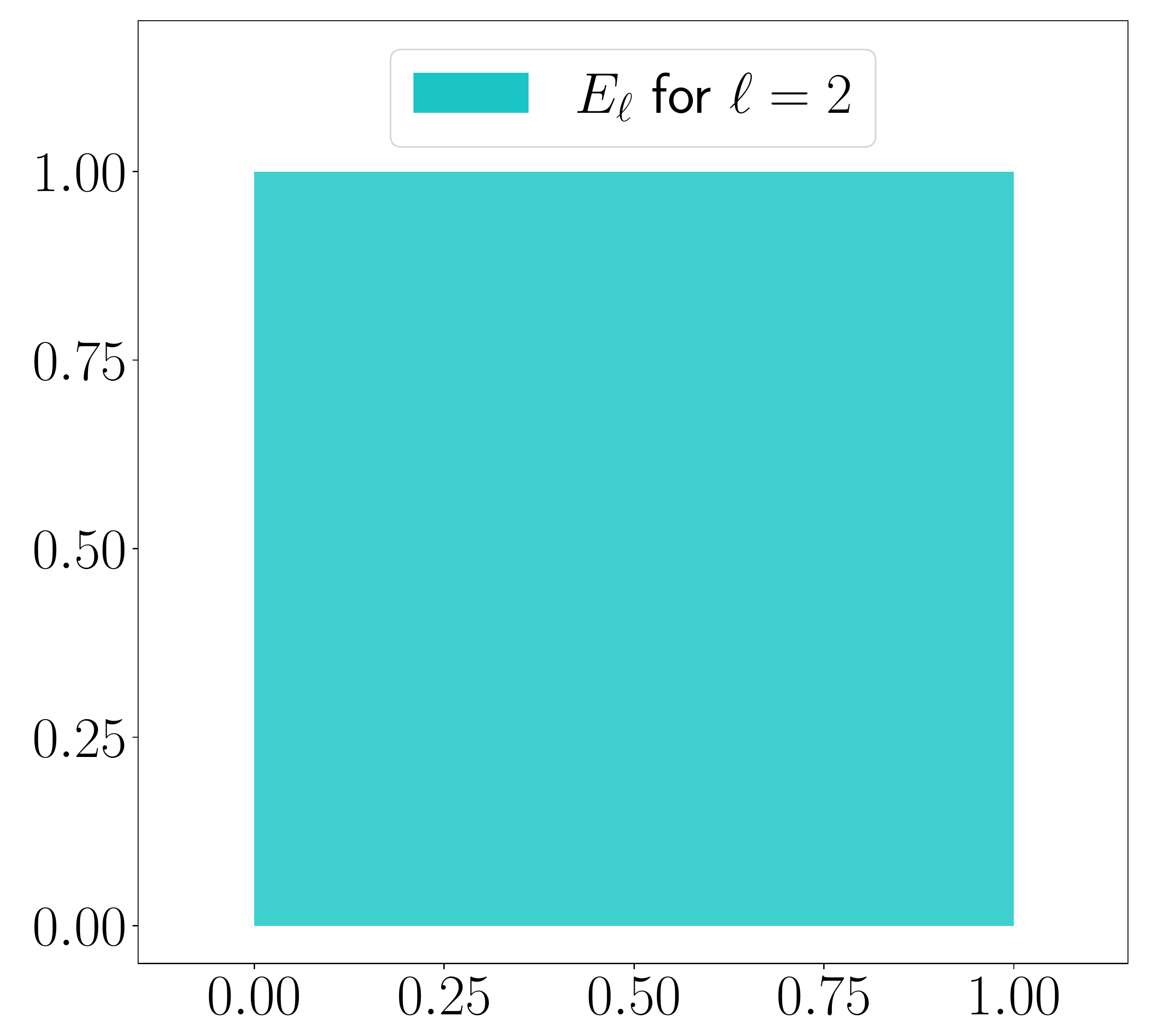}
		\end{subfigure}
		\caption{Illustrations of $E_\ell$ for $\ell=0,1,2$ when $K=4$ and $d=2$.}
		\label{fig:Eell}
	\end{figure}
	
	We would like to construct a sequence of functions $\phi_0,\phi_1,\cdots,\phi_{d}$ by induction, based on  Equation~\eqref{eq:phiInduction}, such that, for each  $\ell\in\{0,1,\cdots,d$\},
	\begin{equation}
		\label{eq:induction}
		\phi_\ell(\bmx)\in \ball\big(f(\bmx),\varepsilon+\ell\cdot\omega_f(\delta)\big)\quad \tn{for any $\bmx\in E_\ell$.}
	\end{equation}
	
	Let us first consider the case $\ell=0$. Note that $\phi_0=g$, $E_0= [0,1]^d\backslash \Omega([0,1]^d,K,\delta)$, and $|g(\bmx)-f(\bmx)|\le \varepsilon$ for any $\bmx\in [0,1]^d\backslash \Omega([0,1]^d,K,\delta)$. Then we have 
	\begin{equation*}
		\phi_0(\bmx)=g(\bmx)\in \ball\big(f(\bmx),\varepsilon\big)\quad \tn{for any $\bmx\in E_0$.}
	\end{equation*}
	That is, Equation~\eqref{eq:induction} is true for $\ell=0$.
	
	Now assume Equation~\eqref{eq:induction} is true for $\ell=i$. We will prove that it also holds for $\ell=i+1$. 
	By the hypothesis of induction, we have
	\begin{equation}
		\label{eq:ell=i}
		\phi_i(x_1,\cdots,x_i,t,x_{i+2},\cdots,x_d)\in \ball \big(f(x_1,\cdots,x_i,t,x_{i+2},\cdots,x_d),\varepsilon+i\cdot\omega_f(\delta)\big)
	\end{equation}
	for any  $x_1,\cdots,x_i\in [0,1]$ and $t,x_{i+2},\cdots,x_d\in [0,1]\backslash \Omega([0,1],K,\delta)$. 
	
	For fixed $x_1,\cdots,x_i\in [0,1]$ and $x_{i+2},\cdots,x_d\in [0,1]\backslash \Omega([0,1],K,\delta)$, denote \[\bmx^{[i]}:=[x_1,\cdots,x_i,x_{i+2},\cdots,x_d]^T\in[0,1]^{d-1}.\] Then define
	\begin{equation*}
		\psi_{\bmx^{[i]}}(t)\coloneqq  \phi_i(x_1,\cdots,x_i,t,x_{i+2},\cdots,x_d)\quad \tn{ for any $t\in \R$}
	\end{equation*}
	and
	\begin{equation*}
		f_{\bmx^{[i]}}(t)\coloneqq  f(x_1,\cdots,x_i,t,x_{i+2},\cdots,x_d)\quad \tn{ for any $t\in \R$.}
	\end{equation*}
	
	It follows from Equation~\eqref{eq:ell=i} that
	\begin{equation*}
		\psi_{\bmx^{[i]}}(t)\in \ball \big(f_{\bmx^{[i]}}(t),\varepsilon+i\cdot\omega_f(\delta)\big)\quad \tn{for any $t\in [0,1]\backslash \Omega([0,1],K,\delta)$.}
	\end{equation*}
	Then by Lemma~\ref{lem:oneDimExtension} (set $g=\psi_{\bmx^{[i]}}$ and $f=f_{\bmx^{[i]}}$ therein), we get, for any $t\in [0,1]$,
	\begin{equation*}
		\begin{split}
			\middleValue\big(\psi_{\bmx^{[i]}}(t-\delta),\psi_{\bmx^{[i]}}(t),\psi_{\bmx^{[i]}}(t+\delta)\big)
			&\in \ball \Big(f_{\bmx^{[i]}}(t),\varepsilon+i\cdot\omega_f(\delta)+\omega_{f_{\bmx^{[i]}}}(\delta)\Big)\\
			&\subseteq \ball \big(f_{\bmx^{[i]}}(t),\varepsilon+(i+1)\omega_f(\delta)\big).
		\end{split}
	\end{equation*}
	That is, for any $x_{i+1}=t\in [0,1]$, 
	\begin{equation*}
		\begin{split}
			&\quad \middleValue\Big(\phi_i(x_1,\cdots,x_i,x_{i+1}-\delta,x_{i+2},
				\cdots,x_d),\phi_i(x_1,\cdots,x_d),
				\phi_i(x_1,\cdots,x_i,x_{i+1}+\delta,x_{i+2},\cdots,x_d)\Big)
			\\
			&\in \ball \Big(f(x_1,\cdots,x_d),\varepsilon+(i+1)\omega_f(\delta)\Big).  
		\end{split}
	\end{equation*}
	Note that $x_1,\cdots,x_i\in [0,1]$, $x_{i+1}=t\in [0,1]$, and $x_{i+2},\cdots,x_d\in[0,1]\backslash \Omega([0,1],K,\delta)$ are arbitrary. Thus, for any $\bmx\in E_{i+1}$, we have
	\begin{equation*}
		\middleValue\big(\phi_i(\bmx-\delta\bme_{i+1}),\phi_i(\bmx),\phi_i(\bmx+\delta\bme_{i+1})\big)\in \ball \big(f(\bmx),\varepsilon+(i+1)\omega_f(\delta)\big),
	\end{equation*}
	which implies
	\begin{equation*}
		\phi_{i+1}(\bmx)\in \ball \big(f(\bmx),\varepsilon+(i+1)\omega_f(\delta)\big)\quad \tn{for any $\bmx\in E_{i+1}$.}
	\end{equation*}
	So Equation~\eqref{eq:induction} is true for $\ell=i+1$, which means we finish the process of mathematical induction.
	
	By the principle of induction, we have
	\begin{equation*}
		\phi(\bmx):=\phi_{d}(\bmx)\in \ball \big(f(\bmx),\varepsilon+d\cdot\omega_f(\delta)\big)\quad \tn{for any $\bmx\in E_{d}=[0,1]^d$.}
	\end{equation*}
	Therefore,
	\begin{equation*}
		|\phi(\bmx)-f(\bmx)|\le \varepsilon+d\cdot\omega_f(\delta)\quad \tn{for any $\bmx\in [0,1]^d$,}
	\end{equation*}
	which means we finish the proof.
\end{proof}

With Lemma~\ref{lem:dDimExtension} in hand, we are ready to prove Theorem~\ref{thm:Gap}. 
\begin{proof}[Proof of Theorem~\ref{thm:Gap}]
	Set $\phi_0=\tildephi$ and define $\phi_i$ for $i\in \{1,2,\cdots,d\}$ by induction as follows:
\begin{equation*}
	\phi_{i+1}(\bmx)\coloneqq  \middleValue\big(\phi_{i}(\bmx-\delta\bme_{i+1}),\phi_{i}(\bmx),\phi_{i}(\bmx+\delta\bme_{i+1})\big)\quad \tn{for $i=0,1,\cdots,d-1$,}
\end{equation*}
where $\{\bme_i\}_{i=1}^d$ is the standard basis in $\mathbb{R}^d$.
Then by Lemma~\ref{lem:dDimExtension} with $\phi=\phi_d$, we have
\begin{equation*}
	|\phi(\bmx)-f(\bmx)|\le \varepsilon+d\cdot \omega_f(\delta)\quad \tn{for any $\bmx\in [0,1]^d$.}
\end{equation*}
It remains to determine the network architecture implementing $\phi=\phi_d$.
Clearly, $\phi_0=\tildephi\in\NNF(\NNwidth\le N\NNspace\NNdepth\le L)$ implies 
\begin{equation*}
	\phi_{0}(\cdot-\delta\bme_{1}),\phi_{0}(\cdot),\phi_{0}(\cdot+\delta\bme_{1})\in \NNF(\NNwidth\le N\NNspace\NNdepth\le L).
\end{equation*}
By defining a vector-valued function $\bmPhi_0:\R^d\to \R^3$ as \[\bmPhi_0(\bmx)\coloneqq\big(\phi_{0}(\bmx-\delta\bme_{1}),\phi_{0}(\bmx),\phi_{0}(\bmx+\delta\bme_{1})\big)\quad \tn{for any $\bmx\in \R^d$,}\]
we have $\bmPhi_0\in \NNF(\NNinput=d\NNspace\NNwidth\le 3N\NNspace\NNdepth\le L\NNspace\NNoutput=3)$. 
Recall that  $\middleValue(\cdot,\cdot,\cdot)\in\NNF(\NNwidth\le 14\NNspace\NNdepth\le 2)$ by Lemma~\ref{lem:approxMid}.
Therefore, $\phi_1=\min(\cdot,\cdot,\cdot)\circ\bmPhi_0$ can be implemented by a ReLU FNN with width $\max\{3N,14\}\le 3(N+4)$ and depth $L+2$.
Similarly, $\phi=\phi_d$ can be implemented by a ReLU FNN with width $3^d(N+4)$ and depth $L+2d$.
So we finish the proof.
\end{proof}

\section{Proof of Theorem~\ref{thm:MainGap}}
\label{sec:4}

In this section, we prove Theorem~\ref{thm:MainGap}, a weaker version of the main theorem of this paper (Theorem~\ref{thm:Main}) targeting a ReLU FNN constructed to approximate a smooth function outside the trifling region.
The main idea is to construct ReLU FNNs through Taylor expansions of smooth functions. We first discuss  the proof sketch in Section~\ref{sec:sketchProofOfThmOld} and give the detailed proof in Section~\ref{sec:proofOfThmOld}.

\subsection{Proof sketch of Theorem~\ref{thm:MainGap}}
\label{sec:sketchProofOfThmOld}

Set $K=\calO(N^{2/d}L^{2/d})$ and 
let $\Omega([0,1]^d,K,\delta)$ partition $[0,1]^d$ into $K^d$ cubes $Q_\bmbeta$ for $\bmbeta\in \{0,1,\cdots,K-1\}^d$.
As we shall see later, the introduction of the trifling region $\Omega([0,1]^d,K,\delta)$ can reduce the difficulty in constructing ReLU FNNs to achieve the optimal approximation error  simultaneously in width and depth, since it is only required to uniformly control the approximation error  outside the trifling region and there is no requirement for the ReLU FNN inside the trifling region.
In particular,
for each $\bmbeta=[\beta_1,\beta_2,\cdots,\beta_d]^T\in \{0,1,\cdots,K-1\}^d$, we define  $\bmx_\bmbeta\coloneqq \bmbeta/K$ and 
\begin{equation*}
	Q_\bmbeta\coloneqq \big\{\bmx=[x_1,x_2,\cdots,x_d]^T: x_i\in [\tfrac{\beta_i}{K},\tfrac{\beta_i+1}{K}-\delta\cdot \one_{ \{\beta_i\le K-2\}}] \tn{ for } i=1,2,\cdots,d  \big\}.
\end{equation*}
Clearly,  $[0,1]^d=\Omega([0,1]^d,K,\delta)\bigcup \big(\cup_{\bmbeta\in\{0,1,\cdots,K-1\}^d} Q_\bmbeta\big)$ and $\bmx_\bmbeta$ is the vertex of $Q_\bmbeta$ with minimum $\|\cdot\|_1$ norm. See Figure~\ref{fig:QthetaExample} for the illustrations of $Q_\bmbeta$ and $\bmx_\bmbeta$.
\begin{figure}[!htp]
	\centering
	\begin{minipage}{0.95\textwidth}
		\centering
		\begin{subfigure}[b]{0.4\textwidth}
			\centering
			\includegraphics[width=0.9\textwidth]{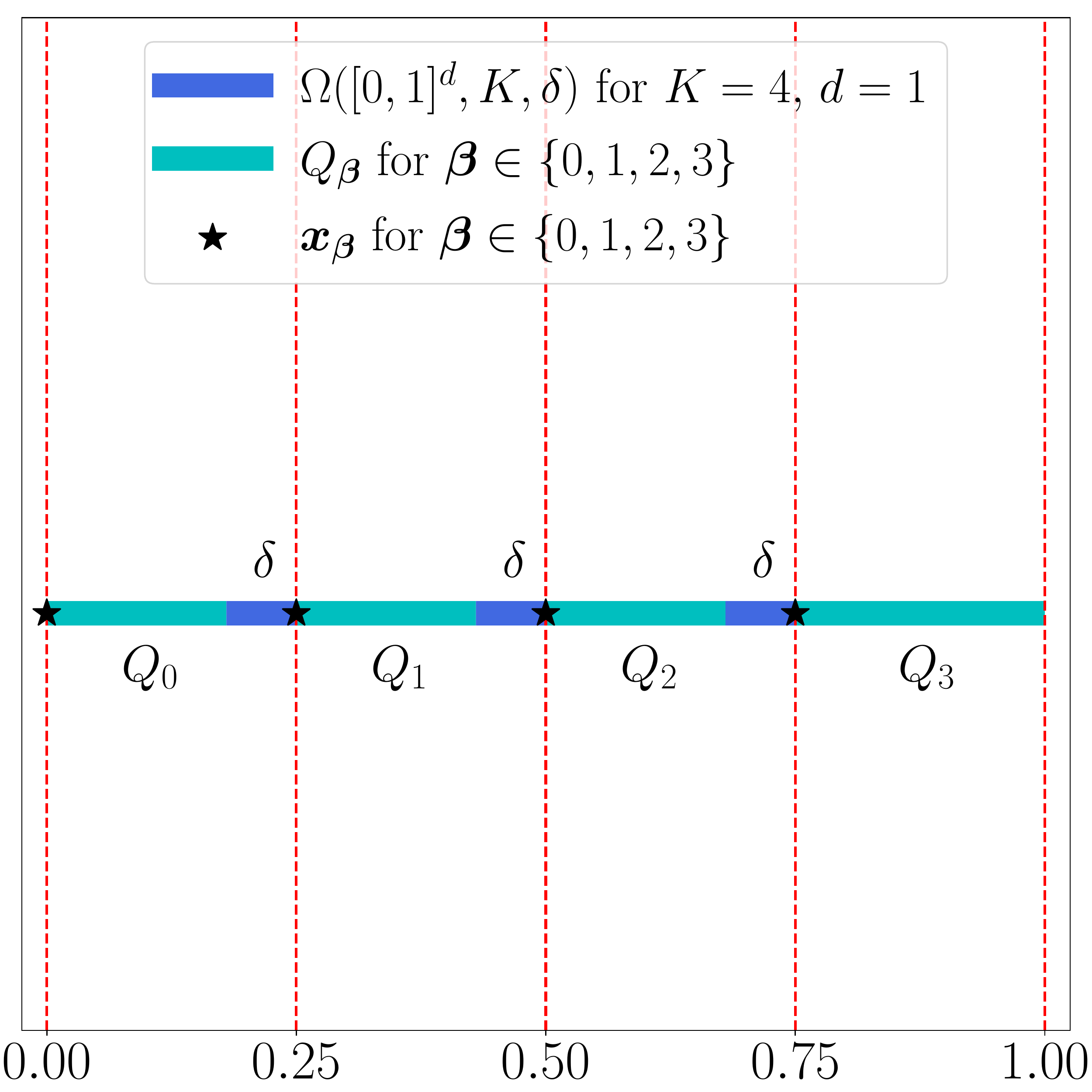}
			\subcaption{}
		\end{subfigure}
			\begin{minipage}{0.01\textwidth}
				\
			\end{minipage}
		\begin{subfigure}[b]{0.4\textwidth}
			\centering
			\includegraphics[width=0.9\textwidth]{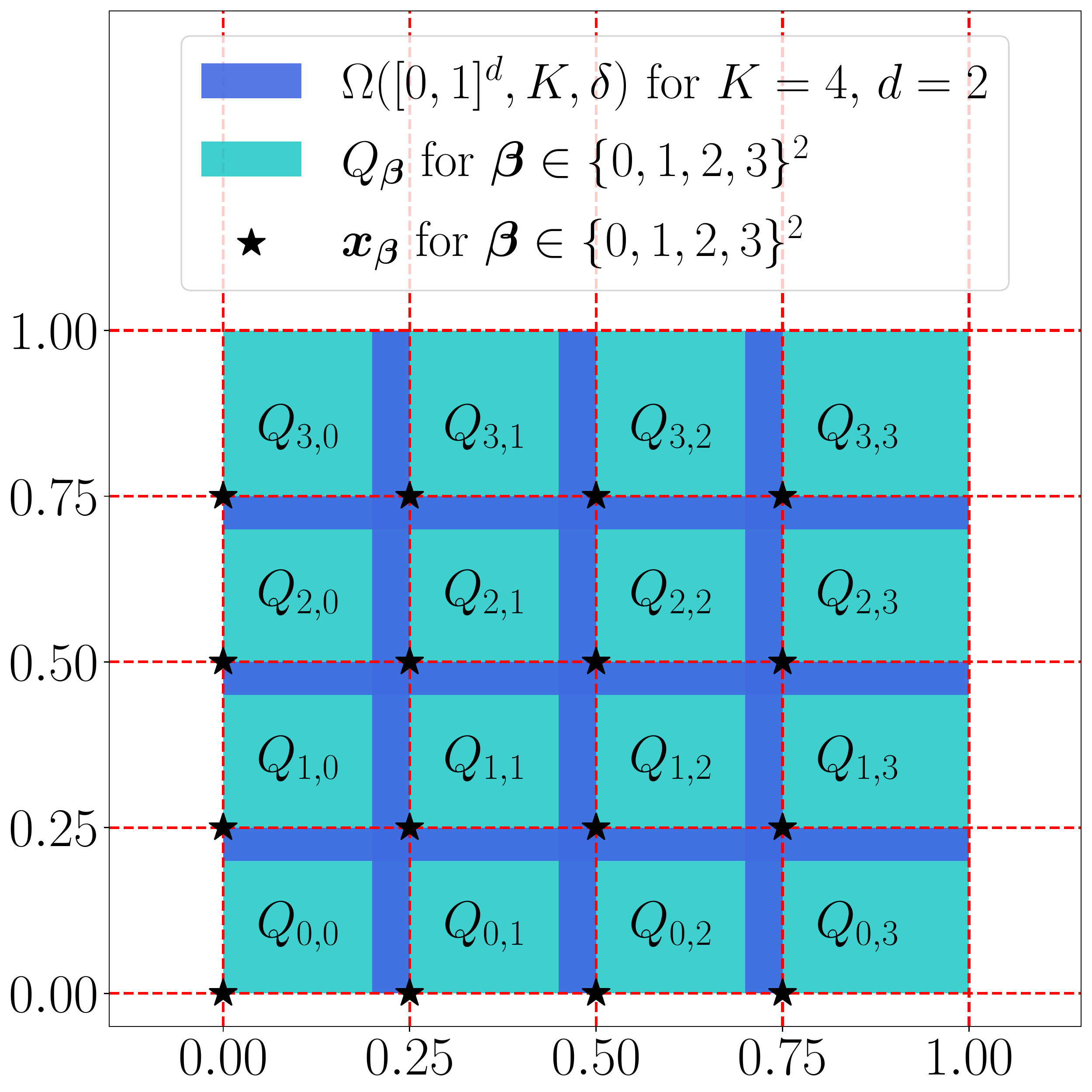}
			\subcaption{}
		\end{subfigure}
	\end{minipage}
	\caption{Illustrations of  $\Omega([0,1]^d,K,\delta)$,  $Q_\bmbeta$, and $\bmx_\bmbeta$ for $\bmbeta\in \{0,1,\cdots,K-1\}^d$. (a) $K=4$ and $d=1$. (b) $K=4$ and $d=2$. }
	\label{fig:QthetaExample}
\end{figure}

For any $\bmbeta\in \{0,1,\cdots,K-1\}^d$ and $\bmx\in Q_\bmbeta$,
there exists $\xi_\bmx\in (0,1)$ such that 
\begin{equation}\label{eqn:tle}
	f(\bmx)=\underbrace{\sum_{\|\bmalpha\|_1\le s-1} 
	\tfrac{\partial^\bmalpha f(\bmx_\bmbeta)}{\bmalpha !}\bmh^\bmalpha
	}_{\scrT_1}
	\   +  \
	\underbrace{\sum_{\|\bmalpha\|_1= s} 
	\tfrac{\partial^\bmalpha f(\bmx_\bmbeta+\xi_\bmx \bmh)}{\bmalpha !}\bmh^\bmalpha
	}_{\scrT_2}
	\eqqcolon \scrT_1+\scrT_2,\footnote{$\sum_{\|\bmalpha\|_1=s}$ is short for $\sum_{\|\bmalpha\|_1=s,\, \bmalpha\in \N^d}$. The same notation is used throughout this paper.}
\end{equation}
where $\bmh(\bmx)=\bmx-\bmx_\bmbeta=\bmx-{\bmbeta}/{K}$. Clearly, the magnitude of $\scrT_2$ is bounded by $\calO(K^{-s})=\calO(N^{-2s/d}L^{-2s/d})$. So we only need to construct a ReLU FNN with width $\calO(N\ln N)$ and depth $\calO(L\ln L)$ to approximate 
\[\scrT_1=\sum_{\|\bmalpha\|_1\le s-1} \tfrac{\partial^\bmalpha f(\bmx_\bmbeta)}{\bmalpha !}\bmh^\bmalpha\]
within an error  $\calO(N^{-2s/d}L^{-2s/d})$. To approximate $\scrT_1$ well by ReLU FNNs, we need three key steps as follows.
\begin{enumerate}[(i)]
    \item   \label{item:step1}
    Construct a ReLU FNN to implement a function $P_\bmalpha:\R^d\to \R$ approximating the polynomial $\bmh^\bmalpha$ well for each $\bmalpha\in\N^d$ with $\|\bmalpha\|_1\le s-1$. 

	\item \label{item:step2}
	Construct a ReLU FNN  to implement a vector-valued function $\bmPsi:\R^d\to\R^d$ projecting the whole cube $Q_\bmbeta$ to a point $\bmx_\bmbeta=\tfrac{\bmbeta}{K}$, i.e., $\bmPsi(\bmx)=\bmx_\bmbeta$ for any $\bmx\in Q_\bmbeta$ and  each $\bmbeta\in \{0,1,\cdots,K-1\}^d$.

	\item \label{item:step3}
	Construct a ReLU FNN to implement a function $\phi_\bmalpha:\R^d\to \R$ approximating $\partial ^\bmalpha f$ via solving a point fitting problem, i.e., $\phi_\bmalpha$ should fit $\partial ^\bmalpha f$ well at all points in $\big\{\bmx_\bmbeta:\bmbeta\in\{0,1,\cdots,K-1\}^d\big\}$ for each $\bmalpha\in\N^d$ with $\|\bmalpha\|_1\le s-1$. That is, for each $\bmalpha\in\N^d$ with $\|\bmalpha\|_1\le s-1$, we need to design $\phi_\bmalpha$ satisfying
	\begin{equation}\label{eq:step3:pointfit}
		\big|\phi_\bmalpha(\bmx_\bmbeta)-\partial ^\bmalpha f(\bmx_\bmbeta)\big|\le \calO(N^{-2s/d}L^{-2s/d})\quad \tn{for any $\bmbeta\in \{0,1,\cdots,K-1\}^d$.}
	\end{equation}
\end{enumerate}

We will establish three propositions corresponding to these three steps above. 
They will be applied to support the construction of the desired ReLU FNNs. Their proofs will be available in Section~\ref{sec:proofOfProposition}.

First, we establish  a general proposition, Proposition~\ref{prop:approxPolynomial} below, showing how to use ReLU FNNs to approximate multivariate  polynomials.
With Proposition~\ref{prop:approxPolynomial} in hand, Step~\eqref{item:step1} is straightforward.
\begin{proposition}
	\label{prop:approxPolynomial}
	Assume $P(\bmx)=\bmx^\bmalpha=x_1^{\alpha_1}x_2^{\alpha_2}\cdots x_d^{\alpha_d}$ for $\bmalpha\in \N^d$ with $\|\bmalpha\|_1\le k\in\N^+$. For any $N,L\in \N^+$, there exists a function $\phi$ implemented by a ReLU FNN with width $9(N+1)+k-1$ and depth $7k^2L$ such that
	\begin{equation*}
		|\phi(\bmx)-P(\bmx)|\le 9k(N+1)^{-7kL}\quad \tn{for any $\bmx\in [0,1]^d$.}
	\end{equation*}
\end{proposition}

Proposition~\ref{prop:approxPolynomial} shows that  ReLU FNNs with width $\calO(N)$ and depth $\calO(L)$ are able to  approximate   polynomials with an error $\calO(N^{-L})$. This reveals the power of depth in ReLU FNNs for approximating polynomials, from the perspective of function compositions. 
The starting point of a  good approximation of functions is to approximate polynomials with high accuracy.  In classical approximation theory, the approximation power of any numerical scheme depends on the degree of 
polynomials that can be locally reproduced. 
Being able to approximate polynomials by ReLU FNNs with high
accuracy  plays a vital role in the proof of Theorem~\ref{thm:Main}.
It is interesting to study whether there is any other function space with reasonable size, besides polynomial space, having  an exponential error $\calO(N^{-L})$ when approximated by ReLU FNNs. Obviously, the space of smooth functions is too big  due to the optimality of Theorem~\ref{thm:Main} as shown in Section~\ref{sec:optimalityOfMainThm}. 

Proposition~\ref{prop:approxPolynomial} can be generalized to the case of polynomials defined on an arbitrary hypercube $[a,b]^d$. Let us give an example for the polynomial $xy$ below. Its proof will be provided later in Section~\ref{sec:approxPolynomial}.

\begin{lemma}
    \label{lem:xyApproxAB}
    For any $N,L\in \N^+$ and $a,b\in \R$ with $a<b$, there exists a function  $\phi$ implemented by a ReLU FNN with width $9N+1$ and depth $L$ such that
    \begin{equation*}
    |\phi(x,y)-xy|\le 6(b-a)^2N^{-L}\quad \tn{for any $x,y\in [a,b]$.}
    \end{equation*}
\end{lemma}

Second, our goal is to construct a step function $\bmPsi$ mapping $\bmx\in Q_\bmbeta$ to $\bmx_\bmbeta=\tfrac{\bmbeta}{K}$ for any $\bmbeta\in\{0,1,\cdots,K-1\}^d$. We only need to approximate one-dimensional step functions, because in the multidimensional case we can simply set $\bmPsi(\bmx)=[\psi(x_1),\psi(x_2),\cdots,\psi(x_d)]^T$, where $\psi$ is a one-dimensional step function. Therefore, to implement Step~\eqref{item:step2}, we need to construct ReLU FNNs with width $\calO(N)$ and depth $\calO(L)$ to approximate one-dimensional step functions with $\calO(K)=\calO(N^{2/d}L^{2/d})$ ``steps'' as shown in Proposition~\ref{prop:approxStepFun} below.

\begin{proposition}
    \label{prop:approxStepFun}
	For any $N,L,d\in \N^+$ and $\delta\in(0, \tfrac{1}{3K}]$ with $K=\lfloor N^{1/d}\rfloor^2 \lfloor L^{2/d}\rfloor$, there exists a one-dimensional function $\phi$ implemented by a ReLU FNN  with width $4\lfloor N^{1/d}\rfloor +3$ and depth $4L+5$ such that
	\begin{equation*}
		\phi(x)=k\quad \tn{if $x\in [\tfrac{k}{K},\tfrac{k+1}{K}-\delta\cdot \one_{\{k\le K-2\}}]$\quad for $k=0,1,\cdots,K-1$.}
	\end{equation*}     
\end{proposition}

Next,  the aim of Step~\eqref{item:step3} is to construct  $\phi_\bmalpha$ implemented by a ReLU FNN such that Equation~\eqref{eq:step3:pointfit} holds for each $\bmalpha$. To this end, 
 we establish a proposition, Proposition~\ref{prop:pointsMatching} below,
to show that  ReLU FNNs with width $\calO(sN\ln N)$ and depth $\calO(L\ln L)$ can be constructed to fit $N^2L^2$ points within an error  $N^{-2s}L^{-2s}$.

\begin{proposition}
    \label{prop:pointsMatching}
	Given any $N,L,s\in \N^+$ and $\xi_i\in [0,1]$ for  $i=0,1,\cdots,N^2L^2-1$, there exists a function $\phi$ implemented by a ReLU FNN with width $16s(N+1)\log_2(8N)$ and depth $5(L+2)\log_2(4L)$ such that
	\begin{enumerate}[(i)]
		\item $|\phi(i)-\xi_i|\le N^{-2s}L^{-2s}$ for $i=0,1,\cdots,N^2L^2-1$;
		\item $0\le \phi(x)\le  1$ for any $x\in\R$.
	\end{enumerate}
\end{proposition}

The proofs of Propositions~\ref{prop:approxPolynomial},   \ref{prop:approxStepFun}, and \ref{prop:pointsMatching} can be found in Sections~\ref{sec:approxPolynomial},  \ref{sec:approxStepFun}, and \ref{sec:pointsMatching}, respectively.
The main ideas of proving Theorem~\ref{thm:Main} are summarized in Table~\ref{tab:constructiveProofShort}.

\begin{table}[ht]   
	\caption{A list of sub-networks for approximating smooth functions.
		Recall that $\bmh=\bmx-\bmPsi(\bmx)=\bmx-\bmx_\bmbeta$ for $\bmx\in Q_\bmbeta$.}
	\label{tab:constructiveProofShort}
	\centering  
	\resizebox{0.985\textwidth}{!}{ 
		\begin{tabular}{ccccc} 
			\toprule
			target function & function implemented by network & width & depth & approximation error  \\
			
			\midrule[0.8pt]
			step function &  $\bmPsi(\bmx)$ & $\calO(N)$ & $\calO(L)$ & no error  outside $\Omega([0,1]^d,K,\delta)$\\
			
			\midrule
			$x_1x_2$ & $\varphi(x_1,x_2)$& $\calO(N)$ & $\calO(L)$  & $\scrE_1=216(N+1)^{-2s(L+1)}$\\

			\midrule
			$\bmh^\bmalpha$ & $P_\bmalpha(\bmh)$& $\calO(N)$ & $\calO(L)$  & $\scrE_2=9s(N+1)^{-7sL}$\\
			
			\midrule
			${\partial^\bmalpha f(\bmPsi(\bmx))}$ & $\phi_{\bmalpha}(\bmPsi(\bmx))$& $\calO(N\ln N)$ & $\calO(L\ln L)$  & $\scrE_3=2N^{-2s}L^{-2s}$\\

			\midrule
			$\sum\limits_{\|\bmalpha\|\le s-1}\tfrac{\partial^\bmalpha f(\bmPsi(\bmx))}{\bmalpha!}\bmh^\bmalpha$ & $\sum\limits_{\|\bmalpha\|\le s-1}\varphi\Big(\tfrac{\phi_{\bmalpha}(\bmPsi(\bmx))}{\bmalpha!},P_\bmalpha(\bmh)\Big)$& $\calO(N\ln N)$ & $\calO(L\ln L)$  & $\calO(\scrE_1+\scrE_2+\scrE_3)$\\
			
			\midrule
			
			$f(\bmx)$ & $\phi(\bmx)\coloneqq\sum\limits_{\|\bmalpha\|\le s-1}\varphi\Big(\tfrac{\phi_{\bmalpha}(\bmPsi(\bmx))}{\bmalpha!},P_\bmalpha(\bmx-\bmPsi(\bmx))\Big)$& $\calO(N\ln N)$ & $\calO(L\ln L)$  & $\begin{array}{c}
				\calO(\|\bmh\|_2^{-s}+\scrE_1+\scrE_2+\scrE_3)\\ \le \calO(K^{-s}) =\calO(N^{-2s/d}L^{-2s/d})
			\end{array}$\\
			
			\bottomrule
		\end{tabular} 
	}
\end{table}

Finally, we would like to compare our analysis with that in \cite{yarotsky2019}. Both \cite{yarotsky2019} and our analysis rely on local Taylor expansions as in Equation~\eqref{eqn:tle} to approximate the target function $f$. Both analysis methods construct ReLU FNNs to approximate polynomials and encode the Taylor expansion coefficients into ReLU FNNs. However, the way to localize the Taylor expansion (i.e., defining the local neighborhood such that the expansion is valid) and the approach to constructing ReLU FNNs are different. We will discuss the details as follows.

{\bf Localization.} In \cite{yarotsky2019},  a ``two-scale" partition procedure and a standard triangulation divide $[0,1]^d$ into simplexes and a partition of unity is constructed using compactly supported functions that are linear on each simplex, which implies that these functions in the partition of unity can be represented by ReLU FNNs. Taylor expansions of $f$ are constructed within each support of the functions in the partition of unity. In this paper, we simply divide the domain into small hypercubes of uniform size as visualized in Figure~\ref{fig:QthetaExample}. Taylor expansions of $f$ are constructed within each hypercube. The reader can understand our approach as a simple way to construct a partition of unity using piecewise constant functions with binary values. The introduction of the trifling region allows us to simply construct ReLU FNNs to approximate these piecewise constant functions without caring about the approximation error  within the trifling region. Hence, our construction can be much simplified and makes it easy to estimate all constant prefactors in our error  estimates, which is challenging in \cite{yarotsky2019}.  

{\bf ReLU FNNs for Taylor expansions.} In \cite{yarotsky2019}, very deep ReLU FNNs with  width $\calO(1)$ are constructed to approximate polynomials in local Taylor expansions, and hence, the optimal approximation error  in width was not explored in \cite{yarotsky2019}. In this paper, we construct ReLU FNNs with arbitrary width and depth to approximate polynomials in local Taylor expansions using Proposition~\ref{prop:approxPolynomial}, which allows us to explore the optimal approximation error  in width and is more challenging. In \cite{yarotsky2019}, the coefficients of adjacent local Taylor expansions, i.e., $\partial ^\bmalpha f$ in Equation~\eqref{eqn:tle}, are encoded into ReLU FNNs via bit extraction, which is the key to achieving a better approximation error  of ReLU FNNs to approximate $f$ than the original local Taylor expansions, since the number of coefficients can be significantly reduced via encoding. Actually, the error in depth by bit extraction is nearly optimal. In this paper, the approximation to $\partial ^\bmalpha f$ is reduced to a point fitting problem that can be solved by constructing ReLU FNNs using bit extraction as sketched out in the previous paragraphs. Hence, we can also achieve the optimal approximation error  in depth. The key to achieving the optimal approximation error  in width in the above approximation is the application of Lemma~\ref{lem:squarePointsLemma} that essentially fits $\calO(N^2)$ samples with ReLU FNNs of width $\calO(N)$ and depth $2$. Due to the simplicity of our analysis, we can construct ReLU FNNs with arbitrary width and depth to approximate $f$ and specify all constant prefactors in our approximation error.

\subsection{Constructive proof}
\label{sec:proofOfThmOld} 

According to the key ideas of proving Theorem~\ref{thm:MainGap} summarized in Section~\ref{sec:sketchProofOfThmOld}, let us present the detailed proof.
\begin{proof}[Proof of Theorem~\ref{thm:MainGap}]
	The detailed proof can be divided  into four steps as follows.
\mystep{1}{Set up.}
Set $K=\lfloor N^{1/d}\rfloor^2\lfloor L^{2/d}\rfloor$ and let $\Omega([0,1]^d,K,\delta)$ partition $[0,1]^d$ into $K^d$ cubes $Q_\bmbeta$ for $\bmbeta\in \{0,1,\cdots,K-1\}^d$. In particular,
for each $\bmbeta=[\beta_1,\beta_2,\cdots,\beta_d]^T\in \{0,1,\cdots,K-1\}^d$, we define  $\bmx_\bmbeta\coloneqq \bmbeta/K$ and 
\begin{equation*}
	Q_\bmbeta\coloneqq \big\{\bmx=[x_1,x_2,\cdots,x_d]^T: x_i\in [\tfrac{\beta_i}{K},\tfrac{\beta_i+1}{K}-\delta\cdot \one_{ \{\beta_i\le K-2\}}] \tn{ for } i=1,2,\cdots,d  \big\}.
\end{equation*}
Clearly,  $[0,1]^d=\Omega([0,1]^d,K,\delta)\bigcup \big(\cup_{\bmbeta\in\{0,1,\cdots,K-1\}^d} Q_\bmbeta\big)$ and $\bmx_\bmbeta$ is the vertex of $Q_\bmbeta$ with minimum $\|\cdot\|_1$ norm. See Figure~\ref{fig:QthetaExample} for the illustrations of $Q_\bmbeta$ and $\bmx_\bmbeta$.

By Proposition~\ref{prop:approxStepFun}, there exists  $\psi\in \NNF(\NNwidth\le 4N+3\NNspace\NNdepth \le 4N+5)$  such that
\begin{equation*}
	\psi(x)=k\quad \tn{if $x\in [\tfrac{k}{K},\tfrac{k+1}{K}-\delta\cdot \one_{\{k\le K-2\}}]$\quad for $k=0,1,\cdots,K-1$.}
\end{equation*}    
Then for each $\bmbeta\in \{0,1,\cdots,K-1\}^d$,  $\psi(x_i)={\beta_i}$ for all $\bmx\in Q_\bmbeta$ for $i=1,2,\cdots,d$.

Define 
\begin{equation*}
	\bmPsi(\bmx)\coloneqq \big[{\psi(x_1),\psi(x_2),\cdots,\psi(x_d)}\big]^T/K\quad \tn{for any $\bmx\in [0,1]^d$,}
\end{equation*}
then
\begin{equation*}
	\bmPsi(\bmx)=\bmbeta/K=\bmx_\bmbeta\quad \tn{if $\bmx\in Q_\bmbeta$\quad for $\bmbeta\in \{0,1,\cdots,K-1\}^d$.}
\end{equation*}

For any $\bmx\in Q_\bmbeta$ and $\bmbeta\in\{0,1,\cdots,K-1\}^d$, by the Taylor expansion, there exists $\xi_\bmx \in (0,1)$ such that
\begin{equation*}
	f(\bmx)=\sum_{\|\bmalpha\|_1\le s-1} \tfrac{\partial^\bmalpha f(\bmPsi(\bmx))}{\bmalpha !}\bmh^\bmalpha+ \sum_{\|\bmalpha\|_1= s} \tfrac{\partial^\bmalpha f(\bmPsi(\bmx)+\xi_\bmx \bmh)}{\bmalpha !}\bmh^\bmalpha,\quad \tn{where $\bmh=\bmx-\bmPsi(\bmx)$.} 
\end{equation*}

\mystep{2}{Construct the desired function $\phi$.}
By Lemma~\ref{lem:xyApproxAB}, there exists 
\[
\varphi\in \NNF\big(\NNwidth\le 9(N+1)+1\NNspace \NNdepth \le 2s(L+1)\big)\] 
such that
\begin{equation}
	\label{eq:error1}
	|\varphi(x_1,x_2)- x_1x_2|\le 216(N+1)^{-2s(L+1)}\eqqcolon \scrE_1\quad \tn{for any $x_1,x_2\in [-3,3]$.}
\end{equation}

For each $\bmalpha\in\N^d$ with  $\|\bmalpha\|_1\le s$, by Proposition~\ref{prop:approxPolynomial},  there exists  
\begin{equation*}
	\begin{split}
		P_\bmalpha\in \NNF\big(\NNwidth\le9(N+1)+s-1\NNspace \NNdepth\le 7s^2L\big)
	\end{split}
\end{equation*}
such that
\begin{equation}\label{eq:error2}
	|P_{\alpha}(\bmx)- \bmx^\bmalpha|\le 9s(N+1)^{-7sL}\eqqcolon \scrE_2\quad  \tn{for any $\bmx\in [0,1]^d $.}
\end{equation}

For each $i\in\{0,1,\cdots,K^d-1\}$, define
\[
\bmeta(i)=[\eta_1,\eta_2,\cdots,\eta_d]^T\in \{0,1,\cdots,K-1\}^d
\]
such that $\sum_{j=1}^d \eta_j K^{j-1}=i$. Such a map $\bmeta$ is a bijection from $\{0,1,\cdots,K^d-1\}$ to $\{0,1,\cdots,K-1\}^d$.
For each $\bmalpha\in \N^d$ with $ \|\bmalpha\|_1\le s-1$, define
\begin{equation*}
	\xi_{\bmalpha,i}=\big({\partial ^\bmalpha f(\tfrac{\bmeta(i)}{K})+1}\big)/{2}\quad \tn{for $i\in\{0,1,\cdots,K^d-1\}$.}
\end{equation*}
Then $\|\partial^\bmalpha f\|_{L^\infty([0,1]^d)}\le 1$ implies $\xi_{\bmalpha,i}\in [0,1]$ for $i=0,1,\cdots,K^d-1$ and each $\bmalpha$. 
 Note that $K^d=\big(\lfloor N^{1/d}\rfloor^2\lfloor L^{2/d}\rfloor\big)^d\le N^2L^2$. By Proposition~\ref{prop:pointsMatching}, there exists 
\begin{equation*}
	\tildephi_\bmalpha\in\NNF\big(\NNwidth \le 16s(N+1)\log_2(8N)\NNspace \NNdepth\le 5(L+2)\log_2(4L)\big)
\end{equation*}
such that,  for each $\bmalpha\in\N^d$ with $\|\bmalpha\|_1\le s-1$, we have
\begin{equation*}
	|\tildephi_\bmalpha(i)-\xi_{\bmalpha,i}|\le N^{-2s}L^{-2s}\quad \tn{for $i=0,1,\cdots,K^d-1$.}
\end{equation*}

For each $\bmalpha\in \N^d$ with $\|\bmalpha\|_1\le s-1$, define \begin{equation*}
	\phi_\bmalpha(\bmx)\coloneqq 2\tildephi_\bmalpha\big(\sum_{j=1}^d x_jK^{j-1}\big)-1\quad \tn{for any $\bmx=[x_1,x_2,\cdots,x_d]^T\in \R^d$.}
\end{equation*} 
It is easy to verify that
\begin{equation*}
	\phi_\bmalpha\in\NN\big(\NNwidth \le 16s(N+1)\log_2(8N)\NNspace \NNdepth\le 5(L+2)\log_2(4L)\big).
\end{equation*}
Then, for  each $\bmalpha\in \N^d$ with $\|\bmalpha\|_1\le s-1$ and each $\bmeta=\bmeta(i)=[\eta_1,\eta_2,\cdots,\eta_d]^T\in \{0,1,\cdots,K-1\}^d$ corresponding to $i=\sum_{j=1}^d \eta_j K^{j-1}\in \{0,1,\cdots,K^d-1\}$, we have
\begin{equation*}
	\begin{split}
		\big|\phi_\bmalpha(\tfrac{\bmeta}{K})-\partial^\bmalpha f(\tfrac{\bmeta}{K})\big|
		&=\Big|2\tildephi_\bmalpha\big(\sum_{j=1}^d \eta_jK^{j-1}\big)-1-(2\xi_{\bmalpha,i}-1)\Big|\\
		&=2|\tildephi_\bmalpha(i)-\xi_{\bmalpha,i}|\le 2N^{-2s}L^{-2s}.
	\end{split}
\end{equation*}
Therefore, for each $\bmbeta\in \{0,1,\cdots,K-1\}^d$ and each $\bmalpha\in \N^d$ with $\|\bmalpha\|_1\le s-1$, we have
\begin{equation}
	\label{eq:error3}
	\big|\phi_\bmalpha(\bmx_\bmbeta)-\partial^\bmalpha f(\bmx_\bmbeta)\big|=\big|\phi_\bmalpha(\tfrac{\bmbeta}{K})-\partial^\bmalpha f(\tfrac{\bmbeta}{K})\big|\le  2N^{-2s}L^{-2s}\eqqcolon \scrE_3.
\end{equation}

Now we can construct the desired function $\phi$ as
\begin{equation}
	\label{eq:phiDefSmoothFun}
	\phi(\bmx)\coloneqq \sum_{\|\bmalpha\|_1\le s-1} \varphi\Big(\tfrac{\phi_\bmalpha(\bmPsi(\bmx))}{\bmalpha !},P_\bmalpha\big(\bmx-\bmPsi(\bmx)\big)\Big)\quad \tn{for any $\bmx\in \R^d$.}
\end{equation}
It remains to estimate the approximation error  and determine the size of the network implementing $\phi$.

\mystep{3}{Estimate approximation error.}

Fix $\bmbeta\in \{0,1,\cdots,K-1\}^d$,
let us estimate the approximation error  for a fixed $\bmx\in Q_\bmbeta$. See Table \ref{tab:constructiveProofShort} for a summary of the approximation errors. 
Recall that $\bmPsi(\bmx)=\bmx_\bmbeta$ and $\bmh=\bmx-\bmPsi(\bmx)=\bmx-\bmx_\bmbeta$.
It is easy to check that $ |f(\bmx)-\phi(\bmx)|$ is bounded by
\begin{equation*}
	\begin{split}
		& \quad  \left|
		\sum_{\|\bmalpha\|_1\le s-1} \tfrac{\partial^\bmalpha f(\bmPsi(\bmx))}{\bmalpha !}\bmh^\bmalpha
		+ \sum_{\|\bmalpha\|_1= s} \tfrac{\partial^\bmalpha f(\bmPsi(\bmx)+\xi_\bmx \bmh)}{\bmalpha !}\bmh^\bmalpha
		-\sum_{\|\bmalpha\|_1\le s-1} \varphi\Big(\tfrac{\phi_\bmalpha(\bmPsi(\bmx))}{\bmalpha !},P_\bmalpha\big(\bmx-\bmPsi(\bmx)\big)\Big) 
		\right|\\
		&\le \underbrace{\sum_{\|\bmalpha\|_1=s} \Big|\tfrac{\partial^\bmalpha f(\bmx_\bmbeta+\xi_\bmx \bmh)}{\bmalpha !}\bmh^\bmalpha\Big|}_{\scrI_1} 
			\quad +  \quad
			\underbrace{\sum_{\|\bmalpha\|_1\le s-1} \Big|\tfrac{\partial^\bmalpha f(\bmx_\bmbeta)}{\bmalpha !}\bmh^\bmalpha- \varphi\big(\tfrac{\phi_\bmalpha(\bmx_\bmbeta)}{\bmalpha !},P_\bmalpha(\bmh)\big)\Big|}_{\scrI_2} 
			\eqqcolon \scrI_1 +\scrI_2.
	\end{split}
\end{equation*}
Recall the fact that
\begin{equation*}
    \sum_{\|\bmalpha\|_1= s}1
    =\big|\big\{\bmalpha\in\N^d:\|\bmalpha\|_1
    = s\big\}\big|\le (s+1)^{d-1}\ \footnote{In fact, we have $\big|\big\{\bmalpha\in\N^d:\|\bmalpha\|_1= s\big\}\big|=\binom{s+d-1}{d-1}$, implying $(s/d+1)^{d-1}\le \sum_{\|\bmalpha\|_1= s}1\le (s+1)^{d-1}$. Thus, the lower bound of the estimate is still exponentially large in $d$. To the best of our knowledge, we cannot avoid a constant prefactor that is exponentially large in $d$ when Taylor expansion is used in the analysis.}
\end{equation*}
and 
\[\sum_{\|\bmalpha\|_1\le s-1}1=\sum_{i=0}^{s-1}\bigg(\sum_{\|\bmalpha\|_1=i}1\bigg)\le\sum_{i=0}^{s-1}(i+1)^{d-1}\le s\cdot(s-1+1)^{d-1}= s^d.\]

For the first part $\scrI_1$, we have
\begin{equation*}
	\scrI_1=\sum_{\|\bmalpha\|_1=s} \Big|\tfrac{\partial^\bmalpha f(\bmx_\bmbeta+\xi_\bmx \bmh)}{\bmalpha !}\bmh^\bmalpha\Big| \le \sum_{\|\bmalpha\|_1=s} \Big|\tfrac{1}{\bmalpha !}\bmh^\bmalpha\Big|\le (s+1)^{d-1}K^{-s}.
\end{equation*}

For the second part $\scrI_2$, we have
\begin{equation*}
	\begin{split}
		\scrI_2 
		=\sum_{\|\bmalpha\|_1\le s-1} 
		\underbrace{ \Big|\tfrac{\partial^\bmalpha f(\bmx_\bmbeta)}{\bmalpha !}\bmh^\bmalpha
		- \varphi\big(\tfrac{\phi_\bmalpha(\bmx_\bmbeta)}{\bmalpha !},P_\bmalpha(\bmh)\big)\Big|
		}_{\scrI_{2}(\bmalpha)}
		\eqqcolon \sum_{\|\bmalpha\|_1\le s-1} \scrI_2(\bmalpha).
	\end{split}
\end{equation*}
Fix $\bmalpha\in\N^d$ with $\|\bmalpha\|_1\le s-1$, we have
\begin{equation*}
	\begin{split}
		\scrI_2(\bmalpha)
		&= \Big|\tfrac{\partial^\bmalpha f(\bmx_\bmbeta)}{\bmalpha !}\bmh^\bmalpha
		- \varphi\big(\tfrac{\phi_\bmalpha(\bmx_\bmbeta)}{\bmalpha !},P_\bmalpha(\bmh)\big)\Big| \\
		& \le \underbrace{ \Big|\tfrac{\partial^\bmalpha f(\bmx_\bmbeta)}{\bmalpha !}\bmh^\bmalpha
		- \varphi\big(\tfrac{\partial^\bmalpha f(\bmx_\bmbeta)}{\bmalpha !},P_\bmalpha(\bmh)\big)\Big|
		}_{\scrI_{2,1}(\bmalpha)}
			+
			\underbrace{\Big| \varphi\big(\tfrac{\partial^\bmalpha f(\bmx_\bmbeta)}{\bmalpha !},P_\bmalpha(\bmh)\big)
			-\varphi\big(\tfrac{\phi_\bmalpha(\bmx_\bmbeta)}{\bmalpha !},P_\bmalpha(\bmh)\big)\Big|
			}_{\scrI_{2,2}(\bmalpha)} 
			\\
		&\eqqcolon  \scrI_{2,1}(\bmalpha)+\scrI_{2,2}(\bmalpha).
	\end{split}
\end{equation*}

Note that $\scrE_2=9s(N+1)^{-7sL}\le 9s(2)^{-7s}\le 2$.
By  $\bmh^\bmalpha\in[0,1]$ and Equation~\eqref{eq:error2}, we have $P_\bmalpha(\bmh)\in [-2,3]\subseteq [-3,3]$. Then by  $\partial^\bmalpha f(\bmx_\bmbeta)\in [-1,1]$ and Equations~\eqref{eq:error1} and \eqref{eq:error2}, we have
\begin{equation*}
	\begin{split}
		\scrI_{2,1}(\bmalpha)
		&
		= \Big|\tfrac{\partial^\bmalpha f(\bmx_\bmbeta)}{\bmalpha !}\bmh^\bmalpha- \varphi\big(\tfrac{\partial^\bmalpha f(\bmx_\bmbeta)}{\bmalpha !},P_\bmalpha(\bmh)\big)\Big|
		\\
		&
		\le\Big|\tfrac{\partial^\bmalpha f(\bmx_\bmbeta)}{\bmalpha !}\bmh^\bmalpha- \tfrac{\partial^\bmalpha f(\bmx_\bmbeta)}{\bmalpha !}P_\bmalpha(\bmh)\Big|+
			\underbrace{\Big|\tfrac{\partial^\bmalpha f(\bmx_\bmbeta)}{\bmalpha !}P_\bmalpha(\bmh)-\varphi\big(\tfrac{\partial^\bmalpha f(\bmx_\bmbeta)}{\bmalpha !},P_\bmalpha(\bmh)\big)\Big|}_{\tn{$\le \scrE_1$ by Eq. \eqref{eq:error1}}}
			\\
		&
		\le  \tfrac{1}{\bmalpha !}\underbrace{\big|\bmh^\bmalpha
		- P_\bmalpha(\bmh)\big|}_{\tn{$\le \scrE_2$ by Eq. \eqref{eq:error2}}}+\scrE_1
			\le \tfrac{1}{\bmalpha !}\scrE_2+\scrE_1
			\le \scrE_1+\scrE_2.
			\\
	\end{split}
\end{equation*}

To estimate $\scrI_{2,2}(\bmalpha)$, we need the following fact derived from Equation~\eqref{eq:error1}: 
\begin{equation}\label{eq:error1:new}
	\begin{split}
		|\varphi(x_1,x_2)-\varphi(\tildex_1,x_2)|
		&\le 
		\underbrace{|\varphi(x_1,x_2)-x_1x_2|}_{\tn{$\le \scrE_1$ by Eq. \eqref{eq:error1}}}
		+\underbrace{|\varphi(\tildex_1,x_2)-\tildex_1x_2|}_{\tn{$\le \scrE_1$ by Eq. \eqref{eq:error1}}}+|x_1x_2-\tildex_1x_2|\\
		&\le 2\scrE_1+3|x_1-\tildex_1|,
	\end{split}
\end{equation}
for any $x_1,\tildex_1,x_2\in [-3,3]$.

Since $\scrE_3=2N^{-2s}L^{-2s}\le 2$ and $\partial^\bmalpha f(\bmx_\bmbeta)\in [-1,1]$, we have $\phi_\bmalpha(\bmx_\bmbeta)\in[-3,3]$ by Equation~\eqref{eq:error3}. Then by $P_\bmalpha(\bmh)\in[-3,3]$ and Equations~\eqref{eq:error1:new} and \eqref{eq:error3}, we have
\begin{equation*}
	\begin{split}
		\scrI_{2,2}(\bmalpha)
		&= 
		\Big| 
		\varphi\big(\tfrac{\partial^\bmalpha f(\bmx_\bmbeta)}{\bmalpha !},P_\bmalpha(\bmh)\big)
		-\varphi\big(\tfrac{\phi_\bmalpha(\bmx_\bmbeta)}{\bmalpha !},P_\bmalpha(\bmh)\big)
		\Big|\\
		& \le 2\scrE_1+3\underbrace{\Big|\tfrac{\partial^\bmalpha f(\bmx_\bmbeta)}{\bmalpha !}
		-\tfrac{\phi_\bmalpha(\bmx_\bmbeta)}{\bmalpha !} \Big|}_{\tn{$\le \scrE_3$ by Eq. \eqref{eq:error3}}}
		\le 2\scrE_1+3\scrE_3.
	\end{split}
\end{equation*}

Therefore, we get
\begin{equation*}
	\begin{split}
		|f(\bmx)-\phi(\bmx)|\le \scrI_1+\scrI_{2}
		&\le \scrI_1+ 
		\sum_{\|\bmalpha\|_1\le s-1}\scrI_{2}(\bmalpha) 
		\le\scrI_1+\sum_{\|\bmalpha\|_1\le s-1}\Big(\scrI_{2,1}(\bmalpha)+\scrI_{2,2}(\bmalpha)\Big)\\
		&\le (s+1)^{d-1}K^{-s}+s^d\Big((\scrE_1+\scrE_2)+(2\scrE_1+3\scrE_3)\Big)\\
		&\le (s+1)^d(K^{-s}+3\scrE_1+\scrE_2+3\scrE_3).
	\end{split}
\end{equation*}

Since $\bmbeta\in \{0,1,\cdots,K-1\}^d$ and  $\bmx\in Q_\bmbeta$ are arbitrary and
\[[0,1]^d
= \Omega([0,1]^d,K,\delta)\bigcup\Big( \cup_{\bmbeta\in\{0,1,\cdots,K-1\}^d}Q_\bmbeta\Big),\]
we have, \tn{for any $\bmx\in [0,1]^d\backslash \Omega([0,1]^d,K,\delta)$},
\begin{equation*}
	\begin{split}
		|f(\bmx)-\phi(\bmx)|\le (s+1)^d(K^{-s}+3\scrE_1+\scrE_2+3\scrE_3).
	\end{split}
\end{equation*}
Recall that  $K=\lfloor N^{1/d}\rfloor^2\lfloor L^{2/d}\rfloor\ge \tfrac{N^{2/d}L^{2/d}}{8}$ and
\[(N+1)^{-7sL}\le(N+1)^{-2s(L+1)}\le (N+1)^{-2s}2^{-2sL}\le N^{-2s}L^{-2s}.\]
 Then we have
\begin{equation*}
	\begin{split}
		&\quad (s+1)^d(K^{-s}+3\scrE_1+\scrE_2+3\scrE_3)\\
		&=(s+1)^d\Big(K^{-s}+648(N+1)^{-2s(L+1)}+9s(N+1)^{-7sL}+6N^{-2s}L^{-2s}\Big)\\
		&\le (s+1)^d\Big(8^s N^{-2s/d}L^{-2s/d}+(654+9s)N^{-2s}L^{-2s}\Big)\\
		&\le (s+1)^d(8^s+654+9s)N^{-2s/d}L^{-2s/d}\le 84(s+1)^d8^sN^{-2s/d}L^{-2s/d}.\\
	\end{split}
\end{equation*}

\mystep{4}{Determine the size of the network implementing $\phi$.}

It remains to estimate the width and depth of the network implementing $\phi$. 
Recall that,
for $\bmalpha\in \N^d$ with $\|\bmalpha\|_1\le s-1$, 

\begin{equation*}
	\left\{\begin{array}{l}
		\bmPsi\in \NNF\big(\NNwidth\le d(4N+3)\NNspace \NNdepth\le 4L+5\big),\\
		\phi_\bmalpha\in \NNF\big(\NNwidth\le 16s(N+1)\log_2(8N)\NNspace \NNdepth\le 5(L+2)\log_2(4L)\big),\\
		P_\bmalpha\in \NNF\big(\NNwidth\le 9(N+1)+s-1\NNspace \NNdepth\le 7s^2L\big),\\
		\varphi\in \NNF\big(\NNwidth\le 9(N+1)+1\NNspace \NNdepth\le 2s(L+1)\big).\\		
	\end{array}\right.
\end{equation*}

\begin{figure}[!htp]
	\centering
	\includegraphics[width=0.85\textwidth]{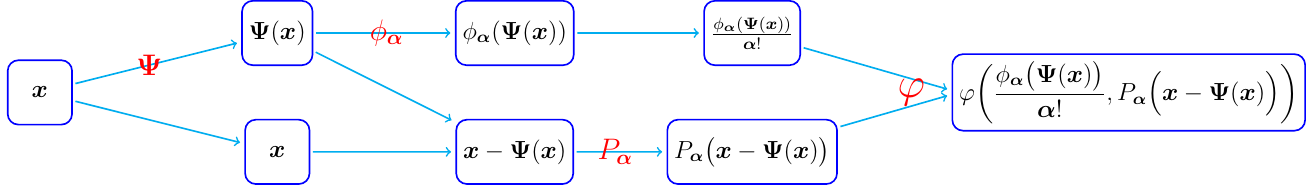}	
	\caption{An illustration of the sub-network architecture  implementing each component of $\phi$, $
		\varphi\Big(\tfrac{\phi_\bmalpha(\bmPsi(\bmx))}{\bmalpha !},P_\bmalpha\big(\bmx-\bmPsi(\bmx)\big)\Big)$ for each $\bmalpha\in\N^d$ with $\|\bmalpha\|\le s-1$. 
	}
	\label{fig:phiForAlpha}
\end{figure}

By Equation~\eqref{eq:phiDefSmoothFun} and Figure~\ref{fig:phiForAlpha}, it easy to verify that $\phi$ can be implemented by a ReLU FNN with width 
\[\begin{split}
	\sum_{\|\bmalpha\|_1\le s-1} 16sd(N+2)\log_2(8N)
	&\le s^d\cdot 16sd(N+2)\log_2(8N)\\
	&= 16s^{d+1}d(N+2)\log_2(8N)
\end{split}\] and depth 
\begin{equation*}
	\begin{split}
		(4L+5)+2s(L+1)+7s^2L+5(L+2)\log_2(4L) +3\le 
		18s^2(L+2)\log_2(4L)
	\end{split} 
\end{equation*} as desired. So we finish the proof.
\end{proof}

\section{Proofs of Propositions in Section~\ref{sec:sketchProofOfThmOld}}
\label{sec:proofOfProposition}
In this section, we will prove all propositions  in Section~\ref{sec:sketchProofOfThmOld}.

\subsection{Proof of Proposition~\ref{prop:approxPolynomial} for polynomial approximation}
\label{sec:approxPolynomial}

To prove Proposition~\ref{prop:approxPolynomial}, we will construct ReLU FNNs to approximate  multivariate polynomials following the  four steps below. 
\begin{itemize}
	\item $f(x)=x^2$. We approximate $f(x)=x^2$ by the combinations and compositions of ``sawtooth'' functions as shown in Figures~\ref{fig:toothFunctions} and \ref{fig:fs}.
	
	\item $f(x,y)=xy$. To approximate $f(x,y)=xy$, we use the result of the previous step and the fact that $xy=2\big((\tfrac{x+y}{2})^2-(\tfrac{x}{2})^2-(\tfrac{y}{2})^2\big)$.
	
	\item $f(x_1,x_2,\cdots,x_k)=x_1x_2\cdots x_k$. 
	We approximate $f(x_1,x_2,\cdots,x_k)=x_1x_2\cdots x_k$ for any $k\ge 2$
	via mathematical induction based on the result of the previous step.
	
	\item A general polynomial $P(\bmx)=\bmx^\bmalpha=x_1^{\alpha_1}x_2^{\alpha_2}\cdots x_d^{\alpha_d}$ with $\|\bmalpha\|_1\le k$. Any one-term polynomial  of degree $\le k$ can be written as $C z_1z_2\cdots z_k$ with some entries equaling $1$, where $C$ is a constant and $\bmz=[z_1,z_2,\cdots,z_k]^T$ can be attained via an affine linear map with $\bmx$ as the input. Then use the result of the previous step.    
\end{itemize}

The idea of using ``sawtooth'' functions (see Figure~\ref{fig:toothFunctions}) was first raised in \cite{yarotsky2017} for approximating $x^2$ using FNNs with width $6$ and depth $\calO(L)$ and achieving an error  $\calO(2^{-L})$; our construction is different from and more general than that in \cite{yarotsky2017}, working for ReLU FNNs of width $\calO(N)$ and depth $\calO(L)$ for any $N$ and $L$, and achieving an error  $\calO(N^{-L})$. As discussed below Proposition~\ref{prop:approxPolynomial}, this $\calO(N^{-L})$ approximation error  of polynomial functions shows the power of depth in ReLU FNNs via function composition.

First,  let us show how to construct  ReLU FNNs to approximate $f(x)=x^2$.
\begin{lemma}
	\label{lem:approxSquare}
	For any $N,L\in \N^+$, there exists a function  $\phi$ implemented by a ReLU FNN with width $3N$ and depth $L$ such that
	\begin{equation*}
		|\phi(x)-x^2|\le N^{-L}\quad \tn{for any $x\in [0,1]$.}
	\end{equation*}
\end{lemma}
\begin{proof}
	Define a set of ``sawtooth'' functions $T_i:[0,1]\to [0,1]$ by induction as follows. Set
	\begin{equation*}
		T_1(x)=\left\{\begin{array}{lc}
			2x, &\tn{if} \  x\in [0,\tfrac12],\\
			2(1-x), &\tn{if} \   x\in(\tfrac12,1],
		\end{array}\right.
	\end{equation*}
	and 
	\begin{equation*}
		T_i=T_{i-1}\circ T_1\quad \tn{for $i=2,3,\cdots$}.
	\end{equation*}
	It is easy to check that $T_i$ has $2^{i-1}$ ``sawteeth'' and
	\begin{equation*}
		T_{m+n}=T_m\circ T_n\quad \tn{for any $m,n\in\N^+$.}
	\end{equation*}
	See Figure~\ref{fig:toothFunctions} for illustrations of $T_i$ for $i=1,2,3,4$.
	\begin{figure}[!htp]
		\centering
		\begin{subfigure}[b]{0.24\textwidth}
			\centering
			\includegraphics[width=0.99\textwidth]{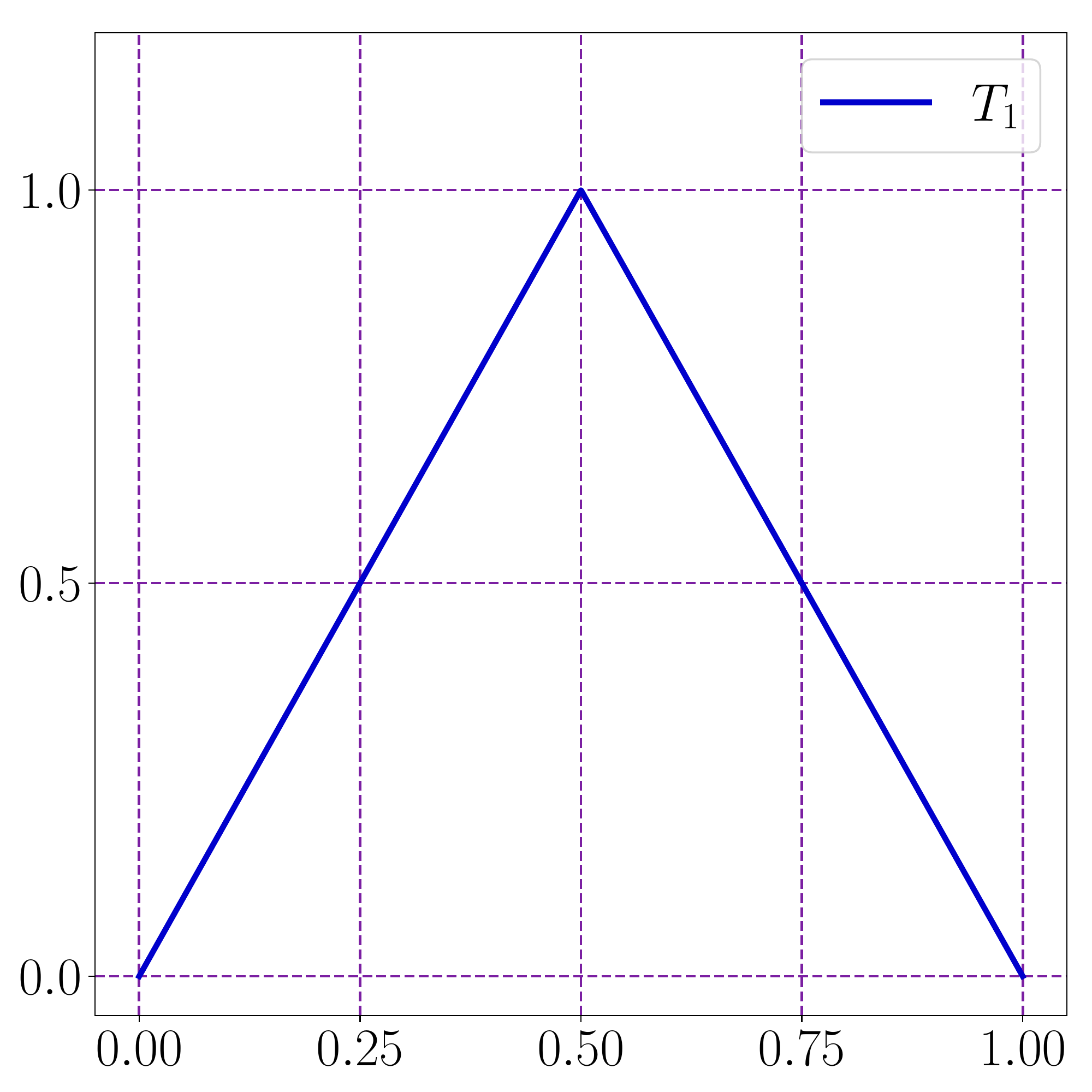}
		\end{subfigure}
		\begin{subfigure}[b]{0.24\textwidth}
			\centering            \includegraphics[width=0.99\textwidth]{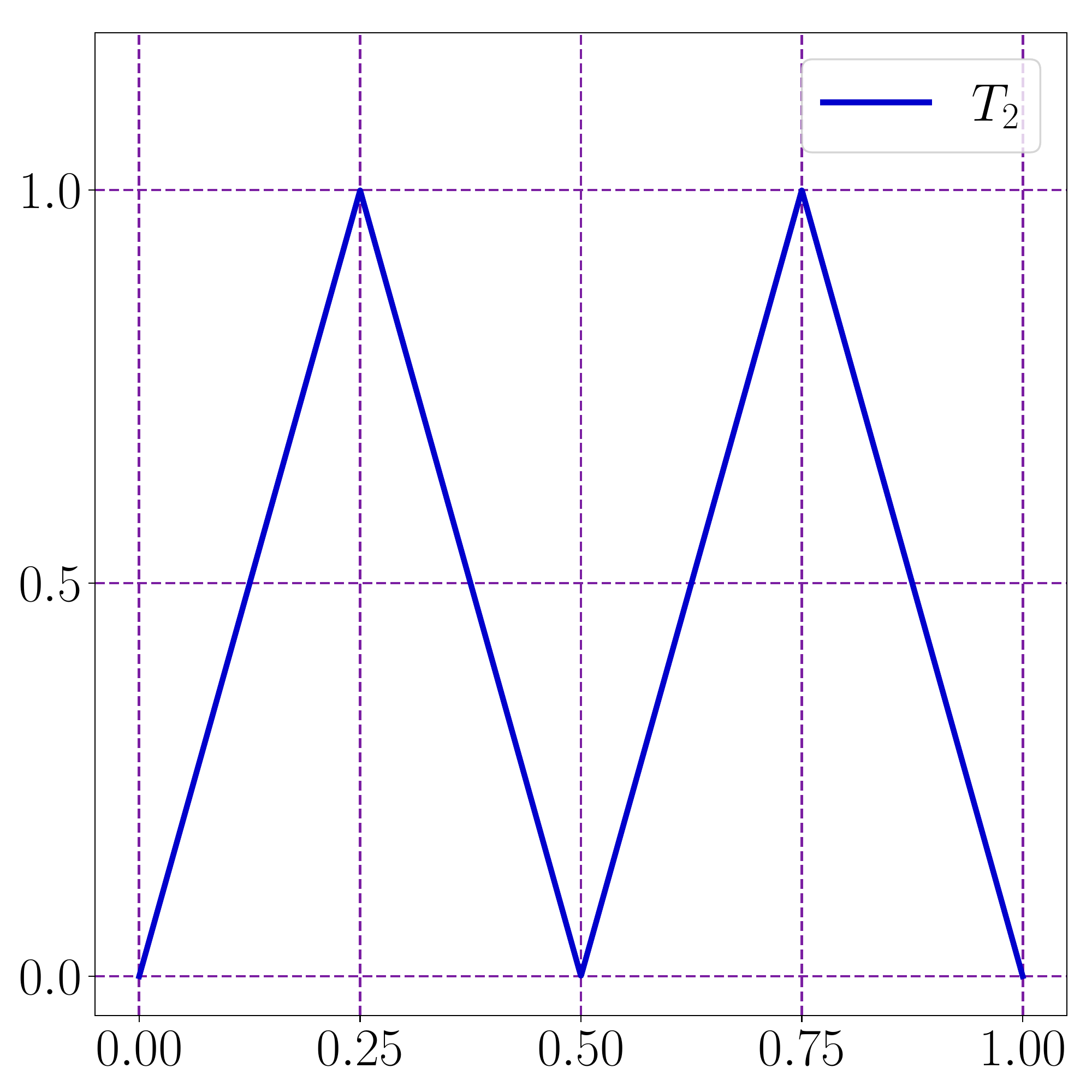}
		\end{subfigure}
		\begin{subfigure}[b]{0.24\textwidth}
			\centering           \includegraphics[width=0.99\textwidth]{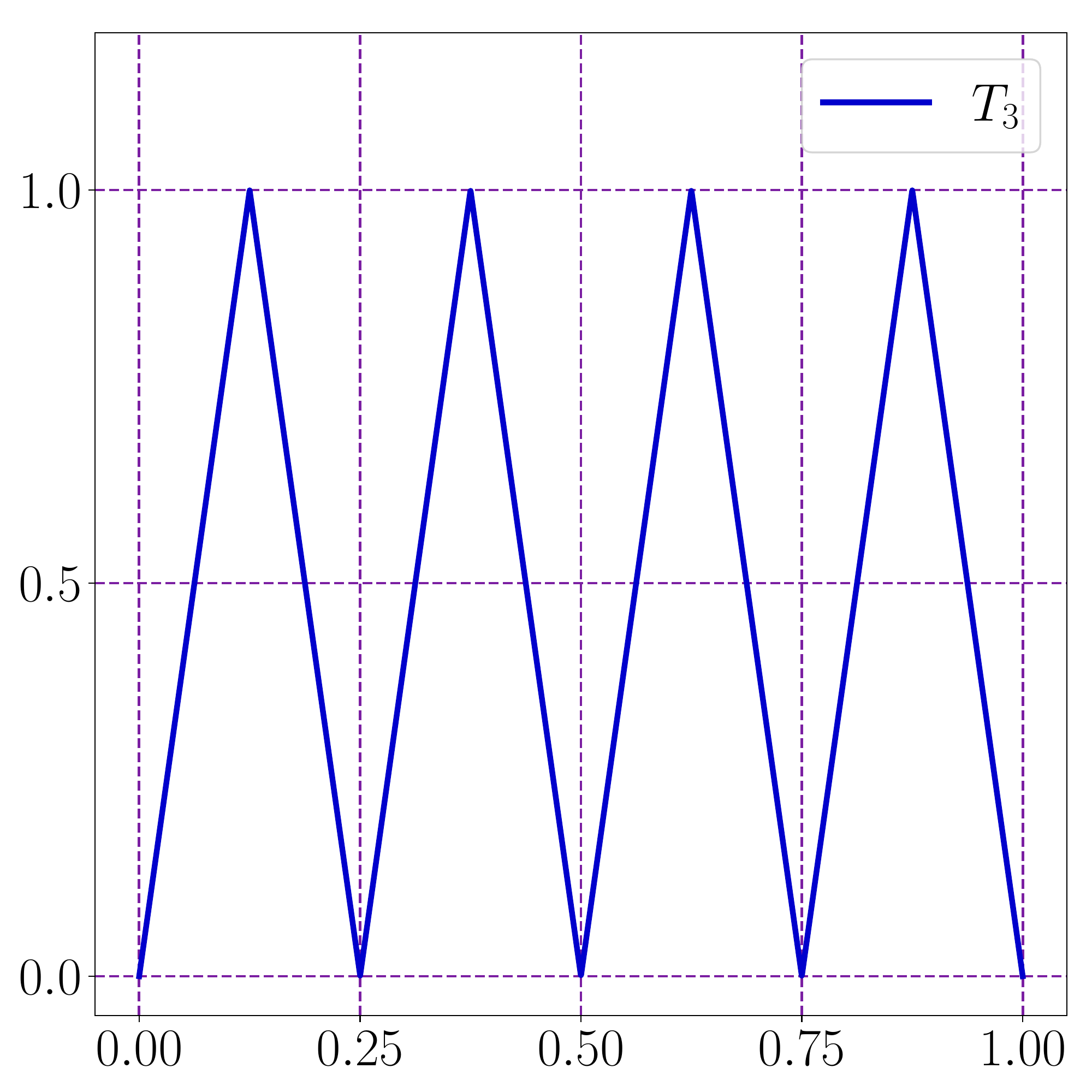}
		\end{subfigure}
		\begin{subfigure}[b]{0.24\textwidth}
			\centering            \includegraphics[width=0.99\textwidth]{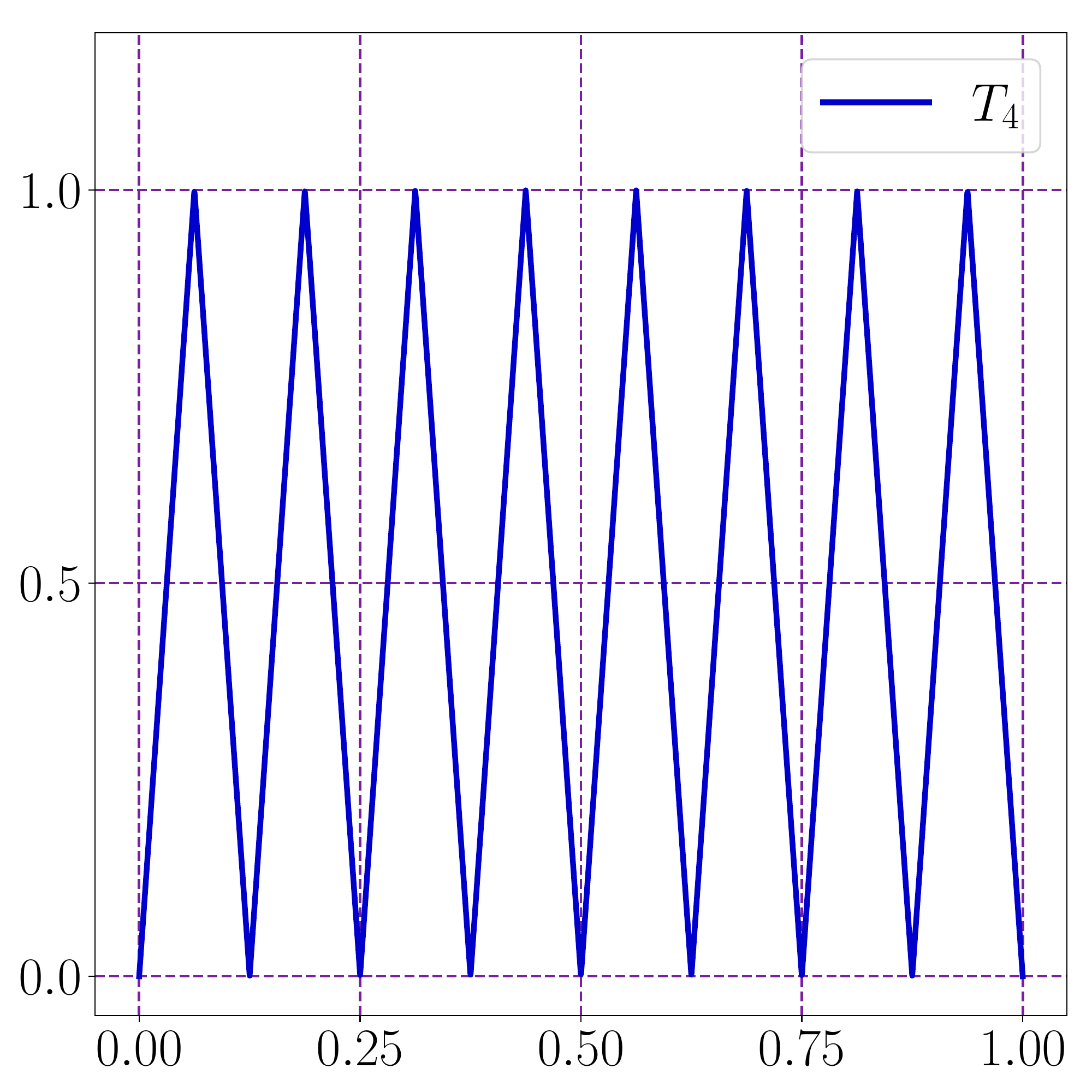}
		\end{subfigure}
		\caption{Examples of ``sawtooth'' functions $T_1,\ T_2,\ T_3$, and $T_4$.}
		\label{fig:toothFunctions}
	\end{figure}
	
	Define piecewise linear functions $f_s:[0,1]\to [0,1]$ for $s\in \N^+$ satisfying the following two requirements (see Figure~\ref{fig:fs} for several examples of $f_s$).
	\begin{itemize}
		\item $f_s(\tfrac{j}{2^s})=\big(\tfrac{j}{2^s}\big)^2$ for $j=0,1,2,\cdots,2^s$.
		\item $f_s(x)$ is linear between any two adjacent points of $\{\tfrac{j}{2^s}:j=0,1,2,\cdots,2^s\}$.
	\end{itemize}
	\begin{figure}[!htp]
		\centering
		\begin{subfigure}[b]{0.24\textwidth}
			\centering
			\includegraphics[width=0.99\textwidth]{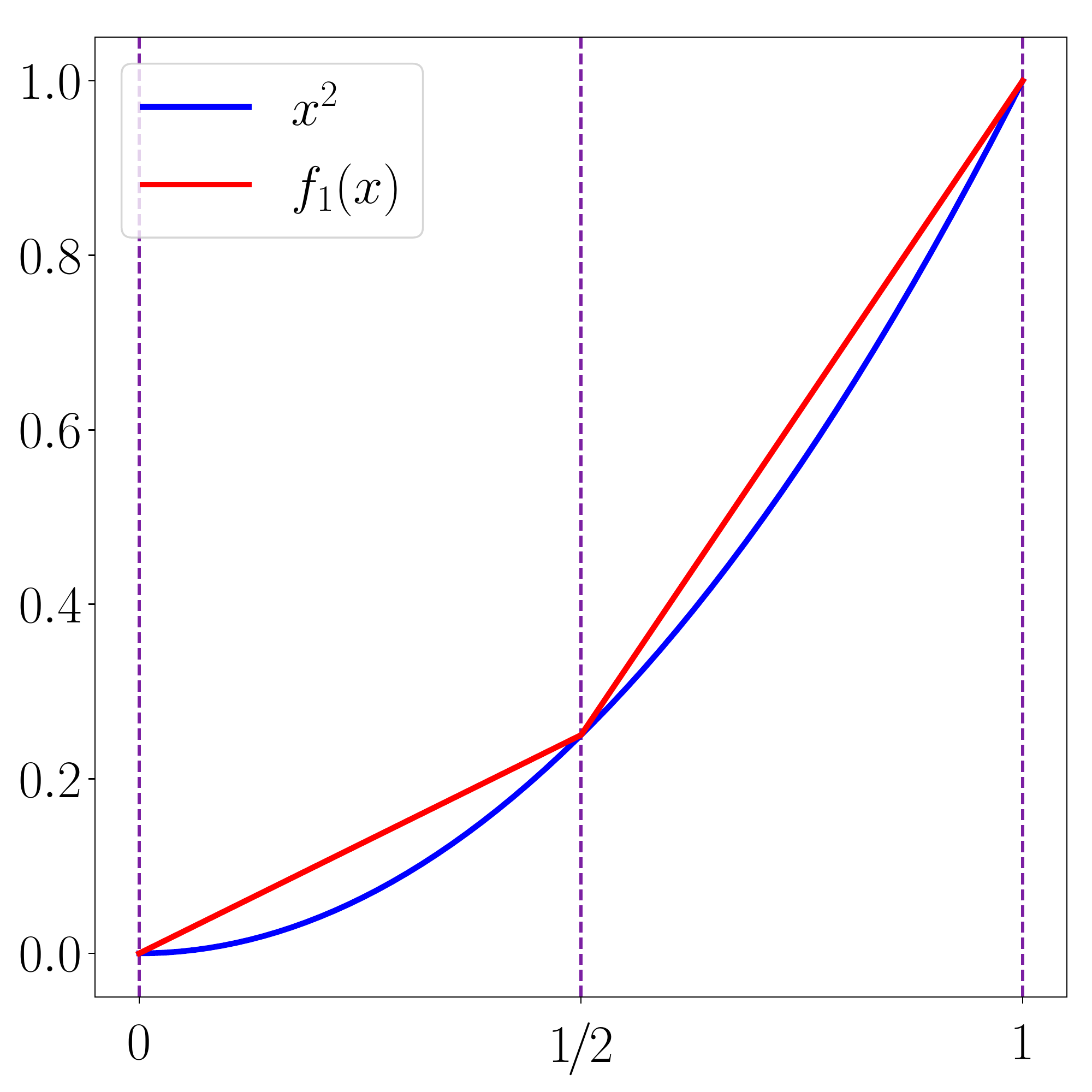}
		\end{subfigure}
		\begin{subfigure}[b]{0.24\textwidth}
			\centering
			\includegraphics[width=0.99\textwidth]{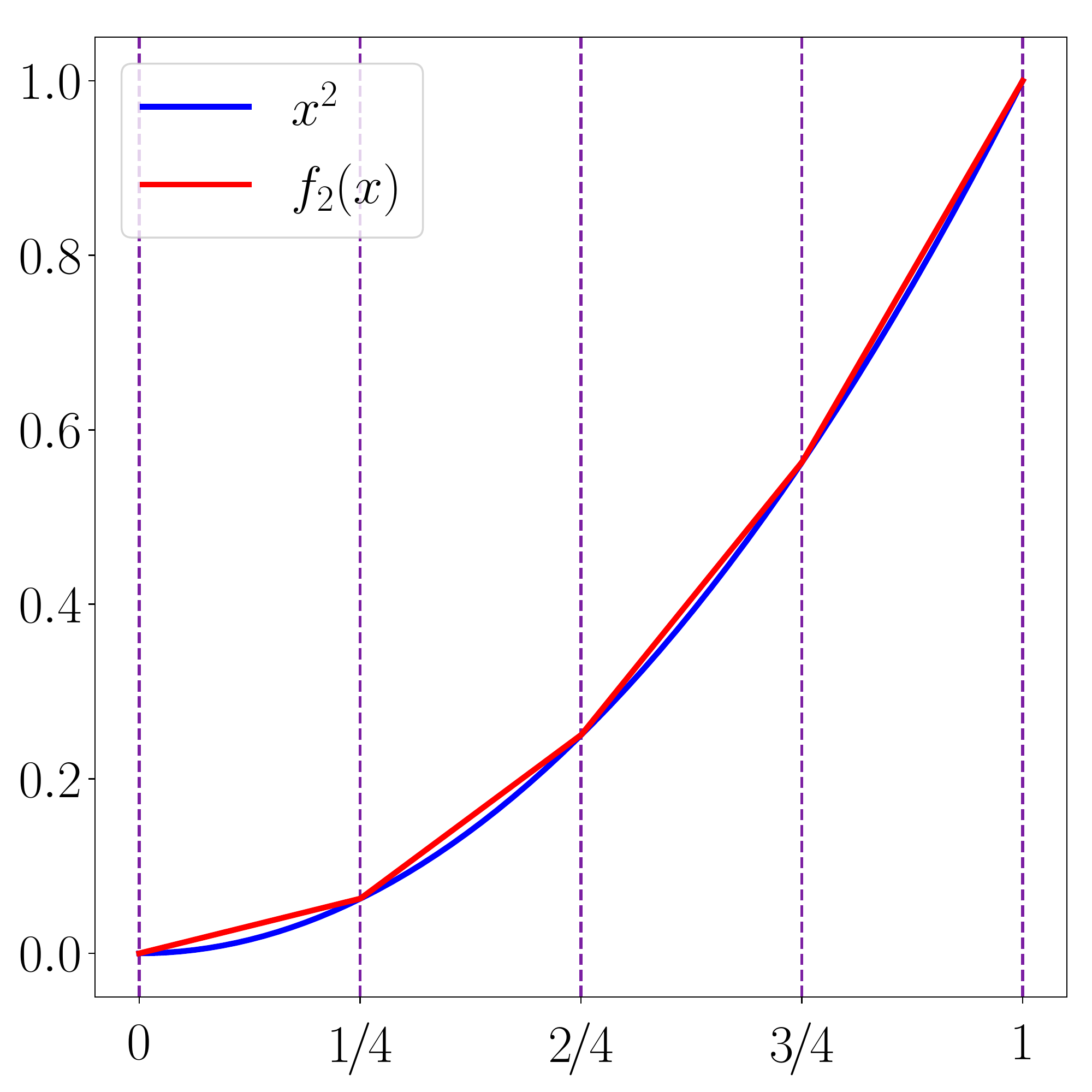}
		\end{subfigure}
		\begin{subfigure}[b]{0.24\textwidth}
			\centering
			\includegraphics[width=0.99\textwidth]{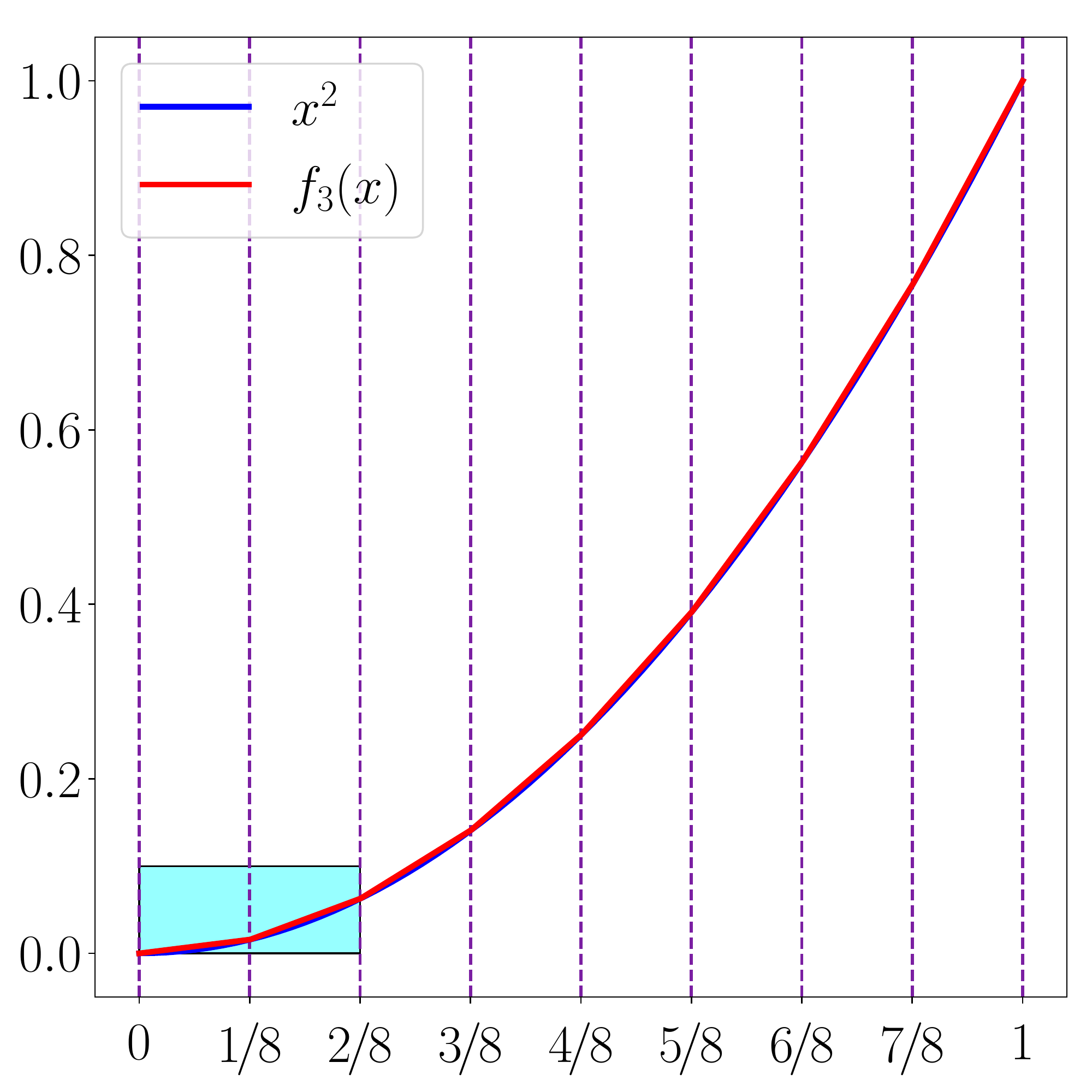}
		\end{subfigure}
		\begin{subfigure}[b]{0.24\textwidth}
			\centering
			\includegraphics[width=0.99\textwidth]{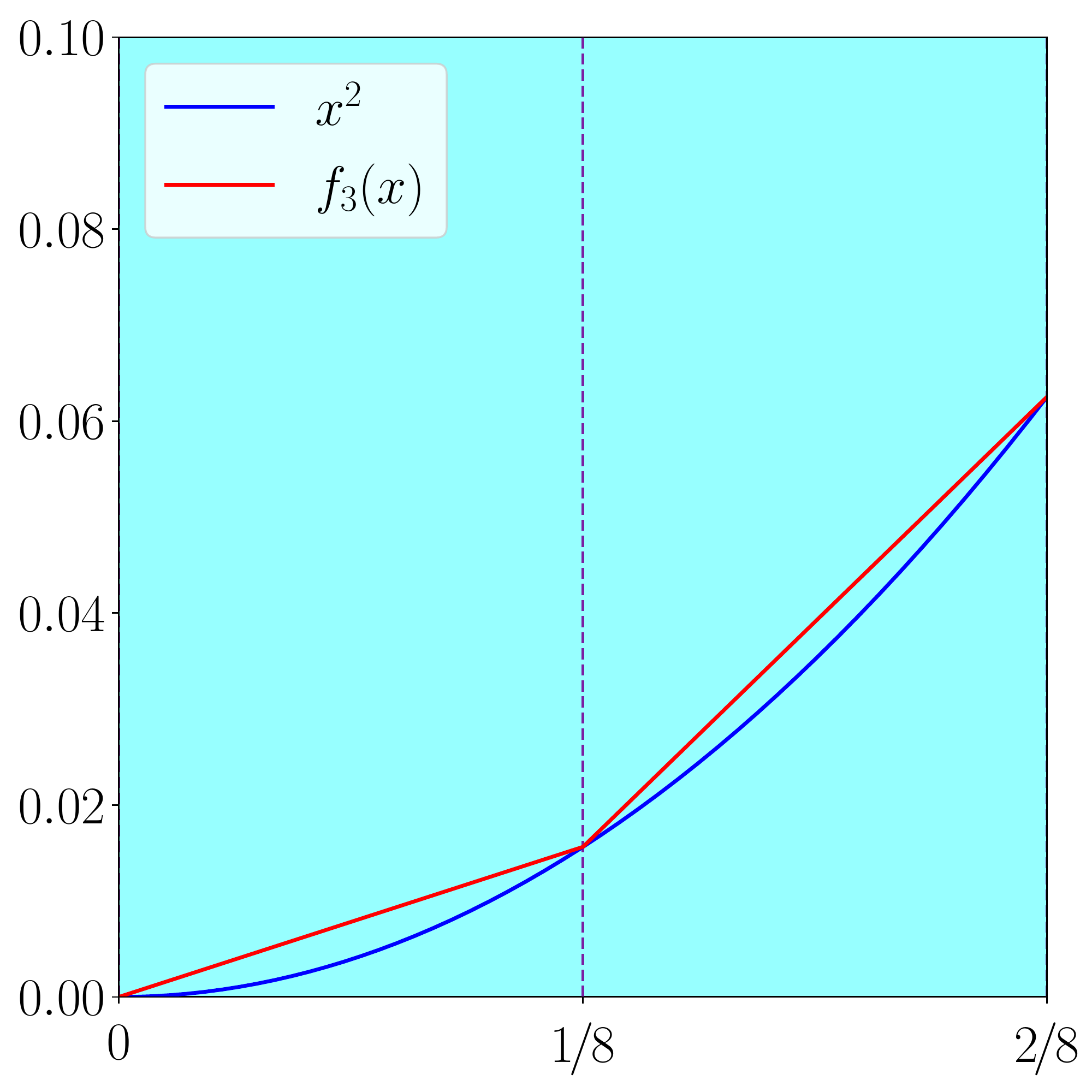}
		\end{subfigure}
		\caption{Illustrations of $f_1,\ f_2$, and $f_3$ for approximating $x^2$. }
		\label{fig:fs}
	\end{figure}
	Recall the fact 
	\begin{equation*}
		0\le tx_1^2+(1-t)x_2^2 - \Big(tx_1+(1-t)x_2\Big)^2\le \frac{(x_2-x_1)^2}{4}\quad \tn{for any $t,x_1,x_2\in [0,1]$.}
	\end{equation*}
		Thus, we have
	\begin{equation}
	0\le f_s(x)-x^2\le \frac{(2^{-s})^2}{4}= 2^{-2(s+1)}\quad \tn{for any $x\in [0,1]$ and $s\in \N^+$.}
	\label{eq:xsquareandfs}
	\end{equation}
	Note that $f_{i-1}(x)=f_i(x)=x^2$ for $x\in \{\tfrac{j}{2^{i-1}}:j=0,1,2,\cdots,2^{i-1}\}$ and the graph of $f_{i-1}-f_i$ is a symmetric ``sawtooth'' between any two adjacent points of $\{\tfrac{j}{2^{i-1}}:j=0,1,2,\cdots,2^{i-1}\}$. 
	It is easy to verify that
	\begin{equation*}
		f_{i-1}(x)-f_i(x)=\tfrac{T_i(x)}{2^{2i}}\quad \tn{for any $x\in [0,1]$ and $i=2,3,\cdots$}.
	\end{equation*}
	Therefore, for any $x\in [0,1]$ and $s\in \N^+$, we have
	\begin{equation*}
		f_s(x)=f_1(x)+\sum_{i=2}^s(f_i-f_{i-1})=x-(x-f_1(x))-\sum_{i=2}^{s}\tfrac{T_i(x)}{2^{2i}}=x-\sum_{i=1}^{s}\tfrac{T_i(x)}{2^{2i}}.
	\end{equation*}

	Given $N\in \N^+$, there exists a unique $k\in \N^+$ such that $ (k-1)2^{k-1}+1\le N\le k2^k$. For this $k$, using $s=Lk$, we can construct a ReLU FNN as shown in Figure~\ref{fig:architectureApproxSquare} to implement a function $\phi=f_{Lk}$  approximating $x^2$ well. Note that $T_i$ can be implemented by a one-hidden-layer ReLU FNN with width $2^i$. Hence, the network in Figure~\ref{fig:architectureApproxSquare} has width $k2^k+1\le 3N$\footnote{This inequality is clear for $k=1,2,3,4$. In the case $k\ge 5$, we have $k2^k+1\le \tfrac{k2^k+1}{N}N\le \tfrac{(k+1)2^k}{(k-1)2^{k-1}}N\le 2\tfrac{k+1}{k-1}N\le 3N$. } and depth $2L$. 
	
	\begin{figure}[!htp]
		\centering
		\includegraphics[width=0.985\textwidth]{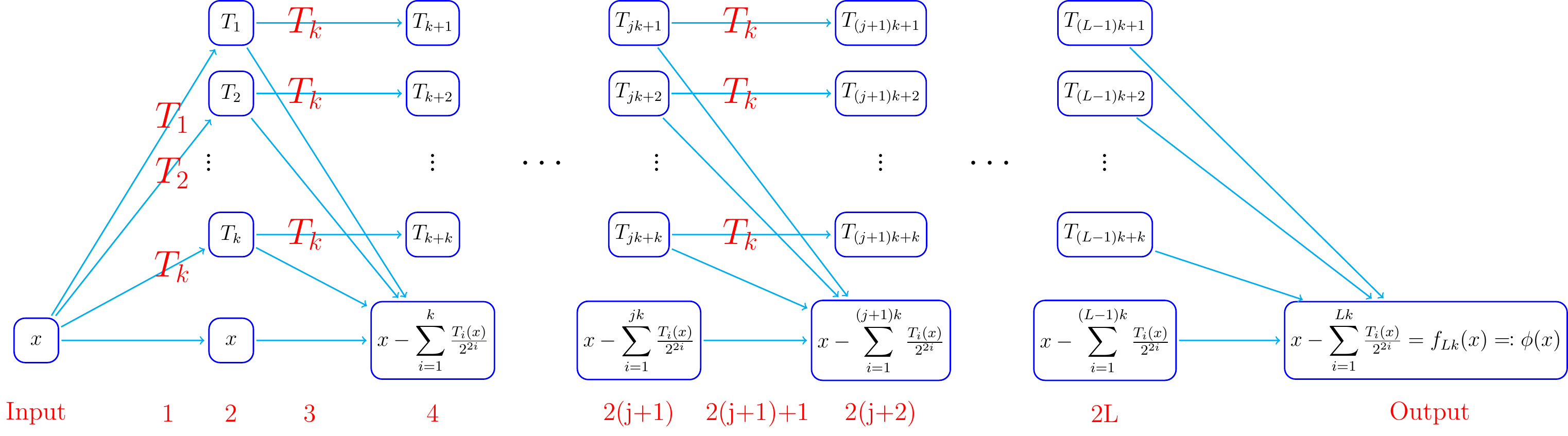}
		\caption{An illustration of the target  network architecture for approximating $x^2$ on $[0,1]$. 
			$T_i$ can be implemented by a one-hidden-layer ReLU FNN with width $2^i$ for $i=1,2,\cdots,K$. The red numbers below the architecture indicate the order of hidden layers.
		}
		\label{fig:architectureApproxSquare}
	\end{figure}

	As shown in Figure~\ref{fig:architectureApproxSquare},  the $(2\ell)$-th hidden layer of the network has the identify function as activation functions for $\ell=1,2,\cdots,L$. Thus, the network in Figure~\ref{fig:architectureApproxSquare} can be interpreted as a ReLU FNN with width $3N$ and depth $L$.
	In fact, if all activation functions in a certain hidden layer are identity maps, the depth can be reduced by one via combining two adjacent linear transforms into one. For example, suppose $\bm{W}_1\in \R^{N_1\times N_2}$, $\bm{W}_2\in \R^{N_2\times N_3}$, and $\varrho$ is an identity map that can be applied to vectors or matrices elementwisely; then $\bm{W}_1\varrho(\bm{W}_2\bmx)=\bmW_3\bmx$ for any $\bmx\in\R^{N_3}$, where $\bmW_3=\bmW_1\cdot\bmW_2\in \R^{N_1\times N_3}$.

	It remains to estimate the approximation error  of $\phi(x)\approx x^2$. By Equation~\eqref{eq:xsquareandfs}, for any $x\in [0,1]$, we have
	\begin{equation*}
		|\phi(x)-x^2|= |f_{Lk}(x)-x^2|\le 2^{-2(Lk+1)}\le 2^{-2Lk}\le N^{-L},
	\end{equation*}
	where the last inequality comes from $N\le k2^k\le 2^{2k}$. So we finish the proof.    
\end{proof}

We have constructed a ReLU FNN to approximate $f(x)=x^2$. By the fact that $xy=2\big((\tfrac{x+y}{2})^2-(\tfrac{x}{2})^2-(\tfrac{y}{2})^2\big)$, it is easy to construct a new ReLU FNN to approximate $f(x,y)=xy$ as follows.

\begin{lemma}
	\label{lem:xyApprox}
	For any $N,L\in \N^+$, there exists a function $\phi$ implemented by a ReLU FNN with width $9N$ and depth $L$ such that
	\begin{equation*}
		|\phi(x,y)-xy|\le 6N^{-L}\quad \tn{for any $x,y\in [0,1]$.}
	\end{equation*}
\end{lemma}
\begin{proof}    
	By Lemma~\ref{lem:approxSquare}, there exists a function $\psi$ implemented by a ReLU FNN with width $3N$ and depth $L$ such that
	\begin{equation*}
		|x^2-\psi(x)|\le N^{-L}\quad \tn{for any $x\in [0,1]$.}
	\end{equation*}    
	Inspired by the fact
	\begin{equation*}
		xy=2\big((\tfrac{x+y}{2})^2-(\tfrac{x}{2})^2-(\tfrac{y}{2})^2\big)\quad \tn{for any $x,y\in \R$,}
		\label{eq:xyEquality}
	\end{equation*}
	we construct the desired function $\phi$ as 
	\begin{equation}\label{eq:x2:xy}
		\phi(x,y)\coloneqq 2\big(\psi(\tfrac{x+y}{2})-\psi(\tfrac{x}{2})-\psi(\tfrac{y}{2})\big)\quad \tn{for any $x,y\in \R$.}
	\end{equation} 
	Then $\phi$ can be implemented by the network architecture in Figure~\ref{fig:xy=phi}. 	 
	\begin{figure}[H]
		\centering
		\includegraphics[width=0.85\textwidth]{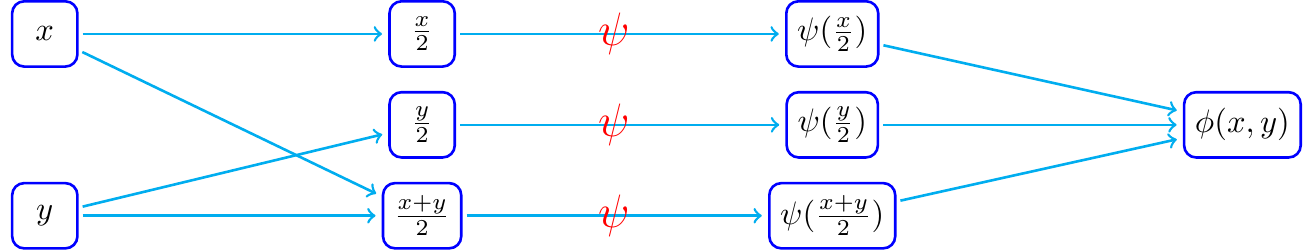}
		\caption{An illustration of the network architecture implementing $\phi$ for approximating $xy$ on $[0,1]^2$.}
		\label{fig:xy=phi}
	\end{figure} 
	It follows from $\psi\in\NNF(\NNwidth\le 3N\NNspace\NNdepth\le L)$ that the network in Figure~\ref{fig:xy=phi} is with width $9N$ and depth $L+2$. Similar to the discussion in the proof of Lemma~\ref{lem:approxSquare}, the network in Figure~\ref{fig:xy=phi} can be interpreted as a ReLU FNN with width $9N$ and depth $L$, since two of the hidden layers have the identify function as their activation functions.
	Moreover, for any $x,y\in [0,1]$,
	\begin{equation*}
		\begin{split}
			|xy-\phi(x,y)|&= \left|2\big((\tfrac{x+y}{2})^2-(\tfrac{x}{2})^2-(\tfrac{y}{2})^2\big)-2\big(\psi(\tfrac{x+y}{2})-\psi(\tfrac{x}{2})-\psi(\tfrac{y}{2})\big)\right|\\
			&\le 2\left|(\tfrac{x+y}{2})^2-\psi(\tfrac{x+y}{2})\right|+2\left|(\tfrac{x}{2})^2-\psi(\tfrac{x}{2})\right|+2\left|(\tfrac{y}{2})^2-\psi(\tfrac{y}{2})\right|
			\le 6N^{-L}.
		\end{split}
	\end{equation*}
	Therefore, we have finished the proof.
\end{proof}

Now let us prove Lemma~\ref{lem:xyApproxAB}, which shows how to construct a ReLU FNN to approximate $f(x,y)=xy$ on $[a,b]^2$ with arbitrary $a<b$, i.e., a rescaled version of Lemma~\ref{lem:xyApprox}.

\begin{proof}[Proof of Lemma~\ref{lem:xyApproxAB}]
	By Lemma~\ref{lem:xyApprox}, there exists a function $\psi$ implemented by a ReLU FNN with width $9N$ and depth $L$ such that
	\begin{equation*}
		|\psi(\tildex,\tildey)-\tildex\tildey|\le 6N^{-L}\quad \tn{for any $\tildex,\tildey\in [0,1]$.}
	\end{equation*}
	By setting $\tildex=\tfrac{x-a}{b-a}$ and $\tildey=\tfrac{y-a}{b-a}$ for any $x,y\in [a,b]$, we have $\tildex,\tildey\in[0,1]$, implying
	\begin{equation*}
		\big|\psi(\tfrac{x-a}{b-a},\tfrac{y-a}{b-a})-\tfrac{x-a}{b-a}\tfrac{y-a}{b-a}\big|\le 6N^{-L}\quad \tn{for any $x,y\in [a,b]$.}
	\end{equation*}
	It follows that, \tn{for any $x,y\in [a,b]$,}
	\begin{equation*}
		\big|(b-a)^2\psi(\tfrac{x-a}{b-a},\tfrac{y-a}{b-a})+a(x+y)-a^2-xy\big|\le 6(b-a)^2N^{-L}.
	\end{equation*}
	Define, for any $x,y\in \R$,
	\begin{equation*}
		\phi(x,y)\coloneqq (b-a)^2\psi(\tfrac{x-a}{b-a},\tfrac{y-a}{b-a})+a\cdot\sigma(x+y+2|a|)-a^2-2a|a|.
	\end{equation*}
	Then $\phi$ can be implemented by the network architecture in Figure~\ref{fig:xy=phi=ab}. 
	\begin{figure}[H]
		\centering
		\includegraphics[width=0.85\textwidth]{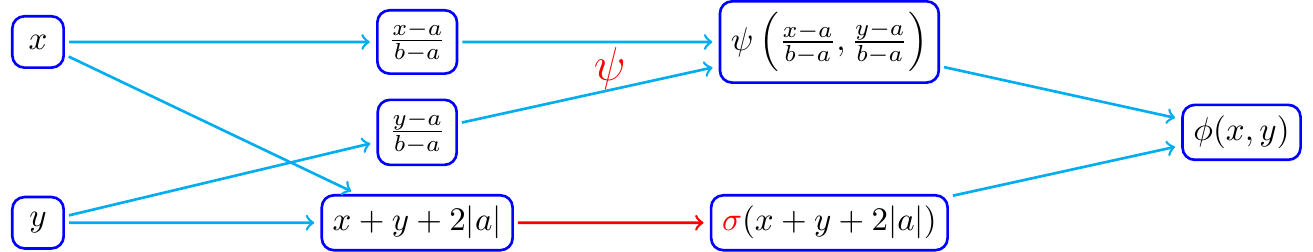}
		\caption{An illustration of the network architecture implementing $\phi$ for approximating $xy$ on $[a,b]^2$. Two of the hidden layers have the identify function as their activation functions, since the red ``\textcolor{red}{$\sigma$}'' comes from the red arrow ``\textcolor{red}{$\longrightarrow$}'', where the red arrow ``\textcolor{red}{$\longrightarrow$}'' is a ReLU FNN with width $1$ and depth $L$. }
		\label{fig:xy=phi=ab}
	\end{figure} 
	It follows from $\psi\in \NNF(\NNwidth\le 9N\NNspace\NNdepth \le L)$ that the network in Figure~\ref{fig:xy=phi=ab} is  with width $9N+1$ and depth $L+2$. 
	Similar to the discussion in the proof of Lemma~\ref{lem:approxSquare}, the network in Figure~\ref{fig:xy=phi=ab} can be interpreted as a ReLU FNN with width $9N+1$ and depth $L$, since two of the hidden layers have the identify function as their activation functions.
	
	Note that $x+y+2|a|\ge 0$ for any $x,y\in [a,b]$, implying
	\begin{equation*}
		\phi(x,y)= (b-a)^2\psi(\tfrac{x-a}{b-a},\tfrac{y-a}{b-a})+a(x+y)-a^2\quad \tn{for any $x,y\in [a,b]$.}
	\end{equation*}
	Hence, 
	\begin{equation*}
		\big|\phi(x,y)-xy\big|\le 6(b-a)^2N^{-L}\quad \tn{for any $x,y\in [a,b]$.}
	\end{equation*}
	So we finish the proof. 
\end{proof}

The next lemma shows how to construct a ReLU FNN to approximate a multivariate  function $f(x_1,x_2,\cdots,x_k)=x_1x_2\cdots x_k$ on $[0,1]^k$.

\begin{lemma}
	\label{lem:xiMultiply}
	For any $N,L,k\in \N^+$ with $k\ge 2$, there exists a function $\phi$ implemented by a ReLU FNN with width $9(N+1)+k-1$ and depth $7kL(k-1)$ such that
	\begin{equation*}
		|\phi(\bmx)-x_1x_2\cdots x_k|\le 9(k-1)(N+1)^{-7kL}\quad \tn{for any $\bmx=[x_1,x_2,\cdots,x_k]^T\in [0,1]^k$.}
	\end{equation*}
\end{lemma}
\begin{proof}
	By Lemma~\ref{lem:xyApproxAB}, there exists a function $\phi_1$ implemented by a ReLU FNN with width $9(N+1)+1$ and depth $7kL$ such that
	\begin{equation}
		|\phi_1(x,y)-xy|\le 6(1.2)^2(N+1)^{-7kL}\le 9(N+1)^{-7kL}\quad \tn{for any $x,y\in [-0.1,1.1]$.}
		\label{eq:xyMultiply}
	\end{equation}
	Next, we construct a sequence of functions $\phi_i:[0,1]^{i+1}\to [0,1]$ for $i\in\{1,2,\cdots,k-1\}$ by induction such that
	\begin{enumerate}[(i)]
		\item $\phi_i$ can be implemented by a ReLU FNN with width $9(N+1)+i$ and depth $7kLi$ for each $i\in \{1,2,\cdots,k-1\}$.
		\item For any $i\in \{1,2,\cdots,k-1\}$ and $x_1,x_2,\cdots,x_{i+1}\in [0,1]$, it holds that
		\begin{equation}
			|\phi_{i}(x_1,\cdots,x_{i+1})-x_1x_2\cdots x_{i+1}|\le 9i(N+1)^{-7kL}.
			\label{eq:xiMultiply}
		\end{equation}
	\end{enumerate}    
	
	First, let us consider the case $i=1$, it is obvious that the two required conditions are true: 1) $9(N+1)+i=9(N+1)+1$ and $7kLi=7kL$ if $i=1$; 2) Equation~\eqref{eq:xyMultiply} implies Equation~\eqref{eq:xiMultiply} for $i=1$.
	
	Now assume $\phi_i$ has been defined; we then define
	\begin{equation*}
		\phi_{i+1}(x_1,\cdots,x_{i+2})\coloneqq \phi_1\big(\phi_i(x_1,\cdots,x_{i+1}),\sigma(x_{i+2})\big)\quad \tn{for any $x_1,\cdots,x_{i+2}\in \R$.}
	\end{equation*}
	Note that $\phi_{i}\in \NNF(\NNwidth\le 9(N+1)+i\NNspace\NNdepth\le 7kLi)$ and $\phi_{1}\in \NNF(\NNwidth\le 9(N+1)+1\NNspace\NNdepth\le 7kL)$. Then $\phi_{i+1}$ can be implemented via a ReLU FNN with width \[\max\{9(N+1)+i+1,9(N+1)+1\}=9(N+1)+(i+1)\] and depth $7kLi+7kL=7kL(i+1)$. 
	
	By the hypothesis of induction, we have 
	\begin{equation}
		\label{eq:hypoInduction}
		|\phi_{i}(x_1,\cdots,x_{i+1})-x_1x_2\cdots x_{i+1}|\le 9i(N+1)^{-7kL}.
	\end{equation}
	Recall the fact that $9i(N+1)^{-7kL}\le 9k2^{-7k}\le 9k\tfrac{2^{-7}}{k}\le0.1$ for any $N,L,k\in\N^+$ and $i\in\{1,2,\cdots,k-1\}$. It follows that 
	\begin{equation*}
		\phi_{i}(x_1,\cdots,x_{i+1})\in [-0.1,1.1]\quad \tn{for any $x_1,\cdots,x_{i+1}\in [0,1]$.}
	\end{equation*}
	
	Therefore, by Equations~\eqref{eq:xyMultiply}  and \eqref{eq:hypoInduction}, we have
	\begin{equation*}
		\begin{split}
			&\ \ \,  {|\phi_{i+1}(x_1,\cdots,x_{i+2})-x_1x_2\cdots x_{i+2}|}
			\\
			&{=\big|\phi_1\big(\phi_i(x_1,\cdots,x_{i+1}),\sigma(x_{i+2})\big)-x_1x_2\cdots x_{i+2}\big|}\\
			&{\le\big|\phi_1\big(\phi_i(x_1,\cdots,x_{i+1}),x_{i+2}\big)-\phi_i(x_1,\cdots,x_{i+1}) x_{i+2}\big|
				+\big|\phi_i(x_1,\cdots,x_{i+1}) x_{i+2}-x_1x_2\cdots x_{i+2}\big|}\\
			& {\le 9(N+1)^{-7kL}+9i(N+1)^{-7kL}=9(i+1)(N+1)^{-7kL},}
		\end{split}
	\end{equation*}    
	for any $x_1,x_2,\cdots,x_{i+2}\in [0,1]$, which means we finish the process of induction.
	
	Now let $\phi\coloneqq \phi_{k-1}$, by the principle of induction, we have 
	\begin{equation*}
		|\phi(x_1,\cdots,x_{k})-x_1x_2\cdots x_{k}|\le 9(k-1)(N+1)^{-7kL}\quad \tn{for any $x_1,\cdots,x_k\in [0,1]$.}
	\end{equation*}
	So $\phi$ is the desired function implemented by a ReLU FNN with width $9(N+1)+k-1$ and depth $7kL(k-1)$, which means we finish the proof.
\end{proof}

With Lemma~\ref{lem:xiMultiply} in hand, we are ready to prove Proposition~\ref{prop:approxPolynomial} for approximating general multivariate  polynomials by ReLU FNNs.

\begin{proof}[Proof of Proposition~\ref{prop:approxPolynomial}]
	The case $k=1$ is trivial, so we assume $k\ge 2$ below. Set $\tildek=\|\bmalpha\|_1\le k$, 
	denote $\bmalpha=[\alpha_1,\alpha_2,\cdots,\alpha_d]^T$, 
	and let $[z_1,z_2,\cdots,z_\tildek]^T\in\R^\tildek$ be the vector such that
	\begin{equation*}
		z_\ell=x_j\quad \tn{if}\ \sum_{i=1}^{j-1}\alpha_i< \ell \le \sum_{i=1}^{j}\alpha_i \quad \tn{for $j=1,2,\cdots,d$}.
	\end{equation*}
	That is,
	\begin{equation*}
		[z_1,z_2,\cdots,z_\tildek]^T=\big[\overbrace{x_1,\cdots,x_1}^{\alpha_1 \ \tn{times}},\overbrace{x_2,\cdots,x_2}^{\alpha_2 \ \tn{times}},\cdots,\overbrace{x_d,\cdots,x_d}^{\alpha_d \ \tn{times}}\big]^T\in \R^\tildek.
	\end{equation*}
	Then we have $P(\bmx)=\bmx^\bmalpha=z_1z_2\cdots z_\tildek$. 
	
	We construct the target ReLU FNN in two steps. First, there exists an affine linear map $\calL:\R^d\to \R^k$ that duplicates  $\bmx$ to form a new vector  $[z_1,z_2,\cdots,z_\tildek,1,\cdots,1]^T\in\R^k$, i.e., $\calL(\bmx)=[z_1,z_2,\cdots,z_\tildek,1,\cdots,1]^T\in\R^k$. Second, by Lemma~\ref{lem:xiMultiply}, there exists a function $\psi:\R^k\to \R$ implemented by a ReLU FNN with width $9(N+1)+k-1$ and depth $7kL(k-1)$ such that $\psi$ maps $[z_1,z_2,\cdots,z_\tildek,1,\cdots,1]^T\in \R^k$ to $z_1z_2\cdots z_\tildek$ within an error  $9(k-1)(N+1)^{-7kL}$. Hence, we can construct the desired function via $\phi\coloneqq\psi\circ \calL$. Then  
	$\phi$ can be implemented by a ReLU FNN with width $9(N+1)+k-1$ and depth $7kL(k-1)\le 7k^2L$, and 
	\begin{equation*}
		\begin{split}
			|\phi(\bmx)-P(\bmx)|=|\phi(\bmx)-\bmx^\bmalpha|&=|\psi\circ\calL(\bmx)-x_1^{\alpha_1}x_2^{\alpha_2}\cdots x_d^{\alpha_d}|\\
			&=|\psi(z_1,z_2,\cdots,z_\tildek,1,\cdots,1)-z_1z_2\cdots z_\tildek|\\
			&\le 9(k-1)(N+1)^{-7kL}\le 9k(N+1)^{-7kL}
		\end{split}
	\end{equation*}  
	for any $x_1,x_2,\cdots,x_d\in [0,1]$.
	So, we finish the proof.
\end{proof}

\subsection{Proof of Proposition~\ref{prop:approxStepFun} for step function approximation}
\label{sec:approxStepFun}
To prove Proposition~\ref{prop:approxStepFun} in this sub-section, we will discuss how to pointwisely approximate step functions by ReLU FNNs except for the trifling region. Before proving Proposition~\ref{prop:approxStepFun}, let us first introduce a basic lemma about fitting $\calO(N_1N_2)$ samples using a two-hidden-layer ReLU FNN with $\calO(N_1+N_2)$ neurons. 

\begin{lemma}
    \label{lem:squarePointsLemma}
    For any $N_1,N_2\in \N^+$, given $N_1(N_2+1)+1$ samples $(x_i,y_i)\in \R^2$ with $x_0<x_1<\cdots<x_{N_1(N_2+1)}$ and  $y_i\ge 0$ for $i=0,1,\cdots,N_1(N_2+1)$,
    there exists $\phi\in \NNF(\NNinput=1;\NNwidthvec=[2N_1,2N_2+1])$ satisfying the following conditions:
    \begin{enumerate}
        \item $\phi(x_i)=y_i$ for $i=0,1,\cdots,N_1(N_2+1)$.
        \item $\phi$ is linear on each interval $[x_{i-1},x_{i}]$ for $i\notin \{(N_2+1)j:j=1,2,\cdots,N_1\} $.
    \end{enumerate}
\end{lemma}

The above lemma is Lemma~$2.2$ of  \cite{SHEN201974}; and the reader is referred to \cite{SHEN201974} for its proof. Essentially, this lemma shows the equivalence of one-hidden-layer ReLU FNNs of size $\calO(N^2)$ and two-hidden-layer ones of size $\calO(N)$ to fit $\calO(N^2)$ samples.

The next lemma below shows that special shallow and wide ReLU FNNs can be represented by deep and narrow ones. This lemma was proposed as Proposition $2.2$ in  \cite{2019arXiv190605497S}.

\begin{lemma}
	\label{lem:widthReduction}
	For any $N,L,d\in \N^+$, it holds that 
	\begin{equation*}
		\begin{split}	       
			&\quad \ \NNF(\NNinput=d\NNspace\NNwidthvec=[N,NL]\NNspace\NNoutput=1)\\
			&\subseteq \NNF(\NNinput=d\NNspace\NNwidth\le 2N+2\NNspace\NNdepth\le L+1\NNspace\NNoutput=1).
		\end{split}
	\end{equation*}
\end{lemma}

With Lemmas~\ref{lem:squarePointsLemma} and \ref{lem:widthReduction} in hand, let us present the detailed proof of Proposition~\ref{prop:approxStepFun}. 
\begin{proof}[Proof of Proposition~\ref{prop:approxStepFun}]
	We divide the proof into two cases: $d=1$ and $d\ge 2$.

\mycase{1}{$d=1$.}

In this case,  $K=\lfloor N^{1/d}\rfloor^2 \lfloor L^{2/d}\rfloor=N^2L^2$. Denote $M=N^2L$ and consider the sample set \[
\begin{split}\big\{(1,M-1),(2,0)\big\}&\bigcup
	\big\{(\tfrac{m}{M},m):m=0,1,\cdots,M-1\big\}\\
	&\bigcup \big\{(\tfrac{m+1}{M}-\delta,m):m=0,1,\cdots,M-2\big\}
	.\end{split}\]
Its size is $2M+1=N\cdot\big((2NL-1)+1\big)+1$. By Lemma~\ref{lem:squarePointsLemma} (set $N_1=N$ and $N_2=2NL-1$ therein), there exists \[\begin{split}
	\phi_1&\in \NNF(\NNwidthvec=[2N,2(2NL-1)+1])\\
	&=\NNF(\NNwidthvec=[2N,4NL-1])
\end{split}\] such that
\begin{itemize}
	\item $\phi_1(\tfrac{M-1}{M})=\phi_1(1)=M-1$ and $\phi_1(\tfrac{m}{M})=\phi_1(\tfrac{m+1}{M}-\delta)=m$ for $m=0,1,\cdots,M-2$;
	\item $\phi_1$ is linear on $[\tfrac{M-1}{M},1]$ and each interval $[\tfrac{m}{M},\tfrac{m+1}{M}-\delta]$ for $m=0,1,\cdots,M-2$. 
\end{itemize}
Then
\begin{equation}
	\label{eq:returnmStepFunc}
	\phi_1(x)=m\quad \tn{if} \ x\in [\tfrac{m}{M},\tfrac{m+1}{M}-\delta\cdot \one_{\{m\le M-2\}}]\quad \tn{for $m=0,1,\cdots,M-1$.}
\end{equation}

Now consider another sample set
\[\begin{split}
	\big\{(\tfrac{1}{M},L-1),(2,0)\big\}&\bigcup\big\{(\tfrac{\ell}{ML},\ell):\ell=0,1,\cdots,L-1\big\}\\
	&\bigcup \big\{(\tfrac{\ell+1}{ML}-\delta,\ell):\ell=0,1,\cdots,L-2\big\}.
\end{split}\] 
Its size is $2L+1=1\cdot\big((2L-1)+1\big)+1$. By Lemma~\ref{lem:squarePointsLemma} (set $N_1=1$ and $N_2=2L-1$ therein), there exists \[\begin{split}
	\phi_2&\in \NNF(\NNwidthvec=[2,2(2L-1)+1])\\
	&=\NNF(\NNwidthvec=[2,4L-1])
\end{split}\] such that
\begin{itemize}
	\item $\phi_2(\tfrac{L-1}{ML})=\phi_2(\tfrac{1}{M})=L-1$ and $\phi_2(\tfrac{\ell}{ML})=\phi_2(\tfrac{\ell+1}{ML}-\delta)=\ell$ for $\ell=0,1,\cdots,L-2$;
	\item $\phi_2$ is linear on $[\tfrac{L-1}{ML},\tfrac{1}{M}]$ and each interval $[\tfrac{\ell}{ML},\tfrac{\ell+1}{ML}-\delta]$ for $\ell=0,1,\cdots,L-2$. 
\end{itemize}
It follows that, for  $m=0,1,\cdots,M-1$ and $\ell=0,1,\cdots,L-1$,
\begin{equation}
	\label{eq:returnlStepFunc}
	\phi_2(x-\tfrac{m}{M})=\ell\quad  \tn{for}\ x\in [\tfrac{mL+\ell}{ML},\tfrac{mL+\ell+1}{ML}-\delta\cdot \one_{\{\ell\le L-2\}}].
\end{equation}

$K=ML$ implies that any $k\in \{0,1,\cdots,K-1\}$ can be unique represented by $k=mL+\ell$ for $m\in\{0,1,\cdots,M-1\}$ and $\ell\in\{0,1,\cdots,L-1\}$. 
Then the desired function $\phi$ can be implemented by ReLU FNN as shown in Figure~\ref{fig:stepFunc}. 

\begin{figure}[H]
	\centering
	\includegraphics[width=0.95\textwidth]{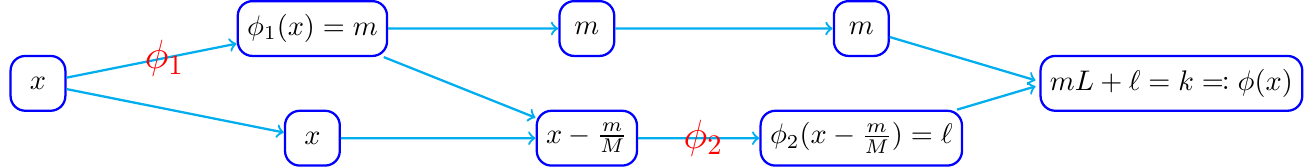}
	\caption{An illustration of the network architecture implementing $\phi$ based on Equations~\eqref{eq:returnmStepFunc} and \eqref{eq:returnlStepFunc} with $x\in [\tfrac{k}{K},\tfrac{k+1}{K}-\delta\cdot \one_{\{k\le K-2\}}]=[\tfrac{mL+\ell}{ML},\tfrac{mL+\ell+1}{ML}-\delta\cdot \one_{\{m\le M-2 \tn{ or }\ell\le L-2\}}]$, where $k=mL+\ell$ for $m=0,1,\cdots,M-1$ and $\ell=0,1,\cdots,L-1$. 
	}
	\label{fig:stepFunc}
\end{figure}
Clearly, 
\begin{equation*}
	\phi(x)=k\quad \tn{if $x\in [\tfrac{k}{K},\tfrac{k+1}{K}-\delta\cdot \one_{\{k\le K-2\}}]$\quad for $k\in\{0,1,\cdots,K-1\}.$}
\end{equation*}
By Lemma~\ref{lem:widthReduction}, $\phi_1\in\NNF(\NNwidthvec=[2N,4NL-1])\subseteq\NNF(\NNwidth\le 4N+2\NNspace\NNdepth\le 2L+1) $ and $\phi_2 \in\NNF(\NNwidthvec=[2,4L-1])\subseteq\NNF(\NNwidth\le 6\NNspace\NNdepth\le 2L+1)$, implying $\phi\in \NNF(\NNwidth\le \max\{4N+2+1,6+1\}=4N+3\NNspace\NNdepth\le (2L+1)+2+(2L+1)+1=4L+5)$.
So we finish the proof for the case $d=1$

\mycase{2}{$d\ge2$.}

Now we consider the case when $d\ge2$. Consider the sample set 
\begin{equation*}
	\begin{split}
		\big\{(1,{K-1}),(2,0)\big\}&\bigcup\big\{(\tfrac{k}{K},k):k=0,1,\cdots,K-1\big\}\\
		&\bigcup \big\{(\tfrac{k+1}{K}-\delta,k):k=0,1,\cdots,K-2\big\},
	\end{split}
\end{equation*} 
whose size is $2K+1=\lfloor N^{1/d}\rfloor \big((2\lfloor N^{1/d}\rfloor \lfloor L^{2/d}\rfloor-1)+1\big)+1$.
By Lemma~\ref{lem:squarePointsLemma} (set $N_1=\lfloor N^{1/d}\rfloor$ and $N_2=2\lfloor N^{1/d}\rfloor \lfloor L^{2/d}\rfloor-1$ therein), there exists 
\begin{equation*}
	\begin{split}
		\phi &\in \NNF(\NNwidthvec =[2\lfloor N^{1/d}\rfloor,2(2\lfloor N^{1/d}\rfloor \lfloor L^{2/d}\rfloor-1)+1])\\
		&= \NNF(\NNwidthvec =[2\lfloor N^{1/d}\rfloor,4\lfloor N^{1/d}\rfloor \lfloor L^{2/d}\rfloor-1])
	\end{split}
\end{equation*} 
such that
\begin{itemize}
	\item $\phi(\tfrac{K-1}{K})=\phi(1)=K-1$, and $\phi(\tfrac{k}{K})=\phi(\tfrac{k+1}{K}-\delta)=k$ for $k=0,1,\cdots,K-2$;
	\item $\phi$ is linear on $[\tfrac{K-1}{K},1]$ and each interval $[\tfrac{k}{K},\tfrac{k+1}{K}-\delta]$ for $k=0,1,\cdots,K-2$. 
\end{itemize}
Then
\begin{equation*}
	\phi(x)=k \quad \tn{if} \ x\in [\tfrac{k}{K},\tfrac{k+1}{K}-\delta\cdot \one_{\{k\le K-2\}}]\quad   \tn{ for $k=0,1,\cdots,K-1$.}
\end{equation*}

By Lemma~\ref{lem:widthReduction}, 
\begin{equation*}
	\begin{split}
		\phi &\in \NNF(\NNwidthvec =[2\lfloor N^{1/d}\rfloor,4\lfloor N^{1/d}\rfloor \lfloor L^{2/d}\rfloor-1])\\
		&\subseteq 
		\NNF(\NNwidth\le 4\lfloor N^{1/d}\rfloor+2\NNspace\NNdepth \le 2\lfloor L^{2/d}\rfloor+1)\\
		&\subseteq 
		\NNF(\NNwidth\le 4\lfloor N^{1/d}\rfloor+3\NNspace \NNdepth \le 4L+5).
	\end{split}
\end{equation*} 
which means we have finished the proof for the case $d\ge 2$.
\end{proof}

\subsection{Proof of Proposition~\ref{prop:pointsMatching} for point fitting}
\label{sec:pointsMatching}
In this sub-section, we will discuss how to use ReLU FNNs to fit a collection of points in $\R^2$.\footnote{Fitting a collection of points $\{(x_i,y_i)\}_i$ in $\R^2$ means that the target ReLU FNN takes a value close to $y_i$ at the location $x_i$.}
It is trivial to fit $n$ points via one-hidden-layer ReLU FNNs with $\calO(n)$ parameters. However, to prove Proposition~\ref{prop:pointsMatching}, we need  to fit $\calO(n)$ points with much fewer parameters, which is the main difficulty of our proof.  Our proof below is mainly based on the ``bit extraction'' technique and the composition architecture of neural networks. 

Let us first introduce a basic lemma
based on the ``bit extraction'' technique, which is actually Lemma~$2.6$ of \cite{2019arXiv190605497S}.

\begin{lemma}
	\label{lem:bitExtractionOld}
	For any $N,L\in \N^+$, any $\theta_{m,\ell}\in \{0,1\}$ for $m=0,1,\cdots,M-1$ and $\ell=0,1,\cdots,L-1$, where $M=N^2L$, there exists a function $\phi$ implemented by a ReLU FNN with width $4N+3$ and depth $3L+3$ such that
	\[\phi(m,\ell)=\sum_{j=0}^{\ell}\theta_{m,j}\quad \tn{
		for $m=0,1,\cdots,M-1\tn{ and }\ell=0,1,\cdots,L-1$.}\]
\end{lemma}

Next, let us introduce Lemma~\ref{lem:bitExtractionNew}, a variant of Lemma~\ref{lem:bitExtractionOld} for a different mapping for the ``bit extraction". Its proof is based on  Lemmas~\ref{lem:squarePointsLemma},
\ref{lem:widthReduction}, and \ref{lem:bitExtractionOld}.

\begin{lemma}
	\label{lem:bitExtractionNew}
	For any $N,L\in \N^+$ and any $\theta_{i}\in \{0,1\}$ for $i=0,1,\cdots,N^2L^2-1$, there exists a function $\phi$ implemented by a ReLU FNN with width $8N+6$ and depth $5L+7$ such that
	\[\tn{$\phi(i)=\theta_{i}$\quad \tn{for} $i=0,1,\cdots,N^2L^2-1$.}\]
\end{lemma}
\begin{proof}
	The case $L=1$ is clear. We assume $L\ge 2$ below.

Denote $M=N^2L$, for each $i\in \{0,1,\cdots,N^2L^2-1\}$, there exists a unique representation $i=mL+\ell$ for $m\in \{0,1,\cdots,M-1\}$ and $\ell\in\{0,1,\cdots,L-1\}$.
Thus, we can define, for $m=0,1,\cdots,M-1$ and $\ell=0,1,\cdots,L-1$,
\begin{equation*}
	a_{m,\ell}\coloneqq  \theta_i,\quad 
	\tn{where $i=mL+\ell$.}
\end{equation*}
Then, for $m=0,1,\cdots,M-1$,  we set $b_{m,0}=0$ and 
$b_{m,\ell}=a_{m,\ell-1}$ for $\ell=1,2,\cdots,L-1$.

By Lemma~\ref{lem:bitExtractionOld}, there exist $\phi_1,\phi_2\in\NNF(\NNwidth\le 4N+3\NNspace \NNdepth\le 3L+3)  $ such that
\begin{equation*}
	\phi_1(m,\ell)=\sum_{j=0}^{\ell}a_{m,j} \quad \tn{and}\quad \phi_2(m,\ell)=\sum_{j=0}^{\ell}b_{m,j}
\end{equation*}
for $m=0,1,\cdots,M-1$ and $\ell=0,1,\cdots,L-1$. 

We consider the sample set \[\{(mL,m):m=0,1,\cdots,M\}\bigcup \big\{\big((m+1)L-1,m\big):m=0,1,\cdots,M-1\big\}.\] 
Its size is $2M+1=N\cdot\big((2NL-1)+1\big)+1$. By Lemma~\ref{lem:squarePointsLemma} (set $N_1=N$ and $N_2=2NL-1$ therein), there exists \[\begin{split}\psi
	&\in \NNF(\NNwidthvec=[2N,2(2NL-1)+1])\\
	&=\NNF(\NNwidthvec=[2N,4NL-1])\end{split}\] such that
\begin{itemize}
	\item $\psi(ML)=M$ and $\psi(mL)=\psi\big((m+1)L-1\big)=m$ for $m=0,1,\cdots,M-1$;
	\item $\psi$ is linear on each interval $[mL,(m+1)L-1]$ for $m=0,1,\cdots,M-1$. 
\end{itemize}
It follows that
\[\psi(x)=m\quad \tn{if}\ x\in [mL,(m+1)L-1]\quad \tn{ for } m=0,1,\cdots,M-1,\]
implying
\begin{equation*}
	\psi(mL+\ell)=m\quad \tn{for $m=0,1,\cdots,M-1$ and $\ell=0,1,\cdots,L-1$.}
\end{equation*}

For $i=0,1,\cdots,N^2L^2-1$, by representing $i=mL+\ell$ for $m=0,1,\cdots,M-1$ and $\ell=0,1,\cdots,L-1$, we have $\psi(i)=\psi(mL+\ell)=m$ and $i-L\psi(i)=\ell$, from which we
deduce
\begin{equation}
	\label{eq:phi1-phi2=phi}
	\begin{split}
		&\quad\ \phi_1\big(\psi(i),i-L\psi(i)\big)-\phi_2\big(\psi(i),i-L\psi(i)\big)\\
		&=\phi_1(m,\ell)-\phi_2(m,\ell)
		=\sum_{j=0}^{\ell} a_{m,j}-\sum_{j=0}^{\ell} b_{m,j}\\
		&=\sum_{j=0}^{\ell} a_{m,j}-\sum_{j=1}^{\ell} a_{m,j-1}-b_0
		=a_{m,\ell}
		=\theta_i.\\
	\end{split}
\end{equation}

Therefore, the desired function $\phi$ can be implemented by the network architecture described in Figure~\ref{fig:depthEfficiency}.
\begin{figure}[H]
	\centering
	\includegraphics[width=0.875\textwidth]{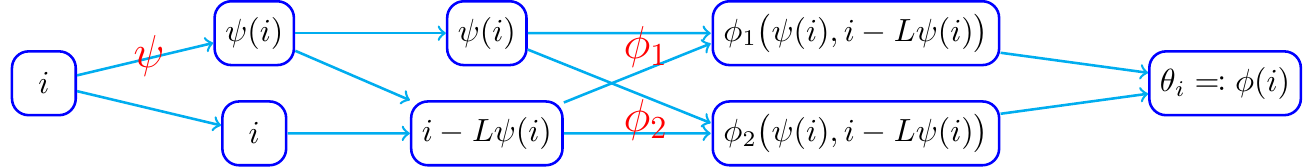}	
	\caption{An illustration of the network architecture implementing the desired function $\phi$ 
		based on Equation~\eqref{eq:phi1-phi2=phi}.
	}
	\label{fig:depthEfficiency}
\end{figure}

Note that \[\phi_1,\phi_2\in \NNF(\NNwidth\le 4N+3\NNspace \NNdepth\le 3L+3). \] 
And by Lemma~\ref{lem:widthReduction}, 
\[\begin{split}\psi&\in \NNF(\NNwidthvec=[2N,4NL-1])\\
	&\subseteq \NNF(\NNwidth\le 4N+2\NNspace \NNdepth\le 2L+1).\end{split}\]
Hence, the network architecture shown in Figure~\ref{fig:depthEfficiency} is with width $\max\{4L+2+1,2(4L+3)\}=8N+6$ and depth $(2L+1)+2+(3L+3)+1=5L+7$, implying $\phi\in\NNF(\NNwidth\le 8N+6\NNspace \NNdepth\le 5L+7)$. So we finish the proof.
\end{proof}

With Lemma~\ref{lem:bitExtractionNew} in hand, we are now ready to prove Proposition~\ref{prop:pointsMatching}.
\begin{proof}[Proof of Proposition~\ref{prop:pointsMatching}]
		Set $J=\lceil{2s\log_2 (NL+1)}\rceil\in\N^+$. For each $\xi_i\in [0,1]$,  there exist $\xi_{i,1}, \xi_{i,2},\cdots,\xi_{i,J}\in \{0,1\}$ such that
	\begin{equation*}
		\big|\xi_i-\bin 0.\xi_{i,1}\xi_{i,2}\cdots \xi_{i,J}\big|\le 2^{-J}\quad \tn{for $i=0,1,\cdots,N^2L^2-1$}.
	\end{equation*}
	
	By Lemma~\ref{lem:bitExtractionNew}, there exist \[\phi_1,\phi_2,\cdots,\phi_J\in \NNF(\NNwidth\le 8N+6\NNspace \NNdepth\le 5L+7)\]
	such that
	\begin{equation*}
		\phi_j(i)=\xi_{i,j}\quad \tn{for $i=0,1,\cdots,N^2L^2-1$\tn{ and }
			$j=1,2,\cdots,J$.}
	\end{equation*}
	Define
	\begin{equation*}
		\tildephi(x)\coloneqq \sum_{j=1}^{J} 2^{-j}\phi_j(x)\quad \tn{for any $x\in \R$.}
	\end{equation*}
	It follows that, for $i=0,1,\cdots,N^2L^2-1$,
	\begin{equation*}
		\begin{split}
			|\tildephi(i)-\xi_i|&=\Big|\sum_{j=1}^{J} 2^{-j}\phi_j(i)-\xi_i\Big|=\Big|\sum_{j=1}^{J} 2^{-j}\xi_{i,j}-\xi_i\Big|\\
			&=\big|\bin 0.\xi_{i,1}\xi_{i,2}\cdots\xi_{i,J}-\xi_i\big|\le 2^{-J}\le N^{-2s}L^{-2s},
		\end{split}
	\end{equation*}
	where the last inequality comes from
	\begin{equation*}
		2^{-J}=2^{-\lceil{2s\log_2 (NL+1)}\rceil}\le 2^{-2s\log_2 (NL+1)}=(NL+1)^{-2s}\le N^{-2s}L^{-2s}.
	\end{equation*}
	
	Now let us estimate the width and depth of the network implementing $\tildephi$.
	Recall that
	\begin{equation*}
		\begin{split}
			J&=\lceil{2s\log_2 (NL+1)}\rceil\le 2s\big(1+\log_2(NL+1)\big)\le 2s\big(1+\log_2(2N)+\log_2L\big)\\
			&\le 2s\big(1+\log_2(2N)\big)\big(1+\log_2L\big)\le 2s\lceil\log_2(4N)\rceil\lceil\log_2(2L)\rceil,
		\end{split}
	\end{equation*}
	and $\phi_j\in \NNF(\NNwidth\le 8N+6\NNspace \NNdepth\le 5L+7)$ for each $j$.

	\begin{figure}[!htp]
		\centering
		\includegraphics[width=0.9\textwidth]{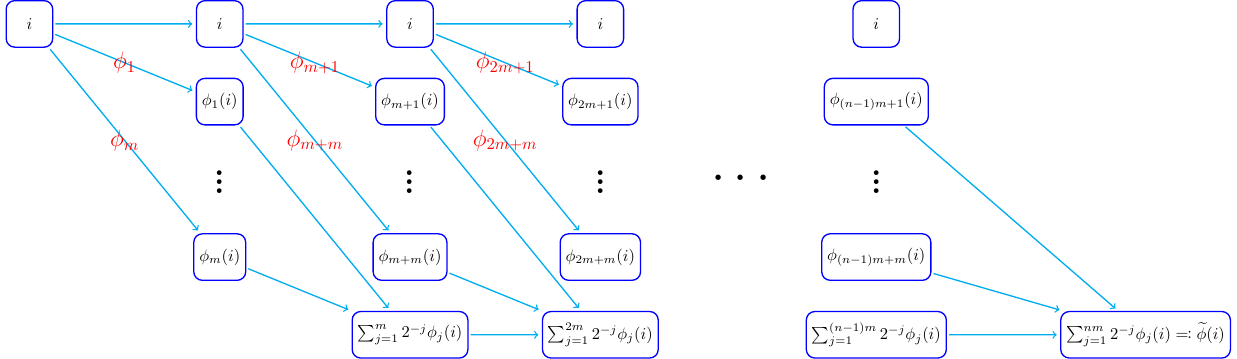}
		\caption{An illustration of the network architecture implementing  $\tildephi=\sum_{j=1}^J 2^{-j}\phi_j$ for any $i\in \{0,1,\cdots,N^2L^2-1\}$. 
			We assume $J=mn$, where $m=2s\lceil\log_2(4N)\rceil$ and $n=\lceil\log_2(2L)\rceil$, since we can set $\phi_{J+1}=\cdots=\phi_{nm}=0$ if $J<nm$.  }
		\label{fig:tphiSmoothFunc}
	\end{figure}

	As we can see from Figure~\ref{fig:tphiSmoothFunc}, $\tildephi=\sum_{j=1}^J 2^{-j}\phi_j$ can be implemented by a ReLU FNN with width 
	\begin{equation*}
	    \begin{split}
	        (8N+6)m+(1+m+1)
	        &=(8N+6)2s\lceil\log_2(4N)\rceil+2s\lceil\log_2(4N)\rceil +2\\
	        &\le 16s(N+1)\log_2(8N)
	    \end{split}
	\end{equation*} 
	and depth 
	\[\big((5L+7)+1\big)n=(5L+8)\lceil\log_2(2L)\rceil\le (5N+8)\log_2(4L).\]

	Finally, we define \[\phi(x)\coloneqq\min\big\{\sigma\big(\tildephi(x)),1\big\}=\min\big\{\max\{0,\tildephi(x)\},1\big\} \quad \tn{for any $x\in \R$}.\]
	Then $0\le \phi(x)\le1$ for any $x\in \R$ and $\phi$ can be implemented by a ReLU FNN with width $16s(N+1)\log_2(8N)$ and depth $(5L+8)\log_2(4L)+3\le 5(L+2)\log_2(4L)$. See Figure~\ref{fig:tphiToPhi} for the network architecture implementing $\phi$. Note that 
	\begin{equation*}
		\tildephi(i)=\sum_{j=1}^J 2^{-j}\phi_j(i)=\sum_{j=1}^J 2^{-j}\xi_{i,j}\in [0,1]\quad \tn{for $i=0,1,\cdots,N^2L^2-1$.}
	\end{equation*}
	
	\begin{figure}[!htp]
		\centering
		\includegraphics[width=0.75\textwidth]{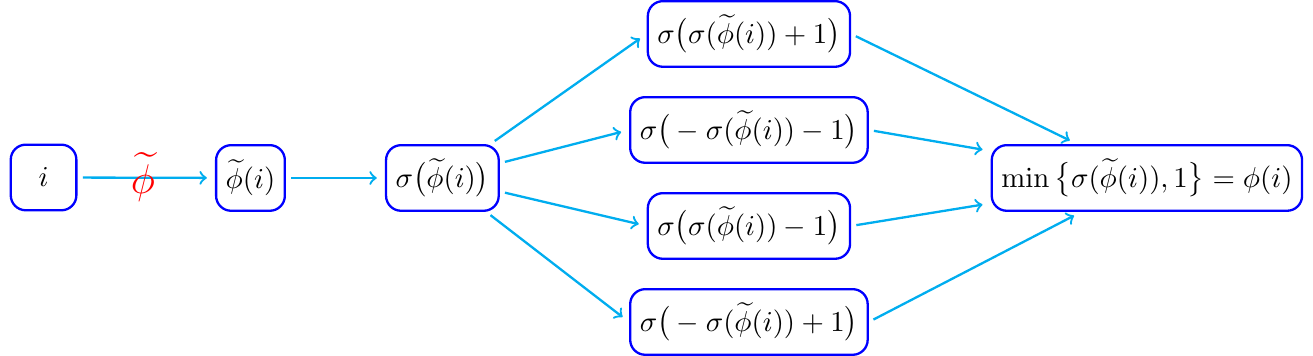}
		\caption{An illustration of the network architecture implementing the desired function $\phi$ based on the fact that $\min\{x_1,x_2\}=\tfrac{x_1+x_2-|x_1-x_2|}{2}=\tfrac{\sigma(x_1+x_2)-\sigma(-x_1-x_2)-\sigma(x_1-x_2)-\sigma(-x_1+x_2)}{2}$. 
		}
		\label{fig:tphiToPhi}
	\end{figure}
	
	It follows that
	\begin{equation*}
		|\phi(i)-\xi_i|=\Big|\min\big\{\max\{0,\tildephi(i)\},1\big\}-\xi_i\Big|=|\tildephi(i)-\xi_i|\le N^{-2s}L^{-2s},
	\end{equation*}  
	\tn{for $i=0,1,\cdots,N^2L^2-1$.}  
	The proof is complete. 
\end{proof}

\section{Conclusions}
\label{sec:conclusion}

This paper has established a nearly optimal approximation error  of ReLU FNNs in terms of both width and depth to approximate smooth functions. It is shown that ReLU FNNs with width $\calO(N\ln N)$ and depth $\calO(L\ln L)$ can approximate functions in the unit ball of $C^s([0,1]^d)$ with an approximation error  $ \calO(N^{-2s/d}L^{-2s/d})$. Through VC-dimension, it is also proved that this approximation error  is asymptotically nearly tight for the closed unit ball of $C^s([0,1]^d)$.

We would like to remark that our analysis is for the fully connected feed-forward neural networks with the ReLU activation function. It would be an interesting direction for further study to generalize our results  to neural networks with other architectures (e.g., convolutional neural networks and ResNet) and activation functions (e.g., tanh and sigmoid functions). These will be subjects of future work. 

\section*{Acknowledgments}
The work of J.~Lu is supported in part by the National Science Foundation via grants DMS-1415939, CCF-1934964, and DMS-2012286.
Z.~Shen is supported by Tan Chin Tuan Centennial Professorship. H.~Yang H.~Yang was partially supported by the National Science Foundation under award DMS-1945029.

\bibliographystyle{plain}%
\bibliography{references}%
\end{document}

%% file: preamble.tex
\usepackage{mathrsfs,amsmath,amsfonts,amssymb}
\usepackage{mathtools,xcolor,enumerate}
\usepackage{dsfont}
\usepackage{bm}
\usepackage{MnSymbol}
\usepackage{indentfirst}
\usepackage[compress,sort]{cite}
\usepackage{url}
\DeclareUrlCommand\email{}
\usepackage{tabularx,multirow,array,booktabs}
\usepackage{colortbl}
\usepackage{graphicx}
\graphicspath{{figures/}}%
\usepackage{float}
\usepackage{subcaption} 
\usepackage[bookmarks,colorlinks]{hyperref}
\hypersetup{citecolor=cyan,linkcolor=blue}
\makeatletter\@addtoreset{equation}{section}\makeatother

\usepackage{amsthm}
\theoremstyle{plain}
\newtheorem{theorem}{Theorem}[section]%
\newtheorem{corollary}[theorem]{Corollary} %
\newtheorem{lemma}[theorem]{Lemma}

\newtheorem{proposition}[theorem]{Proposition}
\theoremstyle{definition}
\newtheorem{definition}[theorem]{Definition}
\theoremstyle{remark}

\newcommand{\R}{\ifmmode\mathbb{R}\else$\mathbb{R}$\fi}
\newcommand{\N}{\ifmmode\mathbb{N}\else$\mathbb{N}$\fi}
\newcommand{\Z}{\ifmmode\mathbb{Z}\else$\mathbb{Z}$\fi}
\newcommand{\Q}{\ifmmode\mathbb{Q}\else$\mathbb{Q}$\fi}
\let\tn\textnormal
\newcommand{\NN}{\mathcal{N\hspace{-2.5pt}N}}
\let\NNF\NN
\newcommand{\NNspace}{;\ }
\newcommand{\NNinput}{\tn{\#input}}
\newcommand{\NNwidthvec}{\tn{widthvec}}

\newcommand{\NNdepth}{\tn{depth}}
\newcommand{\NNnode}{\tn{\#neuron}}

\newcommand{\NNparameter}{\tn{\#parameter}}
\newcommand{\NNwidth}{\tn{width}}

\newcommand{\NNoutput}{\tn{\#output}}
\newcommand{\bmx}{{\bm{x}}}
\newcommand{\bmz}{{\bm{z}}}
\newcommand{\bmW}{{\bm{W}}}
\newcommand{\bme}{{\bm{e}}}
\newcommand{\bmh}{{\bm{h}}}

\newcommand{\bmPsi}{{\bm{\Psi}}}
\newcommand{\bmeta}{{\bm{\eta}}}
\newcommand{\bmbeta}{{\bm{\beta}}}
\newcommand{\bmy}{{\bm{y}}}
\newcommand{\bmtheta}{{\bm{\theta}}}
\newcommand{\bmalpha}{{\bm{\alpha}}}
\newcommand{\bmA}{\bm{A}}

\newcommand{\bmPhi}{{\bm{\Phi}}}
\newcommand{\bmzero}{{\bm{0}}}
\newcommand{\calO}{{\mathcal{O}}}
\newcommand{\calD}{{\mathcal{D}}}
\newcommand{\calN}{{\mathcal{N}}}

\newcommand{\calH}{{\mathcal{H}}}

\newcommand{\calL}{\mathcal{L}}
\newcommand{\calT}{\mathcal{T}}
\newcommand{\calX}{\mathcal{X}}
\newcommand{\tildephi}{{\widetilde{\phi}}}
\newcommand{\tildeg}{{\widetilde{g}}}

\newcommand{\tildek}{{\widetilde{k}}}
\newcommand{\tildex}{{\widetilde{x}}}
\newcommand{\tildey}{{\widetilde{y}}}
\newcommand{\tildeL}{{\widetilde{L}}}
\newcommand{\tildeN}{{\widetilde{N}}}
\newcommand{\tildef}{{\widetilde{f}}}
\newcommand{\tildeC}{{\widetilde{C}}}

\newcommand{\scrI}{{\mathscr{I}}}
\newcommand{\scrE}{{\mathscr{E}}}
\newcommand{\scrT}{{\mathscr{T}}}
\newcommand{\scrF}{{\mathscr{F}}}

\newcommand{\scrX}{{\mathscr{X}}}
\def\one{{\ensuremath{\mathds{1}}}}

\newcommand{\myMathResize}[2][0.9]{\ifmmode
	\scalebox{#1}[#1]{\(\displaystyle #2\)}
	\else
	\[\scalebox{#1}[#1]{\(\displaystyle #2\)}\]
	\fi
}
\newcommand{\Csub}{C^s_u([0,1]^d)}
\newcommand{\vcd}{\tn{VCDim}}
\newcommand{\middleValue}{{\tn{mid}}}
\newcommand{\ball}{{\mathcal{B}}}

\newcommand{\bin}{\tn{bin}\hspace{1.2pt}}

\newenvironment{keywords}{\par \noindent\textbf{Key words}.}{\par}

\newcommand{\mystep}[2]{\par \vspace{0.25cm}\noindent\textbf{\hspace{8pt}Step }$#1\colon$ #2 \vspace{0.18cm} \par }

\newcommand{\mycase}[2]{\par \vspace{0.25cm}\noindent\textbf{\hspace{8pt}Case }$#1\colon$ #2\par \vspace{0.18cm} \par}

\usepackage{tikz}
\newcommand{\myto}[2][1]{\mathop{
		\vcenter{\hbox{\scalebox{1}[#1]{\tikz{\draw[->,line width=0.72pt] (0,0.5) to (0.69*#2,0.5);}}}}
}}

\DeclareMathOperator*{\argmin}{arg\,min}

\usepackage[left,mathlines]{lineno}
\usepackage{refcount}

\definecolor{mycolor}{HTML}{BEBEBE}

\usepackage{refcount}
\makeatletter
\newcommand*\patchAmsMathEnvironmentForLineno[1]{%
    \expandafter\let\csname old#1\expandafter\endcsname\csname #1\endcsname
    \expandafter\let\csname oldend#1\expandafter\endcsname\csname end#1\endcsname
    \renewenvironment{#1}%
    {\linenomath\csname old#1\endcsname}%
    {\csname oldend#1\endcsname\endlinenomath}}%
\newcommand*\patchBothAmsMathEnvironmentsForLineno[1]{%
    \patchAmsMathEnvironmentForLineno{#1}%
    \patchAmsMathEnvironmentForLineno{#1*}}%
\patchBothAmsMathEnvironmentsForLineno{equation}%
\patchBothAmsMathEnvironmentsForLineno{align}%
\patchBothAmsMathEnvironmentsForLineno{flalign}%
\patchBothAmsMathEnvironmentsForLineno{alignat}%
\patchBothAmsMathEnvironmentsForLineno{gather}%
\patchBothAmsMathEnvironmentsForLineno{multline}%
\makeatother

\begin{titlepage}
    \title{Deep Network Approximation for Smooth Functions}
    
    \author{Jianfeng Lu \thanks{Department of Mathematics, Department of Physics, and Department of Chemistry, Duke University (\protect\email{jianfeng@math.duke.edu}).}
    \and Zuowei Shen \thanks{Department of Mathematics,  National University of Singapore
(\protect\email{matzuows@nus.edu.sg}).}
  \and Haizhao Yang \thanks{Department of Mathematics, Purdue University
  (\protect\email{haizhao@purdue.edu}).}
\and Shijun Zhang \thanks{Department of Mathematics,  National University of Singapore
  (\protect\email{zhangshijun@u.nus.edu}).}}
    \date{} 
\end{titlepage}